\documentclass[11pt]{article}
\usepackage{times}
\usepackage{fullpage}
\usepackage{color}
\usepackage{graphicx,amssymb,amsmath}
\usepackage{amsthm}
\usepackage{xspace}
\usepackage{calc}
\usepackage{authblk}

\usepackage{float}
\usepackage{mathabx}
\usepackage{natbib}
 
\usepackage{times}
\usepackage{color}
\usepackage{xspace}
\usepackage{calc}
\usepackage{float}
\usepackage{mathabx}
\usepackage{ytableau}

\usepackage{algorithm}
\usepackage{algorithmic}

\newcommand{\commentout}[1]{}
\newcommand{\eat}[1]{}
\newcommand{\topic}[1]{\vspace{0.2cm}\noindent{\bf {#1}:}}
\newcommand{\calA}{{\mathcal A}}

\newcommand{\calH}{{\mathcal H}}

\newcommand{\calF}{{\mathcal F}}

\newcommand{\calP}{{\mathcal P}}

\newcommand{\imag}{{\mathfrak{i}}}

\newcommand{\Exp}{{\mathbb{E}}}
\newcommand{\Var}{\operatorname{Var}}

\DeclareMathOperator{\Mdis}{Mom}

\DeclareMathOperator{\tran}{Tran}
\DeclareMathOperator{\poly}{poly}
\newcommand{\lip}{1\text{-}\mathsf{Lip}}

\newcommand{\mix}{\boldsymbol{\vartheta}}
\newcommand{\pmix}{\boldsymbol{\phi}}
\newcommand{\hpmix}{\widehat{\pmix}}
\newcommand{\tpmix}{\widetilde{\pmix}}

\newcommand{\tmix}{\widetilde{\mix}}
\newcommand{\bmix}{\overline{\mix}}
\newcommand{\hmix}{\widehat{\mix}}
\newcommand{\cmix}{\widecheck{\mix}}
\newcommand{\mixnew}{\mix_{\text{new}}}
\newcommand{\simplex}{\Delta}
\newcommand{\signedsimplex}{\Sigma}

\newcommand{\convex}{\mathcal{C}}

\newcommand{\Spike}{\mathrm{Spike}}

\newcommand{\veca}{\boldsymbol{\alpha}}
\newcommand{\vecb}{\boldsymbol{\beta}}
\newcommand{\vecw}{\boldsymbol{w}}
\newcommand{\vecx}{\boldsymbol{x}}
\newcommand{\vecc}{\boldsymbol{c}}
\newcommand{\vecr}{\boldsymbol{r}}
\newcommand{\vecy}{\boldsymbol{y}}

\newcommand{\vecp}{\boldsymbol{p}}
\newcommand{\vecq}{\boldsymbol{q}}
\newcommand{\vecg}{\boldsymbol{g}}

\newcommand{\proj}[1]{\Pi_{#1}}

\newcommand{\ha}{\widehat{\alpha}}
\newcommand{\hveca}{\widehat{\veca}}
\newcommand{\ba}{\overline{\alpha}}
\newcommand{\bveca}{\overline{\veca}}
\newcommand{\ta}{\widetilde{\alpha}}
\newcommand{\tveca}{\widetilde{\veca}}

\newcommand{\cveca}{\widecheck{\veca}}
\newcommand{\hb}{\widehat{\beta}}
\newcommand{\hvecb}{\widehat{\vecb}}
\newcommand{\tb}{\widetilde{\beta}}
\newcommand{\tvecb}{\widetilde{\vecb}}
\newcommand{\bb}{\widebar{\beta}}
\newcommand{\bvecb}{\widebar{\vecb}}
\newcommand{\hw}{\widehat{w}}
\newcommand{\hvecw}{\widehat{\vecw}}
\newcommand{\bw}{\widebar{w}}
\newcommand{\bvecw}{\widebar{\vecw}}
\newcommand{\tw}{\widetilde{w}}
\newcommand{\tvecw}{\widetilde{\vecw}}
\newcommand{\cvecw}{\widecheck{\vecw}}
\newcommand{\hc}{\widehat{c}}
\newcommand{\hvecc}{\widehat{\vecc}}
\newcommand{\tvecy}{\widetilde{\vecy}}
\newcommand{\hvecy}{\widehat{\vecy}}
\newcommand{\hy}{\widehat{y}}
\newcommand{\ty}{\widetilde{y}}

\newcommand{\tM}{\widetilde{M}}

\newcommand{\gap}{\xi}
\newcommand{\noise}{\xi}
\newcommand{\minsep}{\zeta}

\newcommand{\R}{\mathbb{R}}
\newcommand{\C}{\mathbb{C}}

\newcommand{\bt}{\mathbf{t}}

\renewcommand{\d}{\mathrm{d}}

\newcommand{\eps}{\epsilon}

\DeclareMathOperator{\Span}{Span}
\DeclareMathOperator{\sspan}{Span}
\newcommand{\supp}{\mathrm{Supp}}

\newenvironment{proofof}[1]{\begin{proof}[Proof of #1]}{\end{proof}}

\newtheorem{theorem}{Theorem}[section]

\newtheorem{lemma}[theorem]{Lemma}

\newtheorem{proposition}[theorem]{Proposition}

{\theoremstyle{remark} \newtheorem{remark}[theorem]{Remark}}
{\theoremstyle{definition} \newtheorem{definition}[theorem]{Definition}}
\definecolor{olive}{rgb}{0.3, 0.4, .1}
\definecolor{fore}{RGB}{249,242,215}
\definecolor{back}{RGB}{51,51,51}
\definecolor{title}{RGB}{255,0,90}
\definecolor{dgreen}{rgb}{0.,0.6,0.}
\definecolor{gold}{rgb}{1.,0.84,0.}
\definecolor{JungleGreen}{cmyk}{0.99,0,0.52,0}
\definecolor{BlueGreen}{cmyk}{0.85,0,0.33,0}
\definecolor{RawSienna}{cmyk}{0,0.72,1,0.45}
\definecolor{Magenta}{cmyk}{0,1,0,0}

\usepackage{ytableau}

\title{Efficient Algorithms for Sparse Moment Problems without Separation}
\author{Zhiyuan Fan\thanks{fan-zy19@mails.tsinghua.edu.cn}\,\,\,\,\, }
\author{\,\,\, Jian Li\thanks{lijian83@mail.tsinghua.edu.cn} }
\affil{IIIS, Tsinghua University}

\begin{document}

\maketitle

\begin{abstract}%
We consider the sparse moment problem of learning a $k$-spike mixture in high-dimensional space from its noisy moment information in any dimension. We measure the accuracy of the learned mixtures using transportation distance. Previous algorithms either assume certain separation assumptions, use more recovery moments, or run in (super) exponential time. Our algorithm for the one-dimensional problem (also called the sparse Hausdorff moment problem) is a robust version of the classic Prony's method, and our contribution mainly lies in the analysis. We adopt a global and much tighter analysis than previous work (which analyzes the perturbation of the intermediate results of Prony's method). A useful technical ingredient is a connection between the linear system defined by the Vandermonde matrix and the Schur polynomial, which allows us to provide tight perturbation bound independent of the separation and may be useful in other contexts. To tackle the high-dimensional problem, we first solve the two-dimensional problem by extending the one-dimensional algorithm and analysis to complex numbers. Our algorithm for the high-dimensional case determines the coordinates of each spike by aligning a 1d projection of the mixture to a random vector and a set of 2d projections of the mixture. Our results have applications to learning topic models and Gaussian mixtures, implying improved sample complexity results or running time over prior work.
\end{abstract}

% \begin{keywords}%
%   Moment Problems, Topic Modeling, Gaussian Mixture Models
% \end{keywords}

\section{Background}

We study the moment problem in which
we are given the (noisy) information of the moments of a measure $\mu$, and our goal is to recover $\mu$ up to certain accuracy. The moment problem is a classical problem studied for over a century in mathematics, statistics, computer science, physics, control theory, medical image, etc. Various versions of the problem have found numerous applications in different domains (e.g., \citet{schmudgen2017moment, lasserre2009moments,gorenflo2002moment,pintarelli2016numerical,bessis1976positivity,natterer2001mathematics}). 
We study the sparse version of the problem, in which the underlying measure is a  mixture of $k$ discrete distributions over a common discrete domain $[d] = \{1,2,\cdots,d\}$.
In particular, we consider mixture $\mix$ which is a $k$-spike distribution supported on the $d-1$ dimensional simplex
$\simplex_{d-1}=\{\vecx=(x_1,\cdots, x_d)\in \R^d \mid \sum_{i=1}^d x_i=1, x_i\geq 0 \,\,\forall i\in [d]\}$. Each point in $\simplex_{d-1}$ represents a discrete distribution over $[d]$.
We first introduce the sparse moment problem for $d=2$
(also known as the sparse {\em Hausdorff moment} problem
or the $k$-coin model
\citep{schmudgen2017moment,li2015learning,gordon2020sparse}).

\topic{The sparse Hausdorff moment problem}
For $d=2$, the mixture is a $k$-spike distribution supported on $\simplex_1\cong[0,1]$.
We call the model {\em the $k$-coin model} (i.e., a mixture of $k$ Bernoulli distributions).
Denote the underlying mixture as $\mix=(\veca,\vecw)$.
Here, $\veca=\{\alpha_1,\alpha_2,\cdots,\alpha_k\}$
and $\vecw=\{w_1,w_2,\cdots,w_k\}\in \simplex_{k-1}$,
where $\alpha_i\in [0,1]$
specifies the $i$th Bernoulli distribution
and $w_i\in [0,1]$ is the probability of mixture component $\alpha_i$.
For mixture $\mix=(\veca,\vecw)$, 
the $t$th moment is defined as 
$
M_{t}(\mix)=\int_{[0,1]} \alpha^t \mix(\d \alpha) =\sum_{i=1}^k w_i \alpha_i^{t}.
$
Our goal is to recover the unknown parameters $(\veca,\vecw)$ of the mixture $\mix$ given the first $K$ noisy moment values $M'_{t}$ with $|M'_{t}-M_{t}(\mix)|_\infty\leq \gap$ for $1\leq t\leq K$.

We call $K$ the {\em moment number} and $\gap$ the moment precision/accuracy, both of which are essential parameters of the problem. 
We measure the quality of our estimation $\tmix$ 
in terms of the transportation distance between probability
distributions, i.e., $\tran(\tmix, \mix)\leq O(\epsilon)$.
Using transportation distance as the metric is advantageous for several reasons:
(1) if the desired accuracy $\eps$ is much smaller than the minimum separation $\minsep=\min\|\alpha_i-\alpha_j\|$, the estimation $\tmix$ must be a per-spike recovery since it must contain a spike that is sufficiently close to an original spike in $\mix$ (given the weight of the spike is lower bounded);
(2) if the desired accuracy $\eps$ is larger than $\minsep$ or the minimum weight $w_{\min}$, we are allowed to be
confused by two very close spikes or miss a spike with a very small weight, thus potentially avoiding the inverse dependency on $\minsep$ and $w_{\min}$, which are otherwise unavoidable if we must recover every spike.
Generally, for this problem, we are interested in understanding
the relations and trade-offs among the moment number $K$, the moment accuracy/precision $\gap$,
and the accuracy $\epsilon$ of our estimation
$\tmix$ (i.e., transportation distance at most $\epsilon$), and designing efficient algorithms for recovering $\mix$ to the desired accuracy.

\topic{The sparse moment problem in higher dimensions}
For the higher dimensional case, 
suppose the underlying mixture is 
$\mix=(\veca,\vecw)$ where $\veca=\{\veca_1,\ldots, \veca_k\}$ ($\veca_i\in \simplex_{d-1}$ which are $d$-dimensional points in a bounded domain) and 
$\vecw=\{w_1,\ldots, w_k\}\in \simplex_{k-1}$.
For vector $\veca_i=(\alpha_{i,1},\ldots, \alpha_{i,d})$ and multi-index $\bt=(t_1,\cdots,t_d)\in \mathbb{Z}_+^d$, we denote monomial $\veca_i^{\bt}=\alpha_{i,1}^{t_1}\alpha_{i,2}^{t_2}\cdots \alpha_{i,d
}^{t_d}$, and
the $\bt$-moment of $\mix$ as 
$
M_{\bt}(\mix)=\int_{\simplex_{d-1}} \veca^\bt \mix(\d \veca) 
=\sum_{i=1}^k w_i \veca_i^{\bt}.
$ 
In the sparse moment problem, we are given noisy access of $\bt$-moment
$M_{\bt}(\mix)$ for $\|\bt\|_1\leq K$. Here we also call $K$ the {\em moment number}.
Since there are $K^{O(d)}$ different moments, when $d$ is large, we also consider the case
where we have noisy access to the moments of the projections of $\mix$ onto  lower-dimensional subspaces (typically 1d or 2d projections).
An affine transformation of the samples in concrete applications such as topic models can often obtain such noisy moments. 
%Also motivated by topic models, we measure the accuracy of our estimation $\tmix$ 
%in term of the $L_1$-transportation distance between probability
%distributions (see the formal definition in Section~\ref{sec:pre}), which is a more stringent metric than $L_2$-transportation.
%For example, the distance between two discrete distributions $(1/d,\ldots,1/d)$ and $(2/d,\ldots,2/d, 0, \ldots, 0)$ is $1$ in $L_1$ metric of $\R^d$ but only $\sqrt{1/d}$ in $L_2$ metric. These two distributions should be regarded as very different distributions in the topic model and thus $L_1$ should be a more appropriate distance measure for our setting.

%\topic{Applications to learning mixture models}
\subsection{Applications and connections to prior work}
The Hausdorff moment problem and its higher dimensional version have many applications in different areas, such as mathematics, statistics, computer science, physics, etc.
In recent years, the noisy sparse moment problem has found applications
to a variety of unsupervised learning scenarios, including learning topic models \citep{Hof99,BNJ03,PRTV97}, learning Gaussian mixtures \citep{wu2020optimal}, collaborative filtering \citep{HP99,kleinberg2008using}, learning a mixture of product distributions \citep{FeldmanOS05,gordon2021source} and causal inference \citep{gordon2021identifying}.
Now, we first discuss the connection between the sparse moment problem and learning mixture models. Then we list some applications in other areas.

\topic{Applications to learning mixture models.}
Indeed, several prior results for the sparse moment problem we study in this paper were presented explicitly or implicitly in the context of learning mixture models
(e.g., \citet{rabani2014learning,li2015learning,Kim2019HowMS,gordon2020sparse,wu2020optimal,doss2020optimal}).
Here we use the problem of learning topic model as an example: 
We are given a corpus of documents. We adopt the popular 
“bag of words” model and take each document as an unordered multiset of words. 
The assumption is that a small number of $k$ “pure” topics are distributed over the underlying vocabulary of $d$ words. A $K$-word document (i.e., a string in $[d]^K$) is generated by first selecting a
topic $p\in\simplex_{d-1}$ from the mixture $\mix$, and then sampling
$K$ words from this topic.
$K$ is called {\em the snapshot number} and we call such a sample a {\em $K$-snapshot} of $p$. 
%By denoting
%$\bt=(t_1,t_2,\cdots,t_d)\in \mathbb{Z}_+^d$ and $A^{\bt}=\alpha_1^{t_1}\alpha_2^{t_2}\cdots \alpha_d^{t_d}$,
%the $\bt$-moment is $M_{\bt}(\mix)=\int_{\simplex_{d-1}} A^\bt \mix(\d A) =\sum_{i=1}^k w_i A_i^{\bt}$.
For all $\|\bt\|_1\leq K$, we can obtain noisy estimates of $M_{\bt}(\mix)$ or $M_{\bt}(\proj{R}(\mix))$
(where $\proj{R}(\mix)$ is the projection of $\mix$ onto a lower dimensional subspace $R$) using 
the $K$-snapshots (i.e., the documents in the corpus).
We aim to recover the mixture $\mix$ using as few samples as possible.
Once we have computed the noisy moments using samples, the problem reduces to the sparse moment problem.
We will discuss the applications
to learning topic models in Section~\ref{sec:topic} and to learning Gaussian mixtures in Section~\ref{sec:Gaussian}.
%The mixture model is indeed a special case of the mixture models of $k$ product distributions \citep{FeldmanOS05,gordon2021source}, which is further a special case of the more general problem of the mixture models of Bayesian networks \citep{gordon2021identifying}. 

\topic{Other Applications.}
We mention a few other problems where our results are directly or potentially useful:

\begin{enumerate}
    \item 
    The moment problem can be used for recovering the initial condition for a class of heat diffusion equations \citep{gorenflo2002moment}: Suppose the initial condition $f(x,t)$ at $t=0$ is a $k$-spike distribution (e.g., $k$ heat sources), and we can measure $f(x,t)$ at $t=1$ (with measurement noise), and the task is to recover the initial $k$-spike distribution \citep{gorenflo2002moment}. 
    See Remark~\ref{remark:heatequation} for more details.
    In some other PDE problems, the input moment information is computed using a particular numerical algorithm (now the noise comes from the error/precision of the numerical algorithms), and the moment problem is used as a sub-procedure for solving the PDE. For example, the moment problem can be used for approximately solving a class of Poisson equations \citep{pintarelli2016numerical}, in which the moment information can be computed using numerical integration.
    \item
    The Hausdorff moment problem is used as a sub-procedure for bounding/locating the zeros of the partition function of the lattice Ising model (Lee-Yang theory) \citep{bessis1976positivity}. Here the moments can be computed exactly for some special lattices and approximated for  more general lattices. 
    \item 
    Our algorithm for higher dimensions is also closely related to computerized tomography, in which we would like to recover the 3d (or higher dimensional) object from its 1d or 2d projections \citep{natterer2001mathematics}. Existing techniques used the information from a continuous set (or a net) of 1d or 2d projections for recovery (e.g., the classical Radon transform). Our technique (see Theorem~\ref{thm:highdim-kspike} and Algorithm~\ref{alg:dim3}) only use only $O(d)$ 1d and 2d projections. 
\end{enumerate}

%In this section, we consider the
%setting where $\mix$ is
%supported on $k$ points in $\simplex_n$ (i.e., $\mix$ is a $k$-spike distribution).

%Besides 1 and 2-snapshot samples, we also use $K=2k-1$-snapshot samples,
%which is known to be necessary even for the coin problem \citep{rabani2014learning}.

\section{Our Contributions}

\noindent
{\bf The $k$-Coin model.}
We first discuss prior work on the $k$-coin model ($\mix$ is supported on $[0,1]$),
and our results.
\citet{rabani2014learning}
showed that recovering a $k$-spike mixture
requires at least the first $K=2k-1$ moments in the worst case
(even without any noise).
Moreover, they showed that for the topic learning problem,
if one use $K=c(2k-1)$ snapshot samples  (for any constant $c\geq 1)$ 
and wish to achieve an accuracy 
$O(1/k)$ in terms of  transportation distance,
$e^{\Omega(K)}$ samples are required
(or equivalently, the moment accuracy should be at most $e^{-\Omega(K)}$).
On the positive side, 
they solved the problem using sample complexity 
$ \max\{(1/\minsep)^{O(k)}, (k/\eps)^{O(k^2)}\}$
(or moment accuracy $\min\{\minsep^{O(k)}, (\eps/k)^{O(k^2)}\}$),
%\jiannote{zhiyuan, please check the original paper. it should depend on $wmin$ as well
%\zynote{checked. The statement of their 1-dim case says they don't really need wmin}
where $\minsep$ is a lower bound of the minimum separation). In fact, their algorithm is a variant of the classic Prony's method
and the (post-sampling) running time is $\poly(k)$ where their bottleneck is solving a special convex quadratic programming instance. 
\citet{li2015learning} provided a linear programming-based algorithm 
that requires the moment accuracy $(\eps/k)^{O(k)}$ (or sample complexity
$(k/\eps)^{O(k)}$), matching the lower bound in \citet{rabani2014learning}. 
However, after computing the noisy moments, the running time of their algorithm  
is super-exponential $k^{O(k^2)}$.
Motivated by a problem in
population genetics that could be reduced to the $k$-coin mixture problem,
\citet{Kim2019HowMS} analyzed the Matrix Pencil Method, 
which requires moment accuracy $\minsep^{O(k)} \cdot w_{\min}^{2} \cdot \epsilon$.
The matrix pencil method requires solving a generalized eigenvalue problem, 
which needs $O(k^3)$ time.
\citet{wu2020optimal} studied the same problem in the context of learning Gaussian mixtures. In fact, they showed that learning 1d Gaussian mixture models with the same variance can be reduced to learning the $k$-coin model. Their algorithm achieves the optimal dependency on 
the moment accuracy without separation assumption. 
Their recovery algorithm is based on SDP and runs in
time $O(k^{2\omega})$ for some $\omega < 2.373$ using the state-of-the-art SDP solver \citep{huang2022solving}.
Recently, \citet{gordon2020sparse} showed that
under the separation assumption, 
it is possible to recover all $k$ spikes up to accuracy $\epsilon$, using 
the first $2k$-moments with accuracy $\minsep^{O(k)} \cdot w_{\min} \cdot \epsilon$. Their algorithm is also a variant of Prony's method and it runs in time $O(k^{2+o(1)})$.
We summarize the prior and our results in Table~\ref{tab:1dim-result}.

%NEED TO FILL THE TABLE WITH NEW WU YANG PAPER.

% The above results can be easily translated into sample complexity results in topic modeling
% with vocabulary size being $2$
% (the model is formally defined in Section~\ref{xxx}).
% In particular, we can use $1/\poly(\noise)$ $K$-snapshot samples to obtain the first $K$ moments up to precision $\noise$.

%The main result of this section is the following theorem,
%which shows that the upper bound in \citet{rabani2014learning,li2015learning}
%can be significantly improved for the $n$-dimensional case,
%and almost matches the lower bound in \citet{rabani2014learning} on the $(2k-1)$-snapshot samples.

\begin{table}[t] 
    \centering
    \begin{tabular}{|l|c|c|c|c|c|}
    \hline 
    Reference & $K$ & Moment Accuracy $(\noise)$ &  Running Time  & Separation\\
    \hline
    \citet{rabani2014learning} & $2k-1$ &  $\min\{\minsep^{O(k)}, (\epsilon/k)^{O(k^2)}\}$ & $\textrm{poly}(k)$ & Required \\
    \hline 
    \citet{Kim2019HowMS} & $2k-1$ &  $\minsep^{O(k)} \cdot w_{\min}^{2} \cdot \epsilon$ & $O(k^3)$ & Required\\
    \hline 
    \citet{gordon2020sparse} & $2k$ &  $\minsep^{O(k)} \cdot w_{\min} \cdot \epsilon$ & $O(k^{2+o(1)})$ & Required\\
    \hline 
    \citet{li2015learning} & $2k-1$ & $(\epsilon/k)^{O(k)}$ &  $(k/\epsilon)^{O(k^2)}$ & No Need\\
    \hline 
    \citet{wu2020optimal} & $2k-1$ & $(\epsilon/k)^{O(k)}$ & $O(k^{2\omega})$ & No Need\\
    \hline 
    Theorem~\ref{thm:1dim-kspikecoin} & $2k-1$ &  $(\epsilon/k)^{O(k)}$ & $O(k^{2})$ & No Need\\
    \hline
    \end{tabular}
    \caption{Algorithms for the $k$-coin problem where $K$ is the moment number.
    The last column indicates whether the algorithm needs
    the separation assumption.
    In particular, the algorithms that require the separation assumption
    need to know $\minsep$ and   $w_{\min}$ where $\minsep$ is the minimum separation, i.e., 
    $\minsep\leq \min_{i\ne j}|\alpha_i-\alpha_j|\leq 1/k$, and $w_{\min}\leq \min_i w_i$. 
    We measure the running time in terms of the arithmetic operations
    In addition, 
    $\poly(k)$ refers to $O(k^c)$ for a relatively large constant $c$ and $\omega$ refers to the smallest real number such that two $k \times k$ matrices can be multiplied using $O(k^{\omega+o(1)})$ operations. We note that there is a trivial lower bound $\omega \geq 2$ while the best-known upper bound is $\omega \leq 2.373$ \citep{alman2021refined}.
    }
    \label{tab:1dim-result}
\end{table}

 %\jiannote{we may need a table for the high-dimensional case as well.}

%\jiannote{Zhiyuan, please update the theorems}

\begin{theorem}
(The $k$-Coin model, informal)
Let $\mix$ be an arbitrary $k$-spike distribution over $[0,1]$.
Suppose we have noisy moments $M'_{t}$ such that 
$|M_t(\mix)-M'_{t}|\leq (\epsilon/k)^{O(k)}$ for $0\leq t \leq 2k-1$. 
We can obtain a mixture $\tmix$ such that $\tran(\tmix, \mix)\leq O(\epsilon)$ in $O(k^{2})$ arithmetic operations.
\end{theorem}

\noindent
{\bf Our techniques:}
Our algorithm for the $k$-coin model is also based on the classic Prony's method \citep{prony1795}.
Suppose the true mixture is $\mix = (\veca, \vecw)$ 
where $\veca = [\alpha_1, \cdots, \alpha_k]^{\top}$ and $\vecw = [w_1, \cdots, w_k]^{\top}$. 
In Prony's method, if we know the moment vector $M$ exactly,
every $\alpha_i$ can be recovered from the roots
of a polynomial whose coefficients are the entries of the eigenvector $\vecc$
(corresponding to eigenvalue 0)
of a Hankel matrix (see e.g., the matrix $\calH_{k+1}$ in \citet{gordon2020sparse}).
However, if we only know the noisy moment vector $M'$, we only get a perturbation of the original Hankel matrix.
Recent analyses of Prony's method \citep{rabani2014learning,gordon2020sparse} aim to upper bound the error of recovery from the perturbed Hankel matrix.  
If $\alpha_i$ and $\alpha_j$ ($i\ne j$) are very close, or any $w_i$ is very small, the second smallest eigenvalue of the Hankel matrix is also very close to $0$. From matrix perturbation theory, we know that the eigenvector $\vecc$ corresponding to eigenvalue $0$ is extremely sensitive to small perturbation (see e.g.,\citet{stewart1990matrix}).  
Hence, the recent analyses of Prony's method \citep{rabani2014learning,gordon2020sparse}
requires that $\alpha_i$ and $\alpha_j$ are separated by at least $\minsep$, and the minimum $w_i$ is lower bounded by $w_{\min}$ to ensure certain stability of the eigenvector (hence the coefficients of the polynomial), and the corresponding sample complexity becomes unbounded when $\minsep$ or $w_{\min}$ approaches to 0.
Our algorithm (Algorithm~\ref{alg:dim1}) can be seen as a robust version of Prony's method. Instead of computing the smallest eigenvector of the perturbed Hankel matrix, we solve a ridge regression to obtain a vector $\hvecc$, which plays a similar role as $\vecc$ (Line~\ref{line:dim1-hc} in Algorithm~\ref{alg:dim1}). However, instead of showing $\hvecc$ is close to $\vecc$ (in fact, they can be very different), we adopt a more global analysis to show that the moment vector of the estimated mixture is close to the true moment vector $M$, and by a moment-transportation inequality
(see Section~\ref{app:momenttransineq}),
we can guarantee the quantity of our solution in terms of transportation distance. 

Another technical challenge lies in bounding the error in recovering the weight vector $\vecw$. In the noiseless case, the weight is simply
the solution of $V_{\veca} \vecw =M$ where $V_{\veca}$ is a Vandemonde matrix
(see Equation \eqref{eq:matrixdef}).
It is known that Vandermonde matrices tend to be badly ill-conditioned \citep{gautschi1987lower,pan2016bad,Moitra15} (with a large dependency on the inverse of the minimal separation). Hence, using standard condition number-based analysis, slight perturbations of $V_{\veca}$ and $M$ may result in an unbounded error.
\footnote{
    Note that for Vandermonde matrix $V=(a^j_i)_{0\leq i<k,1\leq j\leq k}$, its determinant
    is $\det(V)=(-1)^{k(k-1)/2}\prod_{p<q} (a_p - a_q)$.  
    Hence, $\det(V)^{-1}$ depends inversely on the separation.
}
However, we show interestingly, in our case, that we can bound the error independent of the separation (Lemma~\ref{lm:step3}) via the connection between Vandermonde linear systems and Schur polynomial (Lemma~\ref{lm:dim1-step3-matrix}).
Our technology may be useful in analyzing the perturbation of Vandermonde linear systems
in other contexts as well.

% \begin{theorem}
% \label{thm:kspike}
% (High Dimension)
% Let $\mix$ be an arbitrary $k$-spike mixture in $\simplex_{d-1}$.
% Using $\poly(d,k, 1/\epsilon)$ 1- and 2-snapshot samples, and $(k/\epsilon)^{O(k)}$ $(2k-1)$-snapshot samples respectively, we can obtain, with probability 0.99, a mixture $\tmix$ such that
% $\tran(\tmix, \mix)\leq O(\epsilon)$.
% \end{theorem}

\vspace{0.1cm}
\noindent
{\bf Higher Dimensions.}
For the high-dimensional case,
especially when $d$ is large, it is costly to obtain all moments 
$M_{\bt}(\mix)$ for all $\bt$ with $\|\bt\|_1\leq K$.
Indeed, most previous work
\citep{rabani2014learning,li2015learning,wu2020optimal,gordon2020sparse,doss2020optimal} on learning mixtures solved the problem using noisy moment information of the projections of $\mix$ onto some lower-dimensional subspaces (in fact, lines).
%In fact, all prior works \citep{rabani2014learning,li2015learning,wu2020optimal,gordon2020sparse,doss2020optimal} recover the high-dimensional distribution by ``assembling" many one-dimensional projections.
In particular, if one projects $\mix$ to $O(k)$ directions such as in \citet{rabani2014learning,wu2020optimal,gordon2020sparse},
existing techniques require the separation assumption for the recovery.
\footnote{
It is possible to convert an algorithm with a separation assumption
to an algorithm without a separation assumption by merging close-by spikes and removing 
spikes with small weights. However, the resulting algorithm is far from optimal. 
For example, sample complexity $(k/\eps)^{O(k)} \cdot w_{\min}^{-O(1)} \cdot \minsep^{-O(k)}$ in \citet{gordon2020sparse} can be converted to one without separation with sample complexity $(k/\eps)^{O(k^2)}$. We omit the details.
}
Other works without separation condition
such as \citet{li2015learning,doss2020optimal}
requires to project $\mix$ to a net of exponentially many directions. 

In our work, we assume we have noisy access to the moments
of the linear projections $\proj{R}(\mix)$ of $\mix$ onto lines or 2d planes (one may refer to the formal definition of $\proj{R}(\mix)$
in Section~\ref{subsec:algoforhigherdim}).
See Theorem~\ref{thm:highdim-informal}
for an informal statement of our result. 
A comparison between our result and prior work is summarized in
Table~\ref{tab:ndim-result}.
One can see that we either use much fewer projections
(hence better running time) or do not need the separation assumption.
We note that using 2d projections
is crucial for the improvement beyond prior work. 
We hope this idea of using 2d projections 
is helpful in solving other high-dimensional recovery problems.

\begin{table}[t] 
    \centering
    \begin{tabular}{|l|c|c|c|c|}
    \hline 
    Reference & \#Projections & Moment Accuracy $(\noise)$ & Running Time & Separation \\
    \hline
    \citet{rabani2014learning} & $O(k)$ & $(\eps/k)^{O(k^2)} \cdot (w_{\min} \zeta)^{O(k^2)}$ & $\textrm{poly}(k)$ & Required \\
    \hline 
    \citet{gordon2020sparse} & $O(k)$ & $(\eps/k)^{O(k)} \cdot w_{\min}^{O(1)} \cdot \minsep^{O(k)}$ & $\textrm{poly}(k)$ & Required\\
    \hline 
    \citet{wu2020optimal} & $O(d)$ & $(\eps/k)^{O(k)} \cdot w_{\min} \cdot \minsep$ & $O(dk^{2\omega})$ & Required\\
    \hline 
    \citet{li2015learning} & $(k/\epsilon)^{O(k)}$ & $(\eps/k)^{O(k^2)}$ & $(k/\epsilon)^{O(k^2)}$ & No Need\\
    \hline 
    \citet{doss2020optimal} & $(k/\epsilon)^{O(k)}$ & $(\eps/k)^{O(k)}$ & $(k/\epsilon)^{O(k)}$ & No Need\\
    \hline 
    Theorem~\ref{thm:highdim-kspike} & $O(d)$* & $(\eps/k)^{O(k)}$ & $O(dk^3)$ & No Need\\
    \hline
    \end{tabular}
    \caption{Algorithms for recovering the mixture $\mix$ in
    higher dimensional space. 
    We assume that dimension $d < k$ and we have oracle access to noisy moments of any 1d or 2d projections.
    The last column indicates whether the algorithm needs
    the separation assumption.
    *: The second column indicates the number of lower dimensional projections that the algorithm requires. Specifically, our algorithm uses one 1d projection and $O(d)$ 2d projections while all previous algorithms use a certain number of 1d projections.}
    \label{tab:ndim-result}
\end{table}
\vspace{-0.3cm}

\begin{theorem} (Higher dimension, informal)
\label{thm:highdim-informal}
Let $\mix$ be an arbitrary $k$-spike mixture supported in $\Delta_{d-1}$. Suppose we can access any noisy $\bt$-moments of the 1d and 2d linear projection $\proj{R}(\mix)$ with precision $(\epsilon/(dk))^{O(k)}$,
(i.e., onto lines and 2d planes). 
We can construct a $k$-spike mixture $\tmix$ such that $\tran(\tmix, \mix)\leq O(\epsilon)$ using only $O(dk^3)$ arithmetic operations with high probability.
\end{theorem}
% and our algorithm only uses $O(d\poly(k)})$ arithmetic operations

%\zynote{2d-Projection in this lemma}

%As different previous works used different projection methods, it is difficult to 
%compare the accuracy of their moment estimations, as in the one-dimensional case.
%Hence, we directly state the sample complexity result in the context of topic modeling.
\vspace{0.1cm}
\noindent
{\bf Our techniques:}
We first solve the two-dimensional problem.
This is done by extending the previous one-dimensional algorithm and
its analysis to complex numbers. The real and imaginary parts of a complex location can represent the two dimensions, respectively.
While most ideas are similar to the one-dimensional case, 
the analysis requires extra care in various places. In particular,
we need to extend the moment-transportation inequality to complex numbers
(Appendix~\ref{app:momenttransineq}).

Our algorithm for the high-dimensional case uses
the algorithms for one-dimensional and two-dimensional cases as subroutines.
We first pick a random vector $\vecr$
and learn the projection $\tpmix$ of $\mix$ onto $\vecr$ using the 
one-dimensional algorithm.
It is not difficult to show that for two spikes that are far apart, with reasonable probability,
their projections on $\vecr$ are also far apart.
Hence, the remaining task is to recover the coordinates (in $\R^d$) of each spike in $\tpmix$.
So, we learn the 2d-linear map $\hpmix_t$ of $\mix$ onto the two-dimensional subspace spanned by $\vecr$ and the $t$th dimensional axis for each dimension $t\in [d]$. 
Combining the information from the 1d projected measure $\tmix$ and the 2d projected measure $\hmix_t$, we show that we can extract the $t$th coordinates of all spikes.

%One may notice that by merging spikes in close ranges and discarding spikes with low weight, one can manually create a separation condition for arbitrary distribution, eliminating the separation requirement on the original mixture.
%However, we do not have the moment vector for the transformed mixture, while the original moment vector is biased for this recovery.
%As a result, creating $\noise$ separation will also distort the moment accuracy by $\noise$. For the topic model, this casts an extra $O(k)$ factor on the exponent of $(k/\eps)$ term. 

%Although our algorithm uses a clustering statement which creates gaps in the analysis, the two-dimensional subroutines prevent us from analyzing the distorting at the moment space. 
%Indeed, we can directly analyze its impact using transportation distance on the 2d plane.
%This saves one $O(k)$ factor and gives a better result.

\vspace{0.1cm}

\noindent
{\bf Application to Learning Topic Models.}
Our result for the high-dimensional case
can be easily translated into an algorithm for topic modeling.
Previously, \citet{rabani2014learning} showed that
it is possible to produce an estimate $\tmix$ of the original mixture $\mix$
such that $\tran(\mix, \tmix)\leq \epsilon$,
using $\poly(n,k, 1/\epsilon)$ 1- and 2-snapshot samples, and $(k/\epsilon)^{O(k^2)}  \cdot (w_{\min} \zeta)^{-O(k^2)} $
$K$-snapshot samples, under minimal separation assumptions.
\citet{li2015learning} provided
an LP-based learning algorithm which
uses almost the same number of samples (with a slightly worse polynomial for 2-snapshot samples), but
without requiring any separation assumptions.
The number of the $(2k-1)$-snapshot samples used by both \citet{rabani2014learning,li2015learning} are $(k/\epsilon)^{O(k^2)}$
while that of the lower bound (for 1d) is only $e^{\Omega(k)}$.
%\zynote{May also need to note \citep{gordon2020sparse} here.} 
Recently, \citet{gordon2020sparse} showed that
under the minimal separation assumption, 
the sample complexity could be reduced to $(k/\eps)^{O(k)} \cdot w_{\min}^{-O(1)} \cdot \minsep^{-O(k)}$.
%We summarize the prior results in Table~\ref{tab:ndim-result}.
%We first apply the dimension reduction developed in \citet{li2015learning}
%to find a special subspace of dimension $O(k)$, then apply our Algorithm~\ref{alg:dim3} to the projection of $\mix$ on this subspace. 
%See Section~\ref{sec:topic} and Theorem~\ref{thm:highdimtopic}.
We obtain for the first time the worst case optimal sample complexity for the high-dimensional case, improving previous work 
\citep{rabani2014learning,li2015learning,gordon2020sparse} and matching the lower bound even for the one-dimensional case
\citep{rabani2014learning}.
See Section~\ref{sec:topic} for the details.

% A comparison between our results and prior results can be found in Table~\ref{tab:ndim-result}.
\vspace{0.1cm}

\noindent
{\bf Application to Learning Gaussian Mixtures.}
We also study the problem of learning the parameters of Gaussian mixtures.
We assume that all Gaussian components share a variance parameter, following the setting studied by \citet{wu2020optimal}.
We can leverage our algorithm for the $k$-coin model for the one-dimensional setting.
Our algorithm achieves the same sample complexity as in \citet{wu2020optimal}, but an $O(k^2)$ post-sampling running time, improving over the SDP-based $O(k^{2\omega})$-time algorithm developed by \citet{wu2020optimal}. 
See Theorem~\ref{thm:1dim-Gaussian} for more details.

For the high-dimensional setting, 
we can use the dimension reduction technique in \citet{li2015learning} or 
\citet{doss2020optimal} to reduce the dimension $d$ to $O(k)$.
The dimension reduction part is not a bottleneck and we assume $d=O(k)$ for the 
the following discussion.
We show in Section~\ref{sec:highdim-Gaussian} that we can utilize our algorithm for the high-dimensional sparse moment problem and obtain an algorithm without any separation assumption with sample complexity $(k/\epsilon)^{O(k)}$. Note that the algorithm in \citet{wu2020optimal} requires a sample size of $(k/\eps)^{O(k)} \cdot w_{\min}^{-O(1)} \cdot \minsep^{-O(k)}$, which depends on the separation parameter $\minsep$ between Gaussian distributions. 
Recently, \citet{doss2020optimal} removed the separation assumption
and achieved the optimal sample complexity $(k/\eps)^{O(k)}$.
Compared with \citet{doss2020optimal}, the sample complexity of our algorithm is the same, but our running time is substantially better:
during the sampling phase, the algorithm in \citet{doss2020optimal} requires $O(n^{5/4}\poly(k))$ time where $n=(k/\eps)^{O(k)}$ is the number of samples
(for each sample, they need to update $O(n^{1/4}\poly(k))$ numbers, since 
their algorithm requires $O(n^{1/4})$ 1d projections)
while our algorithm only needs $O(n\poly(k))$ time
(we only need one 1d projection and $d=O(k)$ 2d projections).
The post-sampling running time of the algorithm in \citet{doss2020optimal}
is exponential $(k/\eps)^{O(k)}$ while our algorithms
runs in polynomial time $\poly(k)$.
For more details, see Section~\ref{sec:Gaussian}.

Last but not least, we argue that improving the post-sampling running time is very important,
despite the exponential sample complexity for both problems. 
During the sampling phase, we only need to keep track of the first few moments (e.g., using basic operations such as counting or adding numbers), hence the sampling phase can be easily distributed, streamed, or implemented in inexpensive computing devices.  Our mixture recovery algorithm only requires the moment information (without storing the samples) and runs in time polynomial in $k$ (not the sample size).
Moreover, one may well have other means to measure the moments to achieve the desired moment accuracy (e.g., via longer documents, prior knowledge, existing samples, etc.)
Exploiting other settings and applications is an interesting further direction.

\vspace{-0.3cm}

\section{Preliminaries}\label{sec:pre}

We are given a statistical mixture $\mix$ of $k$ discrete distributions
over $[d]=\{1,2,\cdots,d\}$.
Each discrete distribution $\veca_i$ can be regarded as a point in the $(d-1)$-simplex
$\simplex_{d-1}=\{\vecx=(x_1,\cdots, x_d)\in \R^d \mid \sum_{i=1}^d x_i=1, x_i\geq 0 \,\,\forall i\in [d]\}$.
We use $\mix=(\veca,\vecw)$ to represent the mixture where 
$\veca=\{\veca_1,\veca_2,\cdots,\veca_k\}\subset \simplex_{d-1}$ are the locations of spikes
and $\vecw=\{w_1,w_2,\cdots,w_k\}\in \simplex_{k-1}$
in which $w_i$ is the weight of $\veca_i$.
Since the dimension of $\simplex_{d-1}$ is $d-1$, we say the {\em dimension} of the problem is $d-1$.
In addition, we use $\Spike(\simplex_{d-1},\simplex_{k-1})$ 
to denote the set of all such mixtures where $\simplex_{d-1}$ indicates
the domain of $\veca_i$ and $\simplex_{k-1}$ indicates the domain of $\vecw$.
Since our algorithm may produce negative or complex weights as intermediate results, we further denote $\signedsimplex^{\R}_{d-1}=\{\vecx=(x_1,\cdots, x_d)\in \R^d \mid \sum_{i=1}^d x_i=1\}$
and 
$\signedsimplex^{\C}_{d-1}=\{\vecx=(x_1,\cdots, x_d)\in \C^d \mid \sum_{i=1}^d x_i=1\}$.

For the one-dimensional case (which is called the $k$-coin problem or 
sparse Hausdorff moment problem),
we have 
$\mix = (\veca,\vecw) \in \Spike(\simplex_1, \simplex_{k-1})$
where 
$\veca=\{\alpha_1,\alpha_2,\cdots,\alpha_k\}$ is a set of $k$ discrete points in $[0,1]$.
In this scenario, for each $t\in \mathbb{N}$, we denote the $t$th moment as
$$
M_{t}(\mix)=\int_{[0,1]} \alpha^t \mix(\d \alpha) =\sum_{i=1}^k w_i \alpha_i^{t}.
$$

For the higher dimensional case, we use
$\bt=(t_1,t_2,\cdots,t_d)\in \mathbb{Z}_+^d$ such that $\|\bt\|_1 = K 
(= 2k-1)$
to denote a multi-index.
In addition, we denote monomial
$\veca_i^{\bt}=\alpha_{i,1}^{t_1}\alpha_{i,2}^{t_2}\cdots\alpha_{i,d}^{t_d}$ for every discrete distribution
$\veca_i=(\alpha_{i,1},\alpha_{i,2},\cdots,\alpha_{i,d})\in \R^d$.
The moment vector of $\mix$ is then defined as 
\begin{align}
\label{eq:highdimmomvector}
M(\mix)=\left\{M_{\bt}(\mix)\right\}_{\|\bt\|_1\leq K}
\text{ where }
M_{\bt}(\mix)=\int_{\simplex_{d-1}} \veca^\bt \mix(\d \veca) =\sum_{i=1}^k w_i \veca_i^{\bt}.
\end{align}

We define the {\em moment distance} between two mixtures $\mix$ and $\mix'$
as the $L_{\infty}$ norm of the difference between corresponding moment vectors:
$
\Mdis_K(\mix,\mix'):=\max_{|\bt|\leq K} |M_\bt(\mix)-M_\bt(\mix')|.
$

\topic{Transportation Distance}
For any two probability measures $P$ and $Q$ defined over $\R^d$, we define the $L_1$-transportation distance $\tran(P,Q)$ is defined as
$$\tran(P,Q):= \inf \Big\{\int \|x-y\|_1 \,\d\mu(x,y): \mu\in M(P,Q)\Big\}$$
%\begin{align}
%\label{eq:transprimal}
%\tran(P,Q):= \inf \left\{\int \|x-y\|_1 \,\d\mu(x,y) : \mu\in M(P,Q) %\right\},
%\end{align}
where $M(P,Q)$ is the set of all joint distributions (also called coupling) on $\R^d \times \R^d$ with marginals $P$ and $Q$.
%We omit the subscript when the context is clear.
Transportation distance is also called Rubinstein distance, Wasserstein distance
or earth mover distance in the literature.
We also need to define transportation distance for signed measures and
complex domains. See Appendix~\ref{app:transfornonprob}.

\vspace{-0.5cm}

\section{An Efficient Algorithm for the $k$-Coin Problem}
\label{sec:1d-recover}

In this section, we study the $k$-coin model, 
i.e., $\mix$ is a $k$-spike distribution over $[0,1]$.
We present an efficient algorithm reconstructing the mixture from the first $K=2k-1$ noisy moments. 
Let $\mix:=(\veca, \vecw) \in \Spike(\simplex_1,\simplex_{k-1})$ be the ground truth mixture where $\veca = [\alpha_1, \cdots, \alpha_k]^{\top} \in \simplex_1^k$ and $\vecw = [w_1, \cdots, w_k]^{\top} \in \simplex_k$.
Let $M(\mix) = [M_0(\mix), \cdots, M_{2k-1}(\mix)]^{\top}$ be the ground truth moment vector containing moments of degree at most $K = 2k-1$. 
We denote the noisy moment vector as $M' = [M_0', \cdots, M_{2k-1}']$ where the error is bounded by $\|M' -M(\mix)\|_{\infty} \leq \noise$. 
Since $\mix$ is a distribution, we assume $M'_0=M_0(\mix) = 1$. In the light of lower bound in \citet{rabani2014learning} or Lemma~\ref{lm:inequality1d}, we further assume the noise satisfies $\noise \leq 2^{-\Omega(k)}$.
%\jiannote{Notation. Suggestion $M'$.}

For vector $M = [M_0, \cdots, M_{2k-1}]^{\top}$ and $\veca = [\alpha_1, \cdots, \alpha_k]^{\top}$, we denote
\begin{align}\label{eq:matrixdef}
    A_{M} := \begin{bmatrix}
		M_{0} & M_{1} & \cdots & M_{k-1} \\
		M_{1} & M_{2} & \cdots & M_{k} \\
		\vdots & \vdots & \ddots & \vdots \\
		M_{k-1} & M_{k} & \cdots & M_{2k-2}
	\end{bmatrix}, \,\,\,
	b_{M} := \begin{bmatrix}
		M_{k} \\
		M_{k+1} \\
		\vdots \\
		M_{2k-1}
	\end{bmatrix}, \,\,\,
    V_{\veca} := \begin{bmatrix}
		\alpha_1^0 & \alpha_2^0 & \cdots & \alpha_k^0 \\
		\alpha_1^1 & \alpha_2^1 & \cdots & \alpha_k^1 \\
		\vdots & \vdots & \ddots & \vdots \\
		\alpha_1^{2k-1} & \alpha_2^{2k-1} & \cdots & \alpha_k^{2k-1}
    \end{bmatrix}.
\end{align}
Note that $A_M$ is a Hankel matrix and $V_{\veca}$ is a Vandemonde matrix.

%\jiannote{add a english description of the algorithm.}

\begin{algorithm}[t]
\caption{Algorithm for the k-coin problem}\label{alg:dim1}
\begin{algorithmic}[1]
    \REQUIRE number of spikes $k$, noisy moments $M'(\cdot)$, noise level $\noise$ \\
    \ENSURE recovered spike distribution $\cmix$
    \STATE $\hvecc \leftarrow \arg \min_{\vecx \in \R^k} \| A_{M'} \vecx + b_{M'} \|_2^2 + \noise^2 \|\vecx\|_2^2$ 
    \label{line:dim1-hc}
    \hfill{$\Rightarrow\hvecc = [\hc_0, \cdots, \hc_{k-1}]^{\top} \in \R^k$}
    \STATE $\hveca \leftarrow \textrm{roots}(\sum_{i=0}^{k-1} \hc_i x^i + x^k)$ 
    \label{line:dim1-ha}
    \hfill{$\Rightarrow\hveca = [\ha_1, \cdots, \ha_k]^{\top} \in \C^k$}
    \STATE $\bveca \leftarrow \textrm{project}_{\simplex_1}(\hveca)$
    \label{line:dim1-ba}
    \hfill{$\Rightarrow\bveca = [\ba_1, \cdots, \ba_{k}]^{\top} \in \simplex_1^k$}
    \STATE $\tveca \leftarrow \bveca + \textrm{Noise}(\noise)$
    \label{line:dim1-ta}
    \hfill{$\Rightarrow\tveca = [\ta_1, \cdots, \ta_{k}]^{\top} \in \simplex_1^k$}
    \STATE $\hvecw \xleftarrow{O(1)-\textrm{approx}} \arg \min_{\vecx \in \R^k} \|V_{\tveca} \vecx - M'\|_2^2$ 
    \label{line:dim1-hw}
    \hfill{$\Rightarrow\hvecw = [\hw_1, \cdots, \hw_{k}]^{\top}  \in \R^k$}
    \STATE $\tvecw \leftarrow \hvecw / (\sum_{i=1}^{k} \hw_i)$
    \label{line:dim1-tw}
    \hfill{$\Rightarrow\tvecw = [\tw_1, \cdots, \tw_{k}]^{\top}  \in \Sigma_{k-1}$}
    \STATE $\tmix \leftarrow (\tveca, \tvecw)$ 
    \label{line:dim1-tmix}
    \hfill{$\Rightarrow\tmix \in \Spike(\simplex_1, \Sigma_{k-1})$}
    \STATE $\cvecw \leftarrow \arg \min_{\vecx \in \simplex_{k-1}} \tran(\tmix, (\tveca, \vecx))$
    \label{line:dim1-cw}
    \STATE $\cmix \leftarrow (\tveca, \cvecw)$ 
    \label{line:dim1-cmix}
    \hfill{$\Rightarrow\cmix \in \Spike(\simplex_1, \simplex_{k-1})$}
% \EndFunction
\end{algorithmic}
\end{algorithm}

Our algorithm is a variant of Prony's method \citep{prony1795}.
The pseudocode can be found in Algorithm~\ref{alg:dim1}. The algorithm takes the number of spikes $k$, the noisy moment vector $M'$ which $\|M' - M(\mix)\|_{\infty}\leq \noise$ and the moment accuracy $\noise$ as the input. 
We describe our algorithm as follows. 

Let $\vecc = [c_0, \cdots, c_{k-1}]^{\top} \in \R^k$ such that $\prod_{i=1}^{k} (x - \alpha_i) = \sum_{i=0}^{k-1} c_i x^i + x^k$ be the characteristic vector of locations $\alpha_i$.
We first perform a ridge regression to obtain $\hvecc$ in Line~\ref{line:dim1-hc}.
%, which we will show  is an estimator of $\vecc$.
We note that $A_{M(\mix)} \vecc + b_{M(\mix)} = \mathbf{0}.$ (see Lemma~\ref{lm:dim1-char-poly}). 
Hence, $\hvecc$ serves a similar role as $\vecc$
(note that $\hvecc$ is not necessarily close to $\vecc$ without the separation assumption).
From Line~\ref{line:dim1-ha} to Line~\ref{line:dim1-ta}, we aim to obtain estimations of the positions of the spikes, i.e., $\alpha_i$s.
We first solve the roots of polynomial $\sum_{i=0}^{k-1} \hc_i x^i + x^k$. 
For polynomial root findings, 
note that some roots may be complex without any separation assumption. 
\footnote{\citet{gordon2020sparse} proved that all roots
are real and separated under the minimal separation assumption.
}
Hence, we need to project the solutions back to $\simplex_1$ and 
inject small noise, ensuring that all values are distinct and still in $\simplex_1$. 
We note that any noise of size at most $\noise$ suffices.  
%We note any noise of size at most $\noise$ suffices. 
After recovering the positions of the spikes, we aim to recover the corresponding weights $\tvecw$
by the linear regression defined by the Vandemonde matrix $V_{\veca}$ in Line~\ref{line:dim1-hw} with weight normalization in Line~\ref{line:dim1-tw}.
We note that $\tvecw$ may still have some negative components. In Line~\ref{line:dim1-cw}, we 
find the closest $k$-spike distribution in $\Spike(\simplex_1, \simplex_{k-1})$, which is our final output.
The details for implementing the above steps in $O(k^2)$ time can be found in Appendix~\ref{subsec:1d-alg-detail}.
The performance guarantee of the algorithm is given by the following theorem.

\begin{theorem}
\label{thm:1dim-kspikecoin}
Let $\mix$ be an arbitrary $k$-spike distribution over $[0,1]$.
Suppose we have noisy moments $M'_{t}$ such that 
$|M_t(\mix)-M'_{t}|\leq (\epsilon/k)^{O(k)}$ for $0\leq t \leq 2k-1$. 
Then Algorithm~\ref{alg:dim1} outputs a mixture $\tmix$ such that $\tran(\tmix, \mix)\leq O(\epsilon)$ using $O(k^{2})$ arithmetic operations.
\end{theorem}

\vspace{-0.2cm}
\subsection{Proof Sketch}

In this section, we present a proof sketch of Theorem~\ref{thm:1dim-kspikecoin}. 
The complete proof of the theorem can be found in Appendix~\ref{subsec:1d-alg-analysis}.
Firstly, using the relationship between $\hvecc$ and $\hveca$, one can show that
\begin{align*}
     \sum_{i=1}^{k} w_i \prod_{j=1}^{k} (\alpha_i - \ha_j)^2
        = \sum_{i=1}^{k} w_i \left(\sum_{j=0}^{k-1} \hc_j \alpha_i^j + \alpha_i^k\right)^2 =  \sum_{i=0}^{k-1} (\hc_i - c_i) \Bigl(A_{M(\mix)} \hvecc + b_{M(\mix)}\Bigr)_i
\end{align*}
Since we use the ridge regression to get $\hvecc$ (Line~\ref{line:dim1-hc}), the norms of both $A_{M'} \hvecc + b_{M'}$ and $\hc$ are small. Note that $M'$ is close to $M(\mix)$, the norm of vector $A_{M(\mix)} \hvecc + b_{M(\mix)}$ is also small. Hence one can get
$$
\sum_{i=1}^{k} w_i \prod_{j=1}^{k} (\alpha_i - \ha_j)^2 \leq 2^{O(k)} \cdot \zeta.
$$
The above inequality shows that every ground truth spike $a_i \in \veca$ has a close spike $\ta_j \in \tveca$ where the closeness depends on the weight of the spike $w_i$. We note that this property tolerates permutation and small spikes weights, enabling us to analyze mixtures without separation.

Next, note we can decompose $M(\mix)=\sum_{i=1}^{k}w_i M(\alpha_i)$ where $M(\alpha_i)$ is the moment vector generated by the single-spike distribution located at $\alpha_i$.
By triangular inequality, the error of linear regression (Line~\ref{line:dim1-hw}) 
can be decomposed by
$\min_{\vecx} \|V_{\tveca}\vecx - M(\mix)\|_2 \leq \sum_{i=1}^k w_i \min_{\vecx}\|V_{\tveca}\vecx - M(\alpha_i)\|_2,$
Therefore, to bound $\|V_{\tveca}\hvecw - M(\mix)\|_2 \approx \min_{\vecx} \|V_{\tveca}\vecx - M(\mix)\|_2$, it is sufficient to prove $\min_{\vecx}\|V_{\tveca}\vecx - M(\alpha_i)\|$ for each $\alpha_i$ is bounded.
We use $\|V_{\tveca}\vecx^* - M(\alpha)\|_2$ for a
specific $\vecx^*$ as its upper bound.
$\vecx^*$ is chosen to be the vector in which the first $k$ rows of $V_{\tveca} \vecx^*$ and $M(\veca_i)$ are equal
($\vecx^*$ is the solution of a linear system). 
By binomial expansion, we can compute the $j$th row of $V_{\tveca} \vecx^*$ for large $j > k$ by
\begin{align*}
    (V_{\tveca} \vecx^*)_j = \sum_{p=1}^{j} \binom{j}{p} \alpha_i^{j-p} \sum_{t=1}x_t^* (\ta_t - \alpha_i)^p.
\end{align*}
Expressing $\vecx^*$ using the adjoint matrix, using Using Cramer's rule, and plugging it into the right-hand side of the above equality, the resulting term appears to
be difficult to bound
(see the third equation in the analysis of Lemma~\ref{lm:step3}).
Fortunately, by leveraging a matrix equation derived from the Shur polynomial, one can get the following nice equation for the inner summation,
\begin{align*}
    \sum_{t=1}x_t^* (\ta_t - \alpha_i)^p = \prod_{t=1}^k (\ta_t - \alpha_i) \cdot \sum_{s \subseteq (k)^{p-k}} \prod_{t=1}^{k}  (\ta_t - \alpha_i)^{s_t},
\end{align*}
where we use $(a)^{b} \subset \{0,\cdots,b\}^a$ to denote the set that contains all vector $s$ in which $\sum_{i=1}^{a} s_i = b$.
%Combining with the previous equation, we are able to bound $V_{\tveca} \vecx^*$. 
Since $(k)^{p-k}$ has no more than $2^{O(p-k)}$ elements, one can thus easily bound the approximation error by:
\begin{align*}
    \min_{\vecx} \|V_{\tveca}\vecx - M(\veca_i)\|_2 \leq 2^{O(k)} \cdot \prod_{j=1}^{k}|\alpha_i - \ta_j|
\end{align*}
from the right-hand side of the above equation.

Note that $\ta_j$ are $\ha_i$ close, combining all the above results gives $\|V_{\tveca}\hvecw - M(\mix)\|_{\infty} \leq 2^{O(k)} \cdot \sqrt{\zeta}$, indicating moment distance between the recovered distribution and the ground truth is small. Finally, one can reach the desired statement using the moment-transportation inequality (Lemma~\ref{lm:inequality1d}), which asserts that if the moment distance between two measures is sufficiently small, the transportation distance between them is also small.

\vspace{-0.3cm}

\section{Efficient Algorithms for Higher Dimensions}
\label{sec:high-dim-recover}

In this section, we solve the problem in higher dimensions.
We first present the algorithm for the two-dimensional problem.
The algorithm for higher dimensions takes both the one-dimensional and two-dimensional algorithms as subroutines.

\vspace{-0.2cm}
\subsection{An Efficient Algorithm for the Two-Dimensional Problem}

We first generalize the one-dimensional algorithm described in Section~\ref{sec:1d-recover} to the dimension of two.
Let $\mix:=(\veca, \vecw) \in \Spike(\simplex_2,\simplex_{k-1})$ 
be the underlying mixture where
$\veca = [\veca_1, \cdots, \veca_k]^{\top}$ for two-dimensional points $\veca_{i} = (\alpha_{i,1}, \alpha_{i,2})$ and weight vector $\vecw = [w_1, \cdots, w_k]^{\top} \in \simplex_{k-1}$.
The true moments is defined by $M_{i,j}(\mix) = \sum_{t=1}^k w_t \alpha_{t,1}^i \alpha_{t,2}^j$. 
We assume the input noisy moments $M'_{i,j}$ satisfies $|M'_{i,j} - M_{i,j}(\mix)| \leq \noise$ for every $0\leq i, j, i+j \leq 2k-1$. We further assume that $M'_{0,0}=M_{0,0}(\mix) = 1$ and noise level $\noise \leq 2^{-\Omega(k)}$.

The key idea for our algorithm is that,
a distribution supported in $\R^2$ can be mapped to a distribution supported in the complex plane $\C$. In particular, let complex simplex $\simplex_{\C} = \{a+b\imag \mid (a,b) \in \simplex_2\}$.
We denote $\vecb = [\beta_{1}, \cdots, \beta_{k}]^{\top} := [\alpha_{1,1} + \alpha_{1,2} \imag, \cdots, \alpha_{k,1} + \alpha_{k,2} \imag]^{\top} \in \simplex_{\C}^k$ and define $\pmix := (\vecb, \vecw) \in \Spike(\simplex_{\C}^k, \simplex_{k-1})$ to be the complex mixture corresponding to $\mix$. 
%We note that $|\beta_i| \leq 1$ because our restriction on $\simplex_2$. 
The corresponding moments of $\pmix$ can thus be defined as $G_{i,j}(\pmix) = \sum_{t=1}^{k} w_t (\beta_t^{\dagger})^i \beta_t^j$, where $\beta_t^{\dagger}$ is the complex conjugate of $\beta_t$.
We note $G_{i,j}$ can be computed from $M_{i,j}$ (see Lemma~\ref{lm:dim2-complex-corr}).
Therefore, the task reduces to the recovery of the complex mixture $\pmix$ given noisy moments $G'_{i,j}$.
The algorithm and the proof of the following theorem are similar to that in Section~\ref{sec:1d-recover}.
Handling complex numbers requires several minor changes, and we present the full details in Appendix~\ref{sec:2d-recover}.
For the analysis, we also need to generalize the moment-transportation inequality to the complex domains (see Appendix~\ref{subsec:2dmoment}).

\vspace{-0.2cm}

\begin{theorem}
\label{thm:2dim-kspikecoin}
(The Two-dimensional Problem)
Let $\mix\in \Spike(\simplex_2,\simplex_{k-1})$.
Suppose we have noisy moments $M_{i,j}$ ($0\leq i, j, i+j \leq 2k-1$) up to precision $(\epsilon/k)^{O(k)}$. 
We can obtain a mixture $\tmix$ such that $\tran(\tmix, \mix)\leq O(\epsilon)$ using only $O(k^3)$ arithmetic operations.
\end{theorem}

\vspace{-0.3cm}

\subsection{An Efficient Algorithm for Dimension $d \geq 3$}
\label{subsec:algoforhigherdim}

Now, we present our algorithm for higher dimensions.
Let $\mix:=(\veca, \vecw) \in \Spike(\simplex_{d-1}, \simplex_{k-1})$
be the underlying mixture where $\veca = [\veca_1, \cdots, \veca_k]^{\top} \in \mathbb{R}^{k\times d}$ such that $\veca_i = [\alpha_{i,1}, \cdots, \alpha_{i,d}]^{\top}$ and $\vecw = [w_1, \cdots, w_k]^{\top}$. We denote $\veca_{:,i} = [\alpha_{1,i}, \cdots, \alpha_{k,i}]$ as the vector of the $i$th coordinates of the 
spikes and also write $\tveca = [\veca_{:,1}, \cdots, \veca_{:,k}]^{\top}$.

Since the number of different moments is exponential in $d$, 
we consider the setting in which one can 
access the noisy moments of some linear maps of $\mix$ onto lower-dimensional subspaces, as in previous works \citep{rabani2014learning,li2015learning}
(such noisy moments can be easily obtained in applications such as topic modeling, see Section~\ref{sec:topic}, or in learning Gaussian mixtures, see
Section~\ref{sec:Gaussian}). 
For some $\mix = (\veca, \vecw) \in \Spike(\R^d, \simplex_{k-1})$ and a vector $\vecr \in \R^d$, we denote $\proj{\vecr}(\mix) := (\veca\vecr, \vecw) \in \Spike(\R, \simplex_{k-1})$ as the projected distribution of $\mix$ along vector $\vecr$. For a matrix $R = [\vecr_1, \cdots, \vecr_p]\in \R^{d \times p}$, we also denote $\proj{R}(\mix) := (\veca R, \vecw)  \in \Spike(\R^p, \simplex_{k-1})$ as the projected distribution of $\mix$ along multiple dimensions $\vecr_1, \cdots, \vecr_p$. In this section, we consider mappings to one-dimensional and two-dimensional subspaces, which suffice for our purposes.

The idea of our algorithm is as follows:
We first generate a random vector $\vecr$. 
Then with some constant probability, the distance between spikes is roughly kept after the projection along $\vecr$. 
This enables us to recover a projected measure $\tpmix$ along $\vecr$ of the discrete mixture by running the one-dimensional algorithm on the projected moments.
Next, we try to assign every spike of the projected measure a coordinate in $\R^d$.
This is done by running the two-dimensional algorithm over the plane spanned by $\vecr$ and unit vector $e_t$ for each dimension $t \in [d]$ to get $\hpmix_t$, and assigning the $t$-th coordinate of each spike in $\tpmix$ the coordinate of the closet spike in $\hpmix_t$.
We use a geometric argument to show the recovered distribution is close to the ground truth $\mix$ in terms of the transportation distance (see Lemma~\ref{lm:dimn-step3.5}). 
The algorithm and the analysis are deferred to Appendix~\ref{sec:high-dim-recover-full}.

\begin{theorem}\label{thm:highdim-kspike}
    Let $\mix$ be an arbitrary $k$-spike mixture supported in $\Delta_{d-1}$. Suppose we can access noisy $\bt$-moments of 
    $\proj{R}(\mix)$, with precision $(\epsilon/(dk))^{O(k)}$,
    where $\proj{R}(\mix)$ is the projected measure obtained by
    applying linear transformation $R$ (with $\|R\|_\infty\leq 1$) of $\mix$ onto any $h$-dimensional subspace we choose, 
    for all $\|\bt\|\leq K=2k-1$ and $h=1,2$
    (i.e., lines and 2d planes). 
    We can construct a $k$-spike mixture $\tmix$ such that $\tran(\tmix, \mix)\leq O(\epsilon)$ using only $O(dk^3)$ arithmetic operations with probability at least $0.99$.
\end{theorem}

\vspace{-0.2cm}

\section{Applications}

\subsection{Applications to Topic Models}
\label{sec:topic}

This section discusses the application of our results to topic models. 
Due to space limitations, the full details are deferred to Appendix~\ref{sec:topic-full}.
Consider the “bag of words” model in which we take each document as an unordered multiset of words.
Let $d$ be the number of different words in the document.
We assume there are $k$ topics $\{\veca_{1}, \veca_{2}, \cdots \veca_{k}\}$ that are distributed over $[d]$. 
Each document is generated by first selecting a topic $i\in\simplex_{d-1}$ from the mixture $\mix$, and then sampling $K$ i.i.d. words according to $\veca_{i}$ from this topic.
Here, $K$ is referred to as {\em the snapshot number} and we call such a sample a {\em $K$-snapshot} of $p$.
Our goal is to recover $\mix\in \Spike(\simplex_{d-1},\simplex_{k-1})$ that is a discrete distribution over $k$ pure topics.
If $d\gg k$,
we can use the dimension reduction technique in \citet{li2015learning} to reduce the dimension $d$ to $O(k)$, using $\tilde{O}(\epsilon^{-6}  k^3 d)$ 1-snapshots and $\tilde{O}(\epsilon^{-4} k^2d)$ 2-snapshots, respectively.
Then, we can apply Theorem~\ref{thm:highdim-kspike} to derive the required
moment accuracy. Translating the moment accuracy to the sample complexity is standard. In sum, we can obtain the following result.

\begin{theorem}\label{thm:highdimtopic}
There is an algorithm that can learn an arbitrary $k$-spike mixture supported in $\simplex_{d-1}$
for any $d$ within $L_1$ transportation distance $O(\epsilon)$ with probability at least $0.99$ using
$\poly(d,k,\frac{1}{\epsilon})$, $\poly(d,k,\frac{1}{\epsilon})$, 
many $1$-,$2$-,snapshots and $(k/\epsilon)^{O(k)}$ many $(2k-1)$-snapshots.
\end{theorem}

\vspace{-0.4cm}

\subsection{Applications to Gaussian Mixture Learning}
\label{sec:Gaussian}

This section shows how to leverage our results
for sparse moment problems to obtain improved algorithms for learning Gaussian mixtures. The full details can be found in Appendix~\ref{sec:Gaussian-full}.
We consider the following setting
studied in \citet{wu2020optimal, doss2020optimal}.
We parameterise a $k$-Gaussian mixture in $\mathbb{R}^d$ as $\mix_N = (\veca, \vecw, \Sigma)$. Here, $\veca = \{\veca_1, \veca_2, \cdots, \veca_k\}$ and $\vecw = \{w_1, w_2, \cdots, w_k\} \in \Delta_{k-1}$ where $\veca_i \in \mathbb{R}^d$ and $w_i \in [0, 1]$ represents the mean and weight of $i$th component.
We assume all $k$ mixture components share a common covariance matrix $\Sigma \in \mathbb{R}^{d\times d}$ that is known in advance and the maximal eigenvalue of $\Sigma$ is bounded by a constant. 
We focus on parameter learning, that is, to learn the parameter $\veca$ and $\vecw$ given known covariance matrix $\Sigma$ and a set of i.i.d. samples from $\mix_N = \sum_{i=1}^{k} w_i N(\veca_i, \Sigma)$. We have the following result:

\begin{theorem} 
\label{thm:ndim-Gaussian}
There is an algorithm that can learn an arbitrary $d$-dimensional $k$-spike Gaussian mixture (with known common covariance matrix $\Sigma$) within transportation distance $O(\eps)$ with probability at least $0.99$ 
using $(k/\eps)^{O(k)} + \poly(1/\eps, d, k)$ samples.
\end{theorem}

\vspace{-0.2cm}

\section{Concluding Remarks}
We provide efficient algorithms for learning the $k$-mixture over discrete distributions from noisy moments. 
We measure the accuracy in terms of transportation distance, and our analysis is independent of the minimal separation.
The techniques used in our analysis for the one-dimensional case may be useful in the perturbative analysis of other problems involving the Hankel matrix or the Vandermonde matrix (such as super-resolution).

Our problem is a special case of learning mixtures of product distribution \citep{gordon2021source}, which is further a special case of the general problem of learning mixtures of Bayesian networks \citep{gordon2021identifying}, which has important applications in causal inferences. In fact, the algorithms in \citet{gordon2021source,gordon2021identifying} used the algorithm for sparse Hausdorff moment problem as a subroutine. It would be interesting to see if our techniques can be useful in these problems.

\section{Acknowledgement}

We would like to thank Yuval Rabani and Leonard Schulman for their insightful 
discussions and anonymous reviewers for many helpful comments.  In particular, Leonard Schulman pointed out that Lemma~\ref{lm:dim1-step3-matrix}
is related to the Schur polynomial. 
JL also would like to thank Xuan Wu for several discussions around the 2015
and a version of Lemma~\ref{lm:inequality1d} was proved back then, independent from \citet{wu2020optimal}.
The authors are supported in part  by the National Natural Science Foundation of China Grant 62161146004.

% \acks{
% We would like to thank Yuval Rabani and Leonard Schulman for their insightful 
% discussions.  In particular, Leonard Schulman pointed out that Lemma~\ref{lm:dim1-step3-matrix}
% is related to Schur polynomial. 
% JL also would like to thank Xuan Wu for several discussions around the 2015
% and a version of Lemma~\ref{lm:inequality1d} was proved back then, independent from \citet{wu2020optimal}.
% }

\newpage

{\centering {\huge\textsc{Appendix}}}
\appendix

\section{Transportation Distance for Non-Probability Measures}
\label{app:transfornonprob}

Let $\lip$ be the set of 1-Lipschitz functions on $\R^d$, i.e.,
$\lip:=\{f: \R^d\rightarrow \R \mid |f(x)-f(y)|\leq \|x-y\|_1,  \forall x,y\in \R^d\}$.
We need the following important theorem by Kantorovich and Rubinstein (see e.g., \citet{dudley2002real}),
which states the dual formulation of transportation distance:
\begin{align}
\label{eq:trans}
\tran(P,Q)=\sup\left\{\int f \d(P-Q): f\in \lip\right\}.
\end{align}

The dual formulation \eqref{eq:trans} can be generalized to non-probability measures. 
For two signed measures $P$ and $Q$ defined over $\R^d$ with $\int P \d x = \int Q \d x = 1$, we still can use \eqref{eq:trans} to compute its transportation distance.

We now introduce the equivalent primal formulation as follows:
For some signed measure $R$, we define measure $R^+$ as its positive components that $R^+(x) = \max\{R(x), 0\}$. According to the duality formulation \eqref{eq:trans}, one can see $\tran(P, Q) = \tran((P-Q)^+, (Q-P)^+)$ where both $(P-Q)^+$ and $(Q-P)^+$ are non-negative. From the primal formulation of transportation distance, we conclude another definition for generalized transportation distance:
\begin{align}
\label{eq:signedtrans}
\tran(P, Q) = \inf \left \{ \int \|x - y\|_1 \d \mu(x, y) : \mu \in M((P-Q)^+, (Q-P)^+) \right \}.
\end{align}

We also need to define the transportation distance over
complex measures.
Denote $\mathbb{B}_{\C} = \{\alpha \in \C : |\alpha|\leq 1\}$.
For some complex measure $P$ that $\int_{\C}  \textrm{d}P(x) = 1$. We denote the real and imaginary components of $P$ as $P^{\textrm{r}}$ and $P^{\textrm{i}}$ respectively that $P(x) = P^{\textrm{r}}(x) + \imag P^{\textrm{i}}(x)$ for $P^{\textrm{r}}(x), P^{\textrm{i}}(x) \in \mathbb{R}$. We generalize transportation distance to complex weights by
\begin{align}
\label{eq:complextrans}
\tran(P, Q) = \tran(P^{\textrm{r}}, Q^{\textrm{r}}) + \tran(P^{\textrm{i}}, Q^{\textrm{i}}).
\end{align}
% Our technique requires to find the roots of polynomial equations, and the roots may be complex. Hence, we need to generalize the definition of transportation distance to complex numbers.
% A complex measure $P$ is a function $X\rightarrow \C$ such that $\int P(x) \d x = 1$.
% Denote the real and imaginary components of $P$ as $P_r$ and $p_i$ respectively, such that $P(x) = P_r(x) +  i p_i(x)$.
% The generalized transportation distance is defined as
% \begin{align*}
% \tran(P, Q):= \inf\left\{\int d(x, y) \d(\mu_r + \mu_i)(x, y): \mu_r \in M(P_r, Q_r), \mu_i \in M(p_i, Q_i)\right\}
% \end{align*}
% where $d(x, y) = |x-y|$ and $M( \cdot, \cdot )$ is the coupling between two distributions.
We also have this transportation distance in a dual form:
\begin{align*}
\tran(P, Q) &= \sup\left\{\int f^{\textrm{r}} \d(P^{\textrm{r}}-Q^{\textrm{r}}): f^{\textrm{r}}\in \lip\right\} + \sup\left\{\int f^{\textrm{i}}  \d(P^{\textrm{i}}-Q^{\textrm{i}}): f^{\textrm{i}}\in \lip\right\} \\
&= \sup \left \{  \textrm{real} \left(\int (f^{\textrm{r}} - \imag f^{\textrm{i}})\d (P - Q) \right ) :  f^{\textrm{r}}, f^{\textrm{i}} \in \lip\right\} \\
&\leq 2\sup \left \{  \int f\d (P - Q)  :  f \in \lip_{\C}\right\}
\end{align*}
where $\lip_{\C}:=\{f: \C\rightarrow \C \mid |f(x)-f(y)|\leq |x-y| \text{ for any }x,y\in \C\}$ is the complex 1-Lipschitz functions on $\C$, and the last inequality holds since we can assign $f(x) = (f^{\textrm{r}}(x) - \imag f^{\textrm{i}}(x)) / 2$.

\section{Moment-Transportation Inequalities}

\label{app:momenttransineq}

Our analysis needs to bound the transportation distance
by a function of moment distance.
We first present a moment-transportation inequality
in one dimension (Lemma~\ref{lm:inequality1d}),
and then generalize the result to complex numbers and higher dimensions.
The first such inequality in one dimension is obtained
in \citet[(Lemma 5.4)]{rabani2014learning} (with worse parameters).
Lemma~\ref{lm:inequality1d} is firstly proved by
\citet{wu2020optimal} in the context of learning Gaussian mixtures.
%In particular, they show 
%$\tran(\mix,\mix')\leq O(k\Mdis_K(\mix,\mix')^{\frac{1}{4k-2}})$.
%Moreover, it does not seem easy to generalize the proof in \citet{rabani2014learning}
%to higher dimensions, since it
%involves several tricky calculation and
%is inherently one dimensional.
Here, we prove the theorem in our notations for completeness,
and later we extend the proof to complex domains in Section~\ref{subsec:2dmoment}.

\subsection{A One-Dimensional Moment-Transportation Inequality}
\label{subsec:1dmoment}

We allow more general mixtures with spikes in $[-1,1]$ and negative weights.
This slight generalization will be useful since some intermediate coordinates and weights may be negative. Recall that $\signedsimplex_{k-1}=\{\vecw=(w_1,\cdots, w_k)\in \R^k \mid \sum_{i=1}^k w_i=1\}$.

\begin{lemma}
\label{lm:inequality1d} (\citet{wu2020optimal})
For any two mixtures with $k$ components $\mix,\mix' \in \Spike([-1, 1],\Sigma_{k-1})$, it holds that
$$
\tran(\mix,\mix')\leq O(k\Mdis_K(\mix,\mix')^{\frac{1}{2k-1}}).
$$
\end{lemma}
%\jiannote{TODO: need to extend the domain beyond $\Spike^1_k$
%to accommodate signed measures $\Spike(\simplex_1,\signedsimplex_k)$.}

We begin with some notations.
Let $\supp = \supp(\mix) \cup \supp(\mix')$ and let $n=|\supp|$ ($n\leq 2k$).
Arrange the points
in $\supp$ as $\veca=(\alpha_1,\cdots,\alpha_n)$ where $\alpha_1< \cdots <\alpha_n$.

\begin{definition}
Let $\veca=(\alpha_1,\cdots,\alpha_n)$ such that $-1\leq \alpha_1< \cdots< \alpha_n\leq 1$.
Let $\calP_{\veca}$ denote the set of polynomials of degree at most $n-1$ and 1-Lipschitz over 
the discrete points in $\veca$, i.e.,
$$
\calP_{\veca}=\{f\mid \deg f\leq n-1,f(\alpha_1)=0; \, |f(\alpha_i)-f(\alpha_j)|\leq|\alpha_i-\alpha_j|\,\, \forall i,j\in [n].\}
$$
\end{definition}
Note that we do not require $f$ to be 1-Lipschitz over the entire interval $[-1,1]$.
%If $\mix,\mix'$ are both probability distributions, then
%$\int 1\d(\mix-\mix')=0$.
From the dual formulation of transportation distance (see e.g., \citet{dudley2002real}), we have the following proposition:
\begin{proposition}
\label{pro:simple}
$
\tran(\mix,\mix')=\sup\left\{\int f \d(\mix-\mix'): f\in \lip\right\}=\sup\left\{\int p \d(\mix-\mix'): p\in \calP_{\veca}\right\}
$.
\end{proposition}

We can view $\calP_{\veca}$ as a subset in the linear space of all polynomials of degree at most $n-1$.
In fact, $\calP_{\veca}$ is a convex polytope
(a convex combination of two  polynomials being 1-Lipschitz over $\veca$
is also a polynomial that is 1-Lipschitz over $\veca$).
The {\em height} of a polynomial $\sum_i c_i x^i$ is the maximum absolute coefficient $\max_i |c_i|$.
As we will see shortly, the height of a polynomial in $\calP_{\veca}$ is related to the required moment accuracy, and we need the height to be upper bounded by a value independent of the minimum separation
$\minsep$. However, a polynomial in $\calP_{\veca}$ may have a very large height (depending on the inverse of the minimum 
separation $1/|\alpha_{i+1}-\alpha_i|$).
This can be seen from the Lagrangian interpolation formula:
$p(x) := \sum_{j=1}^{n} p(\alpha_j) \ell_j(x)$
where $\ell_j$ is the Lagrange basis polynomial
$
\ell_j(x) := \prod_{1\le m\le n, m\neq j} (x-\alpha_m)/(\alpha_j-\alpha_m).
$
To remedy this, one can show that for any polynomial in $\calP_{\veca}$, there is an approximate polynomial with bounded height:

\begin{lemma}
\label{lm:decomposition}
For any polynomial $p(x)\in \calP_{\veca}$ and $\eta>0$, 
there is a polynomial $p_{\eta}(x)$ of degree at most $n-1$ such that the following properties hold:
\begin{enumerate}
\item
$|p_{\eta}(\alpha_i)-p(\alpha_i)|\leq 2\eta$, for any $i\in [n]$.
\item
The height of $p_{\eta}$ is at most $n2^{2n-2}(\frac{n}{\eta})^{n-2}$.
\end{enumerate}
\end{lemma}
 
Before proving Lemma~\ref{lm:decomposition}, we show Lemma~\ref{lm:inequality1d} can be easily derived from
Proposition~\ref{pro:simple}
and Lemma~\ref{lm:decomposition}. 

\begin{proofof}{Lemma~\ref{lm:inequality1d}}
Fix any constant $\eta>0$.
Let $\noise=\Mdis_K(\mix,\mix')$
\begin{align*}
\tran(\mix,\mix')&=\sup_{f\in \lip} \int f \d(\mix-\mix')
=\sup_{p\in \calP_{\veca}}\int p \d(\mix-\mix')  & \hfill(\text{Proposition~\ref{pro:simple}})\\
&\leq 4\eta+\sup_{p\in \calP_{\veca}} \int p_{\eta} \d(\mix-\mix')  & (\text{Lemma~\ref{lm:decomposition}})\\
&\leq 4\eta+\sup_{p\in \calP_{\veca}} \sum_{i=1}^{n-1} n2^{2n-2}\left(\frac{n}{\eta}\right)^{n-2}\Mdis_K(\mix,\mix')
& (\text{Lemma~\ref{lm:decomposition}})\\
&\leq 4\eta+2^{4k}(2k)^{2k}\frac{\noise}{\eta^{2k-2}} & (n\leq 2k)
\end{align*}
By choosing $\eta=16k\noise^{\frac{1}{2k-1}}$,
we conclude that $\tran(\mix,\mix')\leq O(k\noise^{\frac{1}{2k-1}})$.
\end{proofof}

%\subsubsection{Proof of Lemma~\ref{lm:decomposition}}

Next, we prove Lemma~\ref{lm:decomposition}.
%An arbitrary polynomial in $\calP_{\veca}$ is not easy to deal with.
We introduce the following set $\calF_{\veca}$ of special polynomials.
They are in fact the extreme points of the polytope $\calP_{\veca}$.

\begin{definition}
Let $\veca=(\alpha_1, \cdots, \alpha_n)$ such that $-1\leq \alpha_1<\cdots<\alpha_n\leq 1$.
Let
$$
\calF_{\veca}=\{p\mid \deg p\leq n-1, p(\alpha_1)=0; \,|p(\alpha_{i+1})-p(\alpha_i)|=|\alpha_{i+1}-\alpha_i|\,\, \forall i\in [n-1]\}
$$
\end{definition}

It is easy to see that $\calF_{\veca}\subset \calP_{\veca}$  and $|\calF_{\veca}|=2^{n-1}$.
Now, we modify each polynomial
$p\in \calF_{\veca}$ slightly, as follows.
% \eat{
% Define $v(f)=(f(x_2)-f(x_1),\cdots,f(x_n)-f(x_{n-1}))$
% and use $v(f)_i$ to denote $f(x_{i+1})-f(x_i)$.
% For $\eta>0$, let
% $f_{\eta}=(f_1, f_2 ,\cdots, f_{n-1})$, where $f_i$ is given by $f_1=v(f)_1$ and for any $i=2, \cdots, n-1$,
% we let $f_i=\left\{\begin{array}{cc}
%              f_{i-1}-v(f)_i & \text{if}\;\;\;\; v(f)_if_{i-1}<0, x_{i+1}-x_i\leq\frac{\eta}{n};\\
%              f_{i-1}+v(f)_i & \text{otherwise.}
%           \end{array}
%           \right.$
% }
Consider the intervals $[\alpha_1,\alpha_2], \cdots, [\alpha_{n-1},\alpha_n]$.
If $\alpha_i-\alpha_{i-1}\leq \eta/n$, we say the interval $[\alpha_{i-1}, \alpha_i]$ is 
a {\em small interval}.
Otherwise, it is {\em large}.
We first merge all consecutive small intervals.
For each resulting interval, we merge it with the large interval to its right.
Note that we never merge two large intervals together.
Let $S=\{\alpha_{i_1}=\alpha_1, \alpha_{i_2}, \cdots, \alpha_{i_m}=\alpha_n\}\subseteq \veca$
be the endpoints of the current intervals.
It is easy to see the distance between any two points in $S$
is at least $\eta/n$ (since the current interval contains exactly one original large interval).
%Without loss of generality, assume $x_1$ and $x_n$ are in $S$.
%\footnote{If $x_n-x_1<\eta/n$, the problem is quite trivial.
%This point will be evident soon.}
Define a continuous piecewise linear function $L:[\alpha_1,\alpha_n]\rightarrow \R$ as follows:
(1)
$L(\alpha_1)=p(\alpha_1)=0$.
(2)
The breaking points are the points in $S$;
(3)
Each linear piece of $L$ has slope either $1$ or $-1$;
for two consecutive breaking points $\alpha_{i_j}, \alpha_{i_{j+1}}\in S$,
if $p(\alpha_{i_j})>p(\alpha_{i_{j+1}})$, the slope of the corresponding piece is $-1$.
Otherwise, it is $1$.
%\item
%If $\alpha_n\not\in S$, the slope of the last piece (defined over $[\alpha_{i_m}, \alpha_n]$) is $1$.

\begin{lemma}
\label{lm:close}
For any $\alpha_i\in \veca$, $|L(\alpha_i)-p(\alpha_i)|\leq 2\eta$.
\end{lemma}
\begin{proof}
We prove inductively that $|L(\alpha_i)-p(\alpha_i)|\leq 2i\eta/n$.
The base case ($i=1$) is trivial.
Suppose the lemma holds true for all value at most
$i$ and now we show it holds for $i+1$.
There are two cases.
If $\alpha_{i+1}-\alpha_{i}\leq \eta/n$,
we have
\begin{align*}
|L(\alpha_{i+1})-p(\alpha_{i+1})|& \leq |L(\alpha_{i+1})-L(\alpha_{i})|+|L(\alpha_{i})-p(\alpha_{i})|+|p(\alpha_{i})-p(\alpha_{i+1})| \\
&= 2|\alpha_{i+1}-\alpha_{i}|+|L(\alpha_{i})-p(\alpha_{i})|\leq 2(i+1)\eta/n.
\end{align*}
If $\alpha_{i+1}-\alpha_{i}> \eta/n$, we have $\alpha_{i+1}\in S$ (by the merge procedure).
Suppose $\alpha_j$ ($\alpha\leq i$) is the point in $S$ right before $\alpha_{i+1}$.
We can see that (by the definition of $p$):
$$
|p(\alpha_j)-p(\alpha_{i+1})| \in |\alpha_{i+1}-\alpha_{i}|\pm |\alpha_{i}-\alpha_{j}| \subseteq |\alpha_{i+1}-\alpha_{i}|\pm (j-i)\eta/n.
$$
In addition, we have 
\begin{align*}
|L(\alpha_{i+1})-p(\alpha_{i+1})|& \leq |L(\alpha_{i+1})-L(\alpha_{j})+ p(\alpha_{j})-p(\alpha_{i+1})|+|L(\alpha_{j})-p(\alpha_{j})| \\
&\leq 2(i+1-j)\eta/n +2j\eta/n\leq 2(i+1)\eta/n,
\end{align*}
where second inequality holds since $L(\alpha_{i+1})-L(\alpha_{j})=\alpha_{i-1}-\alpha_j$ when $p(\alpha_{i+1})>p(\alpha_j)$ and 
$L(\alpha_{i+1})-L(\alpha_{j})=\alpha_{j}-\alpha_{i+1}$ when $p(\alpha_{i+1})\leq p(\alpha_j)$.
\end{proof}

Let $T_\eta(p)$ be the polynomial with degree $\leq n-1$ that interpolates
$L(\alpha_i)$, (i.e. $T_{\eta}(p)(\alpha_i)=L(\alpha_i)$ for all $i\in [n]$. )
$T_\eta(p)$ is unique by this definition and
by Lemma~\ref{lm:close}, $|T_\eta(p)(\alpha_i)-p(\alpha_i)|\leq 2i\cdot\frac{\eta}{n}\leq 2\eta$.
Now, we show the height of $T_\eta(p)$ can be bounded, independent of the
minimum separation $|\alpha_{i+1}-\alpha_i|$, due to its special structure.

\begin{lemma}
\label{lm:robustpolynomial}
For any polynomial $p(x) \in \calF_{\veca}$, define
polynomial $T_\eta(p)$ as the above.
Suppose $T_\eta(p)=\sum_{i=0}^{n-1} c_i x^i$.
We have $|c_t|\leq n2^{2n-2}(\frac{n}{\eta})^{n-2}$ for $t\in [n-1]$.
\end{lemma}

\begin{proof}
The key to the proof is Newton's interpolation formula, which we briefly review here
(see e.g., \citet{hamming2012numerical}).
Let $F[\alpha_1,\cdots,\alpha_i]$ be the $i$th order difference of $T_\eta(p)$, which can be defined recursively:
$$
F[\alpha_t,\alpha_{t+1}] = \frac{T_{\eta}(p)(\alpha_{t+1}) - T_{\eta}(p)(\alpha_{t})}{\alpha_{t+1}-\alpha_t}, \cdots, 
$$
$$
F[\alpha_t,\cdots,\alpha_{t+i}] = \frac{F[\alpha_{t+1},\cdots,\alpha_{t+i}] - F[\alpha_{t},\cdots,\alpha_{t+i-1}]}{\alpha_{t+i}-\alpha_{i}}.
$$
Then, by Newton's interpolation formula, we can write that
$$
T_\eta(p)(x)=F[\alpha_1,\alpha_2](x-\alpha_1)+F[\alpha_1,\alpha_2,\alpha_3](x-\alpha_1)(x-\alpha_2)+\cdots+F[\alpha_1,\cdots,\alpha_n]\prod_{i=1}^{n-1}(x-\alpha_i).
$$
By the definition of $T_\eta(p)$,
we know that every $2$nd order difference of $T_\eta(p)$ is either $1$ or $-1$.
Now, we show inductively that the absolute value of
the $i$th order difference absolute value is at most $2^{i-2}(\frac{n}{\eta})^{i-2}$ for any $i=3,\cdots,n$.
The base case is simple:
$F[\alpha,\alpha_{t+1}]=\pm 1$.
%First, we prove it for $i=3$.
%If $F[\alpha_t,\alpha_{t+1}]=F[\alpha_{t+1},\alpha_{t+2}]$, then $F[\alpha_t,\alpha_{t+1},\alpha_{t+2}]=0$.
%Otherwise, we have $\alpha_3-\alpha_2>\frac{\eta}{n}$, by the property of $T_\eta(f)$.
%Thus, $\alpha_3-\alpha_1>\frac{\eta}{n}$ and
%$$
%|F[\alpha_1,\alpha_2,\alpha_3]|\leq \frac{1-(-1)}{\eta/n}=\frac{2 n}{\eta}.
%$$
Now, we prove it for $i\geq 3$.
We distinguish two cases.
\begin{enumerate}
\item
If $\alpha_{t+i-1}-\alpha_t\leq \eta/n$,
all $\alpha_t, \alpha_{t+1}, \cdots, \alpha_{t+i-1}$ must belong to the same segment of $L$
(since all intervals in this range are small, thus merge together).
Therefore, we can see that
$F[\alpha_t,\alpha_{t+1}]=F[\alpha_{t+1},\alpha_{t+2}]=\cdots=F[\alpha_{t+1},\alpha_{t+i-1}]$,
from which we can see any 3rd order difference (hence the 4th, 5th, up to the $i$th)
in this interval is zero.
\item
Suppose $\alpha_{t+i-1}-\alpha_t>\eta/n$.
By the induction hypothesis, we have that
\begin{align*}
|F[\alpha_t,\cdots,\alpha_{t+i-1}]|&=\left|\frac{F[\alpha_t,\cdots,\alpha_{t+i-2}]-F[\alpha_{t+1},\cdots,\alpha_{t+i-1}]}{\alpha_{t+i-1}-\alpha_t}\right| \\
&\leq 2\cdot2^{i-3}\left(\frac{n}{\eta}\right)^{i-3}\cdot \frac{n}{\eta}
\leq 2^{i-2}\left(\frac{n}{\eta}\right)^{i-2}.
\end{align*}
\end{enumerate}
Also since $\alpha_j\in [-1,1]$, the absolute value of the coefficient of $x^t$ in
$\prod_{j=1}^i(x-\alpha_j)$ is less than $2^n$.
So, finally we have $c_i\leq n2^{2n-2}(\frac{\eta}{n})^{n-2}$.
\end{proof}

\begin{proofof}{Lemma~\ref{lm:decomposition}}
For every polynomial $p(x) \in \calP_{\veca}$.
Let $g_i=(p(\alpha_{i+1})-p(\alpha_i))/(\alpha_{i+1}-\alpha_i)$ for $i=[n-1]$.
So $-1\leq g_i\leq 1$, implying that the point $\vecg=(g_1,\cdots,g_{n-1})$ is in the $n-1$ dimensional hypercube.
Hence, by Carathedory theorem, there are $\lambda_1,\cdots,\lambda_{n}\geq 0, \lambda_1+\cdots+\lambda_{n}=1$,
such that $\vecg=\sum_{i=1}^{n}\lambda_i \vecq_i$ where $\vecq_1,\cdots,\vecq_{n}$ are $n$ vertices of the hypercube,
i.e., $\vecq_i\in \{-1,+1\}^{n-1}$.
For each $\vecq_i$, we define $p_i(x)$ be the polynomial with degree at most $n-1$
and
$$
p_i(\alpha_1)=0,\;
p_i(\alpha_m)=\sum_{j=1}^{m-1} q_{i,j}(\alpha_{j+1}-\alpha_j).
$$
where $q_{i,j}$ is the $j$th coordinate of $\vecq_i$.
We can see $p_i\in \calF_{\veca}$.
It is easy to verify that
$
p(x)=\sum_{i=1}^{n} \lambda_i p_i(x),
$
since both the LHS and RHS take the same values at $\{\alpha_1,\cdots, \alpha_n\}$.
For $\eta>0$, we define $p_{\eta}(x)$ as $p_{\eta}=\sum_{i=1}^{n} \lambda_i T_\eta(p_i)$.
We know $|T_\eta(p_i)(\alpha_j)-p_i(\alpha_j)|\leq 2\eta$ for all $i\in [n]$ and $j\in [n]$
and the height of each $T_\eta(p_i)$ is at most $n2^{2n-2}(\frac{n}{\eta})^{n-2}$.
So, $p_{\eta}(x)$ satisfies the properties of the lemma.
\end{proofof}

\subsection{Moment-Transportation Inequality over Complex Numbers}
\label{subsec:2dmoment}

In this subsection, we extend Lemma~\ref{lm:inequality1d} to complex numbers.
We allow the mixture components to be complex numbers, and
the mixture weight $w_i$ can also be complex numbers.
Recall in the complex domain,
the definition of transportation distance is extended according to \eqref{eq:complextrans}.

We first see that $n$ complex numbers in $\C$ can be clustered with the following guarantee.
\begin{proposition}
For $\alpha_1, \alpha_2, \cdots, \alpha_n$ where $\alpha_i \in \C$ and any constant $\eta>0$, we can partition these $n$ points into clusters such that:
\begin{enumerate}
\item
$|\alpha_i - \alpha_j| < \eta$ if $\alpha_i$ and $\alpha_j$ are in the same cluster.
\item
$|\alpha_i - \alpha_j| > \eta / n$ if $\alpha_i$ and $\alpha_j$ are in different clusters.
\end{enumerate}
\end{proposition}

Let $\supp = \supp(\mix) \cup \supp(\mix')$ and let $n = |\supp|$ $(n \leq 2k)$. Arrange the points in $\supp$ as $\veca = (\alpha_1, \cdots, \alpha_n)$ such that each cluster lies in a continuous segment. In other words, if $\alpha_i$ and $\alpha_j$ lie in the same cluster for indexes $i < j$, then $\alpha_{i'}$ and $\alpha_{j'}$ also lie in the same cluster for all $i \leq i' < j' \leq j$; if $\alpha_i$ and $\alpha_j$ lie in the different cluster for indexes $i < j$, then $\alpha_{i'}$ and $\alpha_{j'}$ also lie in the different cluster for all $i' \leq i < j \leq j'$.

\begin{definition}
Suppose $\veca = (\alpha_1, \alpha_2, \cdots, \alpha_n) \in \mathbb{B}_{\C}^n$.
Let $\calP_{\veca}$ be the set of polynomials of degree at most $n-1$ and Lipschitz over $\veca$, i.e.,
$$
\calP_{\veca}=\{P\mid \deg P\leq n-1,P(\alpha_1)=0; \, |P(\alpha_i)-P(\alpha_j)|\leq |\alpha_i-\alpha_j|\,\, \forall i,j\in [n]\}.
$$
\end{definition}

From the discussion in Appendix~\ref{app:transfornonprob}, we have the following proposition on the generalized transportation distance.

\begin{proposition}
\label{prop:ge-simple}
$
\tran(\mix,\mix')\leq 2\sup\left\{\int f \d(\mix-\mix'): f\in \lip_{\C}\right\}=2\sup\left\{\int p \d(\mix-\mix'): p\in \calP_{\veca}\right\}
$.
\end{proposition}

Similar to the real case, we only need to focus on the extreme points of $\calP_{\veca}$.

\begin{definition}
Let $\veca = (\alpha_1, \alpha_2, \cdots, \alpha_n) \in \mathbb{B}_{\C}^n$.
Let
$$
\calF_{\veca}=\left\{p\large\mid \deg p\leq n-1, p(\alpha_1)=0; \, \frac{p(\alpha_{i+1})-p(\alpha_i)}{\alpha_{i+1}-\alpha_i}\in\{\pm 1, \pm \imag\}\,\, \forall i\in [n-1]\right\}.
$$
It is easy to see that $|\calF_{\veca}|=4^{n-1}$.
\end{definition}

%We note that a polynomial in $\calF_{\veca}$ may have very large height. 
Again, we modify each polynomial $p \in \calF_{\veca}$ slightly. 
We define a degree-$(n-1)$ polynomial $T_{\eta}(p) : \C \rightarrow \C$ by assigning values to all $\alpha_i$ for $i \in[n]$ as follows:
\begin{enumerate}
\item If $\alpha_i$ is one of the first two points in the cluster corresponding to $\alpha_i$, assign $T_{\eta}(p)(\alpha_i) = p(\alpha_i)$.
\item 
Otherwise, denote $\alpha_j, \alpha_{j+1}$ as the first two points of the cluster, we would assign $T_{\eta}(p)(\alpha_i)$ to be the linear interpolation of $T_{\eta}(p)(\alpha_j)$  and $T_{\eta}(p)(\alpha_j+1)$. In concrete, $$T_{\eta}(p)(\alpha_i) = \frac{p(\alpha_{j+1}) - p(\alpha_j)}{\alpha_{j+1} - \alpha_j}(\alpha_i - \alpha_j) + p(\alpha_j).$$
\end{enumerate}

Since we have fixed $n$ points in $n-1$-degree polynomial $T_{\eta}(p)$, we can see that $T_{\eta}(p)$ is uniquely determined. 
The following lemma shows $T_{\eta}(p)$ is close to $p$ over the points in $\veca$.

\begin{lemma}
\label{lm:complex-eta-close} $|T_{\eta}(p)(\alpha_i) - p(\alpha_i)| \leq 2\eta$
for each $\alpha_i\in \veca$.
\end{lemma}

\begin{proof}
We only need to prove the case that $\alpha_i$ is not the first two points in the cluster.
\begin{align*}
|T_{\eta}(p)(\alpha_i) - p(\alpha_i)| &\leq \left|\frac{p(\alpha_{j+1}) - p(\alpha_j)}{\alpha_{j+1} - \alpha_j}(\alpha_i - \alpha_j) + p(\alpha_j) - p(\alpha_i) \right| \\
&\leq \left|\frac{p(\alpha_{j+1}) - p(\alpha_j)}{\alpha_{j+1} - \alpha_j}(\alpha_i - \alpha_j)\right| + \left|p(\alpha_j) - p(\alpha_i)\right| \\
&\leq 2|\alpha_i - \alpha_j| \leq 2\eta,
\end{align*}
where the third inequality holds since  $|p(\alpha_{j+1}) - p(\alpha_j)| = |\alpha_{j+1} - \alpha_j|$ and the last inequality holds because $\alpha_i$ and $\alpha_j$ are in the same cluster.
\end{proof}

\begin{lemma}
\label{lm:add-term}
Let $\veca = (\alpha_1, \alpha_2, \cdots, \alpha_n) \in \mathbb{B}_{\C}$. Suppose $T_\eta(p)=\sum_{i=0}^{n-1} c_i x^i$. We have $|c_t| \leq n2^{2n-2}(\frac{n}{\eta})^{n-2}$.
\end{lemma}

\begin{proof}
Again, we would use Newton's interpolation polynomials.
Let $F[x_1, \cdots, x_i]$ be the $i$th order difference of $T_{\eta}(p)$.
%, which can be defined recursively, 
%$$
%F[\alpha_t,\alpha_{t+1}] = \frac{T_{\eta}(p)(\alpha_{t+1}) - %T_{\eta}(p)(\alpha_{t})}{\alpha_{t+1}-\alpha_t}, \cdots, F[\alpha_t,\cdots,\alpha_{t+i}] = %\frac{F[\alpha_{t+1},\cdots,\alpha_{t+i}] - %F[\alpha_{t},\cdots,\alpha_{t+i-1}]}{\alpha_{t+i}-\alpha_{i}}.
%$$
%Then, by Newton's interpolation formula, we can write that
%$$
%T_\eta(p)(x)=F[\alpha_1,\alpha_2](x-\alpha_1)+F[\alpha_1,\alpha_2,\alpha_3](x-\alpha_1)(x-\alpha_2)+\c%dots+F[\alpha_1,\cdots,\alpha_n]\prod_{i=1}^{n-1}(x-\alpha_i).
%$$
By definition of $T_{\eta}(p)$, we know that every 2nd order difference of $T_{\eta}(p)$ is in $\{\pm 1, \pm \imag\}$. Now, we show inductively that the absolute value of $i$th order difference absolute value is at most $(2n/\eta)^{i-2}$ for any $i=3, \cdots, n$. We distinguish two cases:
\begin{enumerate}
\item
If $|\alpha_{t+i}-\alpha_t| < \eta/n$, then all $\alpha_t, \alpha_{t+1}, \cdots, \alpha_{t+i}$ lie in the same cluster (according to the assigned order). Therefore, we can see that $F[\alpha_t,\alpha_{t+1}] = F[\alpha_{t+1},\alpha_{t+2}] = \cdots = F[\alpha_{t+i-1}, \alpha_{t+i}]$, from which any 3rd order difference (hence the 4th, 5th, up to the $i$th) in this interval is zero.
\item Otherwise, $|\alpha_{t+i}-\alpha_t|>\eta/n$.
By the induction hypothesis, we have that
$$
|F[\alpha_t,\cdots,\alpha_{t+i}]|=\left|\frac{F[\alpha_{t+1},\cdots,\alpha_{t+i}] - F[\alpha_{t},\cdots,\alpha_{t+i-1}]}{\alpha_{t+i}-\alpha_{t}}\right|\leq 2 \cdot (2n/\eta)^{i-3} \cdot (n/\eta) \leq (2n/\eta)^{i-2}.
$$
\end{enumerate}
Also since $|\alpha_j|\leq 1$, the absolute value of the coefficient of $x^t$ in
$\prod_{j=1}^i(x-\alpha_j)$ is less than $2^n$.
So, finally we have $c_t\leq n2^{2n-2}(\frac{n}{\eta})^{n-2}$.
\end{proof}

\begin{lemma}
\label{lm:ge-decomposition}
For any polynomial $p\in \calP_{\veca}$, there is a polynomial $p_\eta(x)$ such that the following properties hold:
\begin{enumerate}
\item
$|p_{\eta}(\alpha_i)-p(\alpha_i)|\leq 2\eta$, for any $i\in [n]$.
\item
The height of $g_\eta$ is at most $n2^{2n-2}(\frac{n}{\eta})^{n-2}$.
\end{enumerate}
\end{lemma}

\begin{proof}
For every polynomial $p(x)\in \calP_{\veca}$,
let $\vecg=\frac{p(\alpha_{i+1})-p(\alpha_i)}{\alpha_{i+1}-\alpha_i}$ for $i=[n-1]$.
So $|g_i|\leq 1$, that means the point $\vecp=(g_1,\cdots,g_{n-1})$ is in the $n-1$ dimensional
complex hypercube (which has $4^{n-1}$ vertices).
Then, we apply exactly the same argument as in the proof of Lemma~\ref{lm:decomposition}.
\eat{
Hence, there are $\lambda_1,\cdots,\lambda_{4^{n-1}}\geq 0, \lambda_1+\cdots+\lambda_{4^{n-1}}=1$,
such that $\vecg=\sum_{i=1}^{4^{n-1}}\lambda_i \vecq_i$ where $\vecq_1,\cdots,\vecq_{4^{n-1}}$ are the vertices of hypercube,
i.e., $\vecq_i\in \{\pm1,\pm\imag\}^{n-1}$.
For each $\vecq_i$, we define $p_i(x)$ be the polynomial with degree at most $n-1$
and
$$
p_i(\alpha_1)=0,\;
p_i(\alpha_m)=\sum_{j=1}^{m-1} q_{i,j}(\alpha_{j+1}-\alpha_j).
$$
where $q_{i,j}$ is the $j$th coordinate of $\vecq_i$.
We can see $p_i\in \calF_{\veca}$.
It is easy to verify that
$$
p(x)=\sum_{i=1}^{4^{n-1}} \lambda_i p_i(x),
$$
since both the LHS and RHS take the same values at $\{\alpha_1,\cdots, \alpha_n\}$.
For $\eta>0$, we define $p_{\eta}(x)$ as $p_{\eta}=\sum_{i=1}^{4^{n-1}} \lambda_i T_\eta(p_i)$.
We know $|T_\eta(p_i)(\alpha_j)-p_i(\alpha_j)|\leq 2\eta$ for all $i\in [2^{n-1}]$ and $j\in [n]$
and the height of each $T_\eta(p_i)$ is at most $n2^{2n-2}(\frac{n}{\eta})^{n-2}$.
So, $p_{\eta}(x)$ satisfies the properties of the lemma.
}
\end{proof}

\begin{theorem}
\label{thm:complex-moment-inequality}
For any two mixtures with $k$ components $\mix,\mix' \in \Spike(\mathbb{B}_{\C},\Sigma^{\C}_{k-1})$, it holds that
$$
\tran(\mix,\mix')\leq O(k\Mdis_K(\mix,\mix')^{\frac{1}{2k-1}}).
$$
\end{theorem}

\begin{proof}
Fix any constant $\eta > 0$. Let $\noise = \Mdis_K(\mix, \mix')$

\begin{align*}
\tran(\mix, \mix') &\leq 2\sup_{f\in \lip_{\C}} \int f \d (\mix - \mix') = 2\sup_{p\in \calP_{\veca}} \int p \d(\mix - \mix') \\
&\leq 2\left(4\eta + \sup_{p\in \calP_{\veca}} \int p_{\eta} \d(\mix - \mix')\right)  & (\textrm{Lemma}~\ref{lm:ge-decomposition}) \\
&\leq 2\left(4\eta + \sup_{p\in \calP_{\veca}} \sum_{p=0}^{2k-1} n2^{2n-2}\left(\frac{n}{\eta}\right)^{n-2} \Mdis_K(\mix, \mix')\right)  & (\textrm{Lemma}~\ref{lm:ge-decomposition}) \\
&\leq 2\left(4\eta+2^{4k}(2k)^{2k}\frac{\noise}{\eta^{2k-2}}\right). & (n\leq 2k)
\end{align*}
By choosing $\eta=16k\noise^{\frac{1}{2k-1}}$,
we conclude that $\tran(\mix,\mix')\leq O(k\noise^{\frac{1}{2k-1}})$.
\end{proof}

%\jiannote{Check the relation between this theorem and Theorem~\ref{thm:maininequality}}

\subsection{Moment-Transportation Inequality in Higher Dimensions}

We generalize our proof for one dimension to higher dimensions.
Lemma \ref{lm:prob-proj} shows there exists a vector $\vecr$ such that 
the distance between spikes are still lower bounded (up to factors depending on $k$ and $d$) after projection. This enables us to extend Lemma~\ref{lm:inequality1d} to high dimension.

\begin{theorem}
Let $\mix,\mix'$ be two $k$-spike mixtures in $\Spike(\simplex_{d-1}, \simplex_{k-1})$, and  
$K=2k-1$. Then, the following inequality holds
$$\tran(\mix,\mix')\leq O(k^3d\Mdis_K(\mix,\mix')^{\frac{1}{2k-1}}).$$
\end{theorem}

\begin{proof}
Apply the argument in Lemma \ref{lm:prob-proj} to $\supp = \supp(\mix) \cup \supp(\mix')$. There always exists vector $\vecr \in \mathbb{S}_{d-1}$ and constant $c$ such that $|r^{\top}(\veca_i - \veca_j)| \geq \frac{c}{k^2d} \|\veca_i - \veca_j\|_1$ for all $\veca_i, \veca_j \in \supp$. Note that $|r^{\top}\veca_i| \leq \|r\|_{\infty} \|\veca_i\|_1 \leq 1$, both $\proj{\vecr}(\mix)$ and $\proj{\vecr}(\mix')$ are in $\Spike(\Delta_1, \Delta_{k-1})$. In this case,
\begin{align*}
    \tran(\mix, \mix') &= \inf \left\{\int \|\vecx - \vecy\|_1  \d \mu(\vecx,\vecy) : \mu\in M(\mix,\mix') \right\} & (\textrm{Definition}) \\
    &\leq \inf \left\{ \int \frac{k^2d}{c}|\vecr^{\top}(\vecx - \vecy)| \d \mu(\vecx,\vecy)  : \mu\in M(\mix,\mix') \right\} & (\textrm{Lemma }\ref{lm:prob-proj}) \\
    &= \frac{k^2d}{c} \tran(\proj{\vecr}(\mix), \proj{\vecr}(\mix')) \\
    &\leq O(k^3d \Mdis_K(\mix, \mix')^{\frac{1}{2k-1}}). & (\textrm{Lemma }\ref{lm:inequality1d})
\end{align*}
\end{proof}

\section{Missing Details from Section~\ref{sec:1d-recover}}

\subsection{Implementation Details of Algorithm~\ref{alg:dim1}}
\label{subsec:1d-alg-detail}

In this subsection, we show how to implement Algorithm~\ref{alg:dim1} in $O(k^2)$ time.
We first perform a ridge regression to obtain $\hvecc$ in Line~\ref{line:dim1-hc}.
The explicit solution of this ridge regression is
$\hvecc = (A_{M'}^{\top} A_{M'} + \noise^{2} I)^{-1} A_{M'}^{\top} b_{M'}$.
Since $A_{M'}$ is a Hankel matrix,  $\hvecc = (A_{M'} - \imag \noise I)^{-1}(A_{M'} + \imag \noise I)^{-1}A_{M'} b_{M'}$ holds. 
Note that $\vecx = (A + \lambda I)^{-1}b$ is a single step of the inverse power method, which can be computed in $O(k^2)$ time when $A$ is a Hankel matrix(see \citet{Xu2008AFS}). Hence, Line~\ref{line:dim1-hc} can be implemented in $O(k^2)$ arithmetic operations.

In Line~\ref{line:dim1-ha}, we solve the roots of polynomial $\sum_{i=0}^{k-1} \hc_i x^i + x^k$. 
We can use the algorithm in \citet{Neff1996AnEA}
to find the solution with $\noise$-additive noise in $O(k^{1+o(1)} \cdot \log \log (1/\noise))$ arithmetic operations.

Line~\ref{line:dim1-hw} is a linear regression defined by Vandemonde matrix $V_{\veca}$. We can use the recent algorithm developed in \citet{Meyer2022FastRF} to find a constant factor multiplicative approximation, which uses $O(k^{1+o(1)})$ arithmetic operations. Note that a constant factor approximation suffices for our purpose.  %(SOLVE LINEAR SYSTEM??)

In Line~\ref{line:dim1-cw}, we find the $k$-spike distribution in $\Spike(\simplex_1, \simplex_{k-1})$ closest to $\tilde{\mix}$. 
To achieve $O(k^2)$ running time, we limit the spike distribution to have support $\tveca$. In this case, the optimization problem is equivalent to finding a transportation plan from $(\tveca, \tvecw^{-})$ to $(\tveca, \tvecw^{+})$ where $\tvecw^{-}$ and $\tvecw^{+}$ are the negative components and positive components of $\tvecw$ respectively. In concrete, $\tw^{-}_i = \max\{0, -\tw_i\}$ and $\tw^{+}_i = \max\{0, \tw_i\}$.
Note that the points are in one dimension. 
Using the classical algorithm in \citet{Aggarwal1992EfficientMC}, we can solve this transportation problem in $O(k^{1+o(1)})$ arithmetic operations.

\subsection{Proof of Theorem~\ref{thm:1dim-kspikecoin}}
\label{subsec:1d-alg-analysis}

Now, we start to bound the reconstruction error of the algorithm. 
The following lemma is well known (see e.g., \citet{chihara2011introduction,gordon2020sparse}). We provide a proof for completeness.
\begin{lemma}\label{lm:dim1-char-poly}
    Let $\mix = (\veca, \vecw) \in \Spike(\simplex_1, \simplex_{k-1})$ where $\veca = [\alpha_1, \cdots, \alpha_k]^{\top}$ and $\vecw = [w_1, \cdots, w_k]^{\top}$. 
    Let $\vecc = [c_0, \cdots, c_{k-1}]^{\top} \in \R^k$ such that $\prod_{i=1}^{k} (x - \alpha_i) = \sum_{i=0}^{k-1} c_i x^i + x^k$. 
    Then,
    $\sum_{j=0}^{k-1} M_{i+j}(\mix) c_j + M_{i+k}(\mix) = 0$
    for any $i \geq 0$.
    The equation can also be written as
    $A_{M(\mix)} \vecc + b_{M(\mix)} = \mathbf{0}$
    in the matrix form.
\end{lemma}

\begin{proof}
    By direct calculation,
    \begin{align*}
        \sum_{j=0}^{k-1} M_{i+j}(\mix) c_j + M_{i+k}(\mix)
	 	&= \sum_{j=0}^{k-1} \sum_{t=1}^k w_t \alpha_t^{i+j} c_j + \sum_{t=1}^k w_t \alpha_t^{i+k}  
		= \sum_{t=1}^k w_t \alpha_t^{i} \left(\sum_{j=0}^{k-1} c_j \alpha_t^j + \alpha_t^k\right) \\
		&= \sum_{t=1}^k w_t \alpha_t^{i} \prod_{j=1}^{k} (\alpha_t - \alpha_j) = 0.
    \end{align*}
    The $i$th row of $A_{M(\mix)} \vecc + b_{M(\mix)}$ is  $\sum_{j=0}^{k-1} M_{i+j}(\mix) c_j + M_{i+k}(\mix)$, hence 
    the matrix form.
\end{proof}

%Hence,  Lemma~\ref{lm:dim1-char-poly} is equivalent to the matrix form 
%$$
%A_{M(\mix)} \vecc + b_{M(\mix)} = \mathbf{0}.
%$$
%This motivates Line~\ref{line:dim1-hc} of Algorithm~\ref{alg:dim1}. 

Next, the following lemma shows that the intermediate result $\hvecc$ a good estimation for the solution of $A_{M(\mix)} \vecx + b_{M(\mix)} = \boldsymbol{0}$ with a small norm:

%\jiannote{Zhiyuan write down the matrix form for $A_M$ and $b_M$. implied from the above lemma.}

\begin{lemma} \label{lm:dim1-step1}
    Let $\vecc = [c_0, \cdots, c_{k-1}]^{\top} \in \R^k$ be the sequence of number in which $\prod_{i=1}^{k} (x - \alpha_i) = \sum_{i=0}^{k-1} c_i x^i + x^k$. 
    Suppose $\|M' - M(\mix)\|_{\infty}\leq \noise$.
    Let $\hvecc = [\hc_0, \cdots, \hc_{k-1}]^{\top} \in \R^k$ be the intermediate result (Line~\ref{line:dim1-hc}) in Algorithm~\ref{alg:dim1}.  Then, $\|\vecc\|_1 \leq 2^k$, $\|\hvecc\|_1 \leq 2^{O(k)}$ and $\|A_{M(\mix)} \hvecc + b_{M(\mix)}\|_\infty \leq 2^{O(k)} \cdot \noise$.
\end{lemma}

\begin{proof}
    From Vieta's formulas (see e.g., \citet{barbeau2003polynomials}), we have 
    $c_i = \sum_{S \in \binom{[k]}{k-i}} \prod_{j\in S} (-\alpha_j).$ Thus,
    $$\|\vecc\|_1 = \sum_{i=0}^{k-1} |c_i| \leq \sum_{S \in 2^{[k]}} \prod_{j\in S} |\alpha_j| = \prod_{i=1}^{k} (1+|\alpha_i|) \leq 2^k.$$
    where the last inequality holds since $|\alpha_i| \leq 1$ for all $i$ according to the definition of $\mix$.
    %\jiannote{should $\leq$ be $=$??}
    
    From Lemma~\ref{lm:dim1-char-poly}, we can see that $\|A_{M(\mix)} \vecc + b_{M(\mix)}\|_{\infty} = 0$.
    Therefore, we can further get
    \begin{align*}
		\|A_{M'} \vecc + b_{M'}\|_\infty &\leq \|A_{M(\mix)} \vecc + b_{M(\mix)}\|_{\infty} + \|A_{M'} \vecc - A_{M(\mix)} \vecc\|_{\infty} + \|b_{M'} - b_{M(\mix)}\|_{\infty} \\
		&\leq \|A_{M(\mix)} \vecc + b_{M(\mix)}\|_{\infty} + \|(A_{M'} - A_{M(\mix)}) \vecc\|_{\infty} + \|b_{M'} - b_{M(\mix)}\|_{\infty} \\
		&\leq \|A_{M(\mix)} \vecc + b_{M(\mix)}\|_{\infty} + \|A_{M'} - A_{M(\mix)}\|_{\infty} \|\vecc\|_1 + \|b_{M'} - b_{M(\mix)}\|_{\infty} \\
		&\leq \|A_{M(\mix)} \vecc + b_{M(\mix)}\|_{\infty} +  \|M' - M(\mix)\|_{\infty} (\|\vecc\|_1 + 1) \\
		&\leq 2^{O(k)} \cdot \noise,
	\end{align*}
    where the fourth inequality holds since $(A_{M'} - A_{M(\mix)})_{i,j}=(M'-M(\mix))_{i+j}$ and $(b_{M'}-b_{M(\mix)})_i = (M'-M(\mix))_{i+k}$.
    From the definition of $\hvecc$, we can see that
	\begin{align*}
		\|A_{M'} \hvecc + b_{M'}\|_2^2 + \noise^2 \|\hvecc\|_2^2 
		&\leq \|A_{M'} \vecc + b_{M'}\|_2^2 + \noise^2 \|\vecc\|_2^2 \\
		&\leq k \|A_{M'} \vecc + b_{M'}\|_\infty^2 + \noise^2 \|\vecc\|_1^2 \\
		&\leq 2^{O(k)} \cdot \noise^2.
	\end{align*}
	The second inequality holds since $\|\vecx\|_2 \leq \|\vecx\|_1$ and $\|\vecx\|_2 \leq \sqrt{k} \|\vecx\|_{\infty}$ holds for any vector $\vecx \in \R^k$.
	
	Now, we can directly see that 
	\begin{align*}
		\|A_{M'} \hvecc + b_{M'}\|_\infty &\leq \|A_{M'} \hvecc + b_{M'}\|_2 \leq 2^{O(k)} \cdot \noise, \\
		\|\hvecc\|_1 &\leq \sqrt{k} \|\hvecc\|_2 \leq 2^{O(k)}.
	\end{align*}
    Finally, we  can bound $\|A_{M(\mix)} \hvecc + b_{M(\mix)}\|_\infty$
	as follows:
	\begin{align*}
	    \|A_{M(\mix)} \hvecc + b_{M(\mix)}\|_\infty &\leq \|A_{M'} \hvecc + b_{M'}\|_\infty + \|A_{M'} \hvecc - A_{M(\mix)} \hvecc\|_{\infty} + \|b_{M'} - b_{M(\mix)}\|_{\infty} \\
	    %&\leq \|A_{M'} \hvecc + b_{M'}\|_\infty + \|(A_{M'} - A_{M(\mix)}) %\hvecc\|_{\infty} + \|b_{M'} - b_{M(\mix)}\|_{\infty} \\
	    &\leq \|A_{M'} \hvecc + b_{M'}\|_\infty + \|A_{M'} - A_{M(\mix)}\|_{\infty} \| \hvecc\|_1 + \|b_{M'} - b_{M(\mix)}\|_{\infty} \\
        &\leq \|A_{M'} \hvecc + b_{M'}\|_\infty + \|M' - M(\mix)\|_{\infty} (\|\hvecc\|_1 + 1) \\
        &\leq 2^{O(k)} \cdot \noise.
	\end{align*}
\end{proof}

Using this result, we are able to show that $\tveca$ in Line~\ref{line:dim1-ta} is a good estimation for the ground truth spikes, $\veca$.
In particular, the following lemma shows that every ground truth spike $a_i \in \veca$ has a close spike $\ta_j \in \tveca$ where the closeness depends on the weight of the spike $w_i$.
We note that this property tolerates permutation and small spikes weights, enabling us to analyze mixtures without separations.

\begin{lemma} \label{lm:dim1-step2-1}
   Let $\hveca = [\ha_1, \cdots, \ha_{k}]^{\top}\in \C^k$ be the intermediate result (Line~\ref{line:dim1-ha}) in Algorithm~\ref{alg:dim1}.
   Then, the following inequality holds:
    \begin{align*}
        \sum_{i=1}^{k} w_i \prod_{j=1}^{k} |\alpha_i - \ha_j|^2 \leq 2^{O(k)} \cdot \noise.
    \end{align*}
\end{lemma}

\begin{proof}
    Since $\hvecc$ and $\veca$ are real, from the definition of $\hvecc$, 
    \begin{align*}
        \sum_{i=1}^{k} w_i \prod_{j=1}^{k} |\alpha_i - \ha_j|^2
        = \sum_{i=1}^{k} w_i \left|\prod_{j=1}^{k} (\alpha_i - \ha_j)\right|^2
        = \sum_{i=1}^{k} w_i \left|\sum_{j=0}^{k-1} \hc_j \alpha_i^j + \alpha_i^k\right|^2  
        = \sum_{i=1}^{k} w_i \left(\sum_{j=0}^{k-1} \hc_j \alpha_i^j + \alpha_i^k\right)^2.
    \end{align*}
    By expanding the RHS, we can reach that 
    \begin{align*}
        \sum_{i=1}^{k} w_i \prod_{j=1}^{k} |\alpha_i - \ha_j|^2 
        &= \sum_{i=1}^{k} w_i \left(\sum_{p=0}^{k-1}\sum_{q=0}^{k-1}\hc_p \hc_q \alpha_i^{p+q} + 2\sum_{p=0}^{k-1} \hc_p \alpha_i^{p+k} + \alpha_i^{2k} \right) \\
        &= \sum_{p=0}^{k-1}\sum_{q=0}^{k-1}\hc_p \hc_q \sum_{i=1}^{k}w_i \alpha_i^{p+q} + 2\sum_{p=0}^{k-1}\hc_p \sum_{i=1}^{k} w_i \alpha_i^{p+k} + \sum_{i=1}^k w_i \alpha_i^{2k} \\
        &= \sum_{p=0}^{k-1} \hc_p \left(\sum_{q=0}^{k-1} M_{p+q}(\mix) \hc_q + M_{p+k}(\mix)\right) + \sum_{p=0}^{k-1} \hc_p M_{p+k}(\mix) + M_{2k}(\mix) \\
        &= \sum_{p=0}^{k-1} \hc_p \Bigl(A_{M(\mix)} \hc + b_{M(\mix)}\Bigr)_p + \sum_{p=0}^{k-1} \hc_p M_{p+k}(\mix) + M_{2k}(\mix),
    \end{align*}
    where the last equality holds from the definition of matrix $A_{M(\mix)}$ and vector $b_{M(\mix)}$. Moreover, 
    \begin{align*}
        \sum_{p=0}^{k-1} \hc_p M_{p+k}(\mix) + M_{2k}(\mix) &= \sum_{p=0}^{k-1} \hc_p \left(- \sum_{q=0}^{k-1} M_{p+q}(\mix) c_q\right) + \left(- \sum_{q=0}^{k-1} M_{k+q}(\mix) c_q \right) \\
        &= -\sum_{q=0}^{k-1} c_q \left(\sum_{p=0}^{k-1} M_{p+q}(\mix) \hc_p + M_{k+q}(\mix)\right) \\
        &= -\sum_{q=0}^{k-1} c_q \Bigl(A_{M(\mix)} \hc + b_{M(\mix)}\Bigr)_q,
    \end{align*}
    where the first equality dues to Lemma~\ref{lm:dim1-char-poly} in which $M_{p+k}(\mix) = - \sum_{q=0}^{k-1} M_{p+q}(\mix) c_q$.
    Combining the above results, we finally get
    \begin{align*}
         \sum_{i=1}^{k} w_i \prod_{j=1}^{k} |\alpha_i - \ha_j|^2 &= \sum_{i=0}^{k-1} (\hc_i - c_i) \Bigl(A_{M(\mix)} \hc + b_{M(\mix)}\Bigr)_i \\
         &\leq (\|\hc\|_1 + \|c\|_1) \|A_{M(\mix)} \hc + b_{M(\mix)}\|_{\infty} \\
         &\leq (2^k + 2^{O(k)}) \cdot 2^{O(k)}\cdot \noise & (\textrm{Lemma}~\ref{lm:dim1-step1})\\
         &\leq 2^{O(k)} \cdot \noise.
    \end{align*}
\end{proof}

We provide the following simple inequality that allows us to upper bound the impact of the injected noise. 
%We defer the proof to the appendix.

\begin{lemma} \label{lm:injected-noise}
    Let $b>0$ be some constant.
    For $a_1, \cdots, a_k \in \R$ and $a_1', \cdots, a_k' \in \R$ such that $a_i \in [0, b]$ and $|a_i' - a_i|\leq \noise$ holds for all $i$. In case that $\noise \leq n^{-1}$, then,
    \begin{align*}
        \prod_{i=1}^k a_i' \leq \prod_{i=1}^{k} a_i + O(b^k \cdot k\noise).
    \end{align*}
\end{lemma}

\begin{proof}
    The proof is similar to the proof of Lemma 1 in \citep{li2010ranking}.
    \begin{align*}
        \prod_{i=1}^k a_i' - \prod_{i=1}^{k} a_i &= \sum_{S\subseteq [k], S\neq \emptyset} \prod_{i \in S} a_i \prod_{i\not\in S} (a_i' - a_i)\ \\
        &\leq \sum_{t=1}^{k} \sum_{|S|=t} \prod_{i \in S} a_i \prod_{i\not\in S} (a_i' - a_i) \\
        &\leq b^k \sum_{t=1}^{k} \binom{k}{t} \noise^{k} \leq b^k \sum_{t=1}^{k} \frac{(k\noise)^{k}}{t!} \\
        &\leq b^k (e^{k\noise} - 1) = O(b^k \cdot k\noise).
    \end{align*}
    The third inequality holds because $\binom{k}{t} \leq \frac{k^t}{t!}$. The last inequality holds since $e^{x} = \sum_{i>0} \frac{x^i}{i!}$ and  the fact that $e^{f(x)} = 1 + O(f(x))$ if $f(x) = O(1)$.
\end{proof}

With this lemma, we are able to show that the projection (Line~\ref{line:dim1-ba}) and the injected noise (Line~\ref{line:dim1-ta}) do not introduce much extra error for our estimation of the positions of ground truth spikes.

\begin{lemma} \label{lm:dim1-step2}
    Let $\tveca = [\ta_1, \cdots, \ta_{k}]^{\top} \in \R^k$ be the intermediate result (Line~\ref{line:dim1-ta}) in Algorithm~\ref{alg:dim1}. Then,
    \begin{align*}
        \sum_{i=1}^{k} w_i \prod_{j=1}^{k} |\alpha_i - \ta_j|^2 \leq 2^{O(k)} \cdot \noise.
    \end{align*}
\end{lemma}

\begin{proof}
    Let $\bveca = [\ba_1, \cdots, \ba_{k}]^{\top} \in \R^k$ be the set of projections (Line~\ref{line:dim1-ba}). Since the projected domain $\simplex_1$ is convex, 
    $|\alpha_i - \ba_j| \leq |\alpha_i - \ha_j|$ always holds. This gives 
    \begin{align*}
        \prod_{j=1}^{k} |\alpha_i - \ba_j|^2 \leq \prod_{j=1}^{k} |\alpha_i - \ha_j|^2.
    \end{align*}
    Since $\ta_j$ is noisy $\ba_j$ with additive noise of size less than $\noise$, we have $|\alpha_i - \ta_j| \leq |\alpha_i - \ba_j| + \noise$.
    Applying Lemma~\ref{lm:injected-noise} by regarding $|\alpha_i - \ta_j|$ as $a_j$ and $|\alpha_i - \ha_j|$ as $a_j'$, with the constrain guaranteed according to $|\alpha_i - \ta_j|\leq |\alpha_i| + |\ta_j| \leq 2$, we can conclude that
    \begin{align*}
        \prod_{j=1}^{k} |\alpha_i - \ta_j|^2 \leq \prod_{j=1}^{k} |\alpha_i - \ba_j|^2 + O(2^k \cdot k\noise).
    \end{align*}
    Combining two inequalities, we get the desired inequality: 
    $$
       \sum_{i=1}^{k} w_i \prod_{j=1}^{k} |\alpha_i - \ta_j|^2 \leq \sum_{i=1}^{k} w_i \prod_{j=1}^{k} |\alpha_i - \ha_j| + O(2^k \cdot k\noise) \leq 2^{O(k)} \cdot \noise,
    $$
    where the last inequality is given by Lemma~\ref{lm:dim1-step2-1}.
\end{proof}

We now start to bound the error in Line~\ref{line:dim1-hw} which is a linear regression
defined by the Vandermonde matrix.
We first introduce the Schur polynomial which has a strong connection to the Vandermonde matrix.

\begin{definition}[Schur polynomial]
Given a partition $\lambda = (\lambda_1, \cdots, \lambda_k)$ where $\lambda_1 \geq \cdots \geq \lambda_k \geq 0$, the Schur polynomial $s_{\lambda}(x_1, x_2, \cdots, x_k) $ is defined  by the ratio 
$$
    s_{\lambda}(x_1, x_2, \cdots, x_k) = 
    \left|\begin{matrix}
        x_1^{\lambda_1+k-1} & x_2^{\lambda_1+k-1} & \cdots & x_k^{\lambda_1+k-1} \\
        x_1^{\lambda_2+k-2} & x_2^{\lambda_2+k-2} & \cdots & x_k^{\lambda_2+k-2} \\
        \vdots & \vdots & \ddots & \vdots \\
        x_1^{\lambda_k} & x_2^{\lambda_k} & \cdots & x_k^{\lambda_k}
    \end{matrix}\right| \cdot \left|\begin{matrix}
        1 & 1 & \cdots & 1 \\
        x_1^1 & x_2^1 & \cdots & x_k^1 \\
        \vdots & \vdots & \ddots & \vdots \\
        x_1^{k-1} & x_2^{k-1} & \cdots & x_k^{k-1}
    \end{matrix}\right|^{-1}.
$$
\end{definition}

A Schur polynomial is a symmetric function because the numerator and denominator are 
both {\em alternating} (i.e., it changes sign if we swap two variables), and a polynomial since all alternating polynomials are divisible by the Vandermonde determinant
$\prod_{i<j}(x_i-x_j)$.
\footnote{It is easy to see that if a polynomial $p(x_1,\ldots, x_n)$ is alternating,
it must contain $(x_i-x_j)$ as a factor .}
The following classical result (see, e.g., \citet{10.5555/2124415}) shows an alternative way to calculate Schur polynomial. 
It plays a crucial role in bounding the error of a linear system
defined by the Vandermonde matrix (Lemma~\ref{lm:step3}).

\begin{theorem} (e.g. \citet{10.5555/2124415})
Let $\textrm{SSYT}(\lambda)$ be the set consists all semistandard Young tableau of shape $\lambda = (\lambda_1, \cdots, \lambda_k)$. Then, Schur polynomial $s_{\lambda}(x_1, x_2, \cdots, x_k)$ can also be computed as follows: 
\begin{align*}
    s_{\lambda}(x_1, x_2, \cdots, x_k) = \sum_{T\in \textrm{SSYT}(\lambda)} \prod_{i=1}^{k} x_i^{T_i}.
\end{align*}
where $T_i$ counts the occurrences of the number $i$ in $T$.
\end{theorem}

In a semistandard Young tableau, we
allow the same number to appear more than once (or not at all), and we require the entries weakly increase along each row and strictly increase down each column
As an example, for $\lambda = (2, 1, 0)$, the list of semistandard Young tableau are:

\begin{figure}[H]
    \centering
    \begin{ytableau}
       1 & 1 \\
       2 & \none
    \end{ytableau}
    \hspace{.1cm}
    \begin{ytableau}
       1 & 1 \\
       3 & \none
    \end{ytableau}
    \hspace{.1cm}
    \begin{ytableau}
       1 & 2 \\
       2 & \none
    \end{ytableau}
    \hspace{.1cm}
    \begin{ytableau}
       1 & 2 \\
       3 & \none
    \end{ytableau}
    \hspace{.1cm}
    \begin{ytableau}
       1 & 3 \\
       2 & \none
    \end{ytableau}
    \hspace{.1cm}
    \begin{ytableau}
       1 & 3 \\
       3 & \none
    \end{ytableau}
    \hspace{.1cm}
    \begin{ytableau}
       2 & 2 \\
       3 & \none
    \end{ytableau}
    \hspace{.1cm}
    \begin{ytableau}
       2 & 3 \\
       3 & \none
    \end{ytableau}
\end{figure}
The corresponding Schur polynomial is 
\begin{align*}
    s_{(2, 1, 0)}(x_1, x_2, x_3) = x_1^2x_2 + x_1^2x_3 + x_1x_2^2 + x_1x_2x_3 + x_1x_3x_2 + x_1x_3^2 + x_2^2x_3 + x_2x_3^2.
\end{align*}

Considering the special case where $\lambda = (j-k+1, 1, 1, \cdots, 1)$, this theorem implies the following result. We also present a self-contained proof of the lemma for completeness.

% In the next Lemma, 

% %To the best of our knowledge, we are not aware this equation in the literature.
% \jiannote{Schur polynomial}
% Let $(a)^{b} \subset \{0,\cdots,b\}^a$ contains all $s$ such that $\sum_{i=1}^{a} s_i = b$. 
% According to Schur polynomial, we directly have the following Lemma.
% \jiannote{this is a bit too abrupt. we need some back ground of Schur polynomial.
% for example, we need to define it.}

\begin{lemma} \label{lm:dim1-step3-matrix}  
For all $j \geq k$ and distinct numbers $a_1, \cdots, a_k \in \R$ (or $a_1, \cdots, a_k \in \C$),
    $$\left|\begin{matrix}
        a_1^j & a_2^j & \cdots & a_k^j \\
        a_1^1 & a_2^1 & \cdots & a_k^1 \\
        \vdots & \vdots & \ddots & \vdots \\
        a_1^{k-1} & a_2^{k-1} & \cdots & a_k^{k-1}
    \end{matrix}\right| \cdot \left|\begin{matrix}
        1 & 1 & \cdots & 1 \\
        a_1^1 & a_2^1 & \cdots & a_k^1 \\
        \vdots & \vdots & \ddots & \vdots \\
        a_1^{k-1} & a_2^{k-1} & \cdots & a_k^{k-1}
    \end{matrix}\right|^{-1} = \prod_{i=1}^{k} a_i \cdot \sum_{s \in (k)^{j-k}} \prod_{i=1}^{k} a_i^{s_i} $$
    where $(a)^{b} \subset \{0,\cdots,b\}^a$ denotes the set that contains all vector $s$ in which $\sum_{i=1}^{a} s_i = b$.
\end{lemma}

\begin{proof}
    We provide a self-contained proof here for completeness.
    According to the property of Vandermonde determinant, we have
    \begin{align*}
        \left|\begin{matrix}
        1 & 1 & \cdots & 1 \\
        a_1^1 & a_2^1 & \cdots & a_k^1 \\
        \vdots & \vdots & \ddots & \vdots \\
        a_1^{k-1} & a_2^{k-1} & \cdots & a_k^{k-1}
        \end{matrix}\right| &= (-1)^{k(k-1)/2}\prod_{p<q} (a_p - a_q).  
    \end{align*}
    By expansion along first row, we have 
    \begin{align*}
        \left|\begin{matrix}
            a_1^j & a_2^j & \cdots & a_k^j \\
            a_1^1 & a_2^1 & \cdots & a_k^1 \\
            \vdots & \vdots & \ddots & \vdots \\
            a_1^{k-1} & a_2^{k-1} & \cdots & a_k^{k-1}
        \end{matrix}\right| &= \sum_{i=1}^{k} (-1)^{i+1} a_i^j  
        \left|\begin{matrix}
            a_1^1 & \cdots & a_{i-1}^1 & a_{i+1}^1 & \cdots & a_k^1 \\
            \vdots & \ddots & \vdots & \vdots & \ddots & \vdots \\
            a_1^{k-1} & \cdots & a_{i-1}^{k-1} & a_{i+1}^{k-1} & \cdots & a_k^{k-1} \\
        \end{matrix}\right| \\
        &= \sum_{i=1}^{k} (-1)^{i+1} a_i^j (-1)^{(k-1)(k-2)/2} \prod_{t\neq i} a_t \prod_{p<q,p\neq i, q\neq i} (a_p - a_q) \\
        &= (-1)^{k(k-1)/2} \cdot (-1)^{k+1}
        \sum_{i=1}^{k}a_i^j (-1)^{i+1} \prod_{t\neq i} a_t \prod_{p<q,p\neq i, q\neq i} (a_p - a_q)
    \end{align*}
    Thus, 
    \begin{align*}
        \text{LHS}(j) = \sum_{i=1}^{k} a_i^{j} \cdot \prod_{t\neq i} \frac{a_t }{a_i - a_t}
    \end{align*}
    Let $L(z) = \sum_{j\geq k} \text{LHS}(j) z^j $ and $R(z) = \sum_{j \geq k} \text{RHS}(j) z^j$ be the generating functions of LHS and RHS corresponding to $j$ respectively. Then,
    \begin{align*}
        L(z) &= \sum_{i=1}^{k} \left(\sum_{j\geq k} a_i^j z^j \right) \prod_{t\neq i} \frac{a_t }{a_t - a_i} = z^k \sum_{i=1}^{k} \frac{a_i^k } {1 - a_i z} \prod_{t\neq i} \frac{a_t }{a_i - a_t} \\
        R(z) &= \prod_{i=1}^{k} \left( \sum_{j \geq 1}  a_i^j z^j \right) = z^k \prod_{t=1}^{k} \frac{a_t}{1 - a_t z} 
    \end{align*}
    Consider the quotient between two functions,
    \begin{align*}
        \frac{L(z)}{R(z)} = \sum_{i=1}^{k} a_i^{k-1}\prod_{t\neq i} \frac{1 - a_t z}{a_i - a_t}.
    \end{align*}
        For $z = a_p^{-1}$ where $p \in [1, k]$, the product in each  additive term with $i \neq p$ equals $0$. As a result, we can see that
    \begin{align*}
        \frac{L(a_p^{-1})}{R(a_p^{-1})} = a_p^{k-1} \prod_{t\neq p} \frac{1 - a_t a_p^{-1}}{a_i - a_t} = 1.
    \end{align*}
    Since $L(z)/R(z)$ is polynomial of $z$ of degree $k-1$, from the uniqueness of polynomial, we have
    $$\frac{L(z)}{R(z)} = 1 \Rightarrow \text{LHS} = \text{RHS}$$
    which directly proves the statement.
\end{proof}

%\jiannote{proof goes to appendix}

Using this result, we can show the reconstructed moment vector $V_{\tveca}\hvecw$ is close to the ground truth $M(\mix)$.
The next lemma shows the discrete mixture with support $\tveca$ can well approximate a single spike distribution at $\alpha$ in terms of the moments if any of the spikes in $\tveca$ is close to $\alpha$.

\begin{lemma} \label{lm:step3}
    Let $M(\alpha)$ be the vector
    $[\alpha^0, \alpha^1, \cdots, \alpha^{2k-1}]^{\top}$.
    Suppose $F$ is either $\{\R$ or $\C\}$.
    For $\alpha \in F$ with $|\alpha| \leq 1$ and $\tveca = [\ta_1, \cdots, \ta_{k}]^{\top}\in F^k$ with
    all $\ta$s distinct and $\|\tveca\|_{\infty} \leq 1$. We have,
    $$\min_{\vecx\in F^k}\|V_{\tveca}\vecx - M(\alpha)\|_2 \leq 2^{O(k)} \prod_{j=1}^{k} |\alpha - \ta_j|.$$
\end{lemma}

\begin{proof}
    Let $\Delta \ta_j = \ta_j - \alpha$ and let $\vecx^* = [x_1^*, \cdots, x_k^*]^{\top}$ be the solution to linear equation
    \begin{align*}
        \begin{bmatrix}
    		1 & 1 & \cdots & 1 \\
    		\Delta \ta_1^1 & \Delta \ta_2^1 & \cdots & \Delta \ta_{k}^1 \\
    		\vdots & \vdots & \ddots & \vdots \\
    		\Delta \ta_1^{k-1} & \Delta \ta_2^{k-1} & \cdots & \Delta \ta_{k}^{k-1}
        \end{bmatrix}
    	\begin{bmatrix}
    		x_1^* \\
    		x_2^*  \\
    		\vdots  \\
    		x_k^* \\
    	\end{bmatrix} =
    	\begin{bmatrix}
    		1 \\
    		0  \\
    		\vdots  \\
    		0 \\
    	\end{bmatrix}.
    \end{align*}
    Denote $(V_{\tveca}\vecx^* - M(\alpha))_j$ as the $j$th element of vector $V_{\tveca}\vecx^* - M(\alpha)$. We will show that $|(V_{\tveca}\vecx^* - M(\alpha))_j| \leq 2^{k} \cdot 2^{2j}  \prod_{t=1}^{k} |\Delta \ta_t|$ for any integer $j \geq 0$.It is clear that we have $\sum_{t=1}^{k} x_t^*  \Delta \ta_t^j = 0$ for any $1 \leq j < k$ according to the construction of $\vecx^*$. For larger $j$, we can expand the element by 
    \begin{align*}
        (V_{\tveca}\vecx^* - M(\alpha))_j &= \sum_{t=1}^{k} x^*_t \ta_t^j - \alpha^j 
        = \sum_{t=1}^{k} x^*_t \sum_{p=0}^{j} \binom{j}{p} \Delta \ta_t^p \alpha^{j-p} - \alpha^j \\
        &= \sum_{p=1}^{j} \binom{j}{p} \alpha^{j-p} \sum_{t=1}^{k} x_t^* \Delta \ta_t^p.
    \end{align*}
    where the second equality follows from the definition of $\Delta \ta_j$.
    So it is sufficient to bound $\sum_{t=1}^{k} x_t^* \Delta \ta_t^j$ for any $j\geq k$.
    We have that 
    % Considering the relation between the inverse matrix
    % the corresponding adjoint matrix, we have for example
    % \begin{align*}
    %     x_1^* = \left|\begin{matrix}
    % 		\Delta \ta_2^1 & \cdots & \Delta \ta_{k}^1 \\
    % 		\vdots & \ddots & \vdots \\
    % 		\Delta \ta_2^{k-1} & \cdots & \Delta \ta_{k}^{k-1}
    %     \end{matrix}\right| \cdot  \left|\begin{matrix}
    % 		1 & 1 & \cdots & 1 \\
    % 		\Delta \ta_1^1 & \Delta \ta_2^1 & \cdots & \Delta \ta_{k}^1 \\
    % 		\vdots & \vdots & \ddots & \vdots \\
    % 		\Delta \ta_1^{k-1} & \Delta \ta_2^{k-1} & \cdots & \Delta \ta_{k}^{k-1}
    %     \end{matrix}\right|^{-1}.
    % \end{align*}
    % Thus, for $j \geq k$, we have that
    \begin{align*}
        \sum_{t=1}^{k} x_t^*  \Delta \ta_t^j 
        &=
        \left|\begin{matrix}
    		\Delta \ta_1^{j} & \Delta \ta_2^{j} & \cdots & \Delta \ta_{k}^{j} \\
    		\Delta \ta_1^1 & \Delta \ta_2^1 & \cdots & \Delta \ta_{k}^1 \\
    		\vdots & \vdots & \ddots & \vdots \\
    		\Delta \ta_1^{k-1} & \Delta \ta_2^{k-1} & \cdots & \Delta \ta_{k}^{k-1}
        \end{matrix}\right| \cdot \left|\begin{matrix}
    		1 & 1 & \cdots & 1 \\
    		\Delta \ta_1^1 & \Delta \ta_2^1 & \cdots & \Delta \ta_{k}^1 \\
    		\vdots & \vdots & \ddots & \vdots \\
    		\Delta \ta_1^{k-1} & \Delta \ta_2^{k-1} & \cdots & \Delta \ta_{k}^{k-1}
        \end{matrix}\right|^{-1}\\
        &= \prod_{t=1}^{k} \Delta \ta_t \cdot \sum_{s \subseteq (k)^{j-k}} \prod_{t=1}^{k} \Delta \ta_t^{s_t}, 
    \end{align*} where the first equality dues to the relation between inverse matrix and adjoint matrix, and the other equation holds from Lemma~\ref{lm:dim1-step3-matrix}.
    Since $|\Delta \ta_t| \leq |\alpha|+|\ta_t| \leq 2$, $(k)^{j-k}$ has exactly $\binom{j}{k}$ terms, we can bound the summation by  
    \begin{align*}
        \left|\sum_{t=1}^{k} x_t^*  \Delta \ta_t^j\right| \leq 2^{k} \cdot 2^{j-k} \prod_{t=1}^{k} |\Delta \ta_t|.
    \end{align*}
    Plugging in this result, we can reach that 
    \begin{align*}
        |(V_{\tveca}\vecx^* - M(\alpha))_j| &= \left|\sum_{p=1}^{j} \binom{j}{p} \alpha^{j-p} \sum_{t=1}^{k} x_t^* \Delta \ta_t^p\right|  
        \leq \sum_{p=k}^{j} \binom{j}{p} \cdot 2^{k} \cdot 2^{p-k} \prod_{t=1}^{k} |\Delta \ta_t| \\ &\leq 2^{k} \cdot 2^{2j}  \prod_{t=1}^{k} |\Delta \ta_t|.
    \end{align*}
    As a result, we conclude that 
    \begin{align*}
        \min_{\vecx\in F^k}\|V_{\tveca}\vecx - M(\alpha)\|_2 &\leq \|V_{\tveca}\vecx^* - M(\alpha)\|_2
        \leq \sqrt{\sum_{j=k}^{2k-1} \left(2^{k} \cdot 2^{2j}  \prod_{t=1}^{k} |\Delta \ta_t|\right)^2}
        \leq 2^{O(k)} \prod_{j=1}^{k} |\alpha - \ta_j|.
    \end{align*}
\end{proof}

Lemma~\ref{lm:dim1-step2} shows that every discrete distribution in ground truth $\mix$ has a close spike in recovered positions $\tveca$ and thus can be well approximated according to the above lemma. Therefore, the mixture of the discrete distributions can also be well approximated with support $\tveca$. As the following lemma shows, solving linear regression in Line~\ref{line:dim1-hw} finds a weight vector $\tw$ in which the corresponding moment vector $V_{\tveca}\hvecw$ is close to the ground truth $M(\mix)$.

\begin{lemma}\label{lm:dim1-step4-pre}
    Let $\hvecw = [\hw_1, \cdots, \hw_{k}]^{\top} \in \R^k$ be the intermediate result (Line~\ref{line:dim1-hw}) in Algorithm~\ref{alg:dim1}. Then,
    $$\|V_{\tveca}\hvecw - M(\mix)\|_{\infty} \leq 2^{O(k)} \cdot \noise^{\frac{1}{2}}.$$
\end{lemma}

\begin{proof}
    Firstly, we have
    \begin{align*}
        \|V_{\tveca}\hvecw - M(\mix)\|_2 &\leq \|V_{\tveca}\hvecw - M'\|_2 + \|M(\mix) - M'\|_2  \\
        &\leq O(1) \cdot \min_{\vecx \in \R^k} \|V_{\tveca} \vecx - M'\|_2 + \|M(\mix) - M'\|_2 \\ 
        &\leq O(1) \cdot \min_{\vecx \in \R^k} \|V_{\tveca} \vecx - M(\mix)\|_2 + O(1) \cdot \|M(\mix) - M'\|_2
    \end{align*}
    where the first and third inequalities hold due to triangle inequality, and the second inequality holds since $\|V_{\tveca} \hvecw - M'\|_2^2 \leq O(1) \cdot \min_{\vecx \in \R^k} \|V_{\tveca} \vecx - M'\|_2^2$
    (we only find an $O(1)$-approximation in this step).
    
    For the first term, we can see that
    \begin{align*}
        \min_{\vecx \in \R^k} \|V_{\tveca} \vecx - M(\mix)\|_2 &= \min_{\vecx_1, \cdots, \vecx_k \in \R^k} \left\|\sum_{i=1}^{k}w_i V_{\tveca}\vecx_i - \sum_{i=1}^{k}w_i M(\alpha_i)\right\|_2 \\
        &\leq \sum_{i=1}^{k} w_i \min_{\vecx \in \R^k} \|V_{\tveca}\vecx - M(\alpha_i)\|_2 \\
        &\leq 2^{O(k)} \sum_{i=1}^{k} w_i \prod_{j=1}^{k} |\alpha_i - \ta_j| & (\text{Lemma~\ref{lm:step3}}) \\
        &\leq 2^{O(k)} \sqrt{\sum_{i=1}^{k} w_i \prod_{j=1}^{k} |\alpha_i - \ta_j|^2} & (\text{AM-QM Inequality}, \vecw \in \Delta_{k-1}) \\
        &\leq 2^{O(k)} \cdot \sqrt{\noise}. & (\text{Lemma~\ref{lm:dim1-step2}}).
    \end{align*}
    For the second term, $\|M(\mix) - M'\|_2 \leq \sqrt{k} \|M(\mix) - M'\|_{\infty} \leq \sqrt{k} \cdot \noise$.
    As a result, $$\|V_{\tveca}\hvecw - M(\mix)\|_{\infty} \leq \|V_{\tveca}\hvecw - M(\mix)\|_2 \leq O(1) \cdot 2^{O(k)} \cdot \sqrt{\noise} + O(1) \cdot \sqrt{k} \cdot \noise \leq 2^{O(k)} \cdot \zeta^{\frac{1}{2}}.$$
\end{proof}

The moment-transportation inequality (Lemma~\ref{lm:inequality1d})
requires the input to be signed measure. 
So, we normalize the recovered spikes in Line~\ref{line:dim1-tw}. 
The following lemma shows the moment vector would not change too much after normalization.

\begin{lemma}\label{lm:dim1-step4}  
  Let $\tmix = (\tveca, \tvecw)$ be the intermediate result 
  (Line~\ref{line:dim1-tmix}) in Algorithm~\ref{alg:dim1}. 
  Then, we have 
    $$\Mdis_K(\tmix, \mix) \leq 2^{O(k)} \cdot \noise^{\frac{1}{2}}.$$
\end{lemma}

\begin{proof} 
    According to the definition of $V_{\tveca}$ and $M(\mix)$, we have $(V_{\tveca}\hvecw - M(\mix))_1 = \sum_{i=1}^{k} \hw_i - 1$.
    Therefore, $|\sum_{i=1}^{k} \hw_i - 1| \leq |(V_{\tveca}\hvecw - M(\mix))_1| \leq \|V_{\tveca}\hvecw - M(\mix)\|_{\infty} \leq 2^{O(k)} \cdot \noise^{\frac{1}{2}}.$ 
    
    Note that $\|\tvecw - \hvecw\|_{1} = \|\tvecw\|_1 - \|\hvecw\|_{1} = |(\sum_{i=1}^{k} \hw_i)^{-1} - 1|$. For $2^{O(k)} \cdot \noise^{\frac{1}{2}} \leq 1/2$, we can conclude that $\|\tvecw - \hvecw\|_{1} = |(\sum_{i=1}^{k} \hw_i)^{-1} - 1| \leq 2 |\sum_{i=1}^{k} \hw_i - 1| \leq  2^{O(k)} \cdot \noise^{\frac{1}{2}}.$
    
    Thus, we can reach that
    \begin{align*}
        \Mdis_K(\tmix, \mix)&=\|M(\tmix) - M(\mix)\|_\infty = \|V_{\tveca} \tvecw - M(\mix)\|_{\infty} \\
        &\leq  \|V_{\tveca} \hvecw - M(\mix)\|_{\infty} + \|V_{\tveca} \tvecw - V_{\tveca} \hvecw)\|_{\infty} \\
        &\leq  \|V_{\tveca} \hvecw - M(\mix)\|_{\infty} + \|V_{\tveca}\|_\infty \|\tvecw - \hvecw\|_1 \\
        &\leq 2^{O(k)} \cdot \noise^{\frac{1}{2}}.
    \end{align*}
    where the first inequality holds due to triangle inequality and the third holds since $\|V_{\tveca}\|_{\infty} \leq 1$.
\end{proof}

Moreover, the weight can still be negative after normalization. To find a reconstruction in the original space, we want to find a mixture of discrete distributions in $\Spike(\simplex_1, \simplex_{k-1})$ that is close to $\tmix$. However, this step can cast a huge impact on the moment. Thus, we directly estimate the influence in terms of the transportation distance instead of the moment distance. We note the transportation distance in the next lemma 
is defined possibly for non-probability measures (see Equation \eqref{eq:signedtrans}
in Appendix~\ref{app:transfornonprob} for details).

\begin{lemma}\label{lm:step5} 
    Let $\convex$ be a compact convex set in $\mathbb{R}^d$ and let $\tmix = (\tveca, \tvecw) \in \Spike(\convex, \Sigma_{k-1})$ be a $k$-spike distribution over support $\convex$. 
    Then, we have 
    $$\min_{\cmix \in \Spike(\tveca, \simplex_{k-1})} \tran(\tmix, \cmix) \leq 2 \min_{\bmix \in \Spike(\convex, \simplex_{k-1})} \tran(\tmix, \bmix).$$
    Note that the minimization on the left is over support $\tveca$.
\end{lemma}

\begin{proof}
    Consider any $\bmix = (\bveca, \bvecw) \in \Spike(\convex, \Delta_{k-1})$ such that $\bveca = [\ba_1, \cdots, \ba_k]^{\top}$ and $\bvecw = [\bw_1, \cdots, \bw_k]^{\top}$. Let $\mix' = \arg \min_{\mix' \in \Spike(\tveca, \Sigma_{k-1})} \tran(\mix', \bmix)$. Since $\tmix \in \Spike(\tveca, \Sigma_{k-1})$, we have $\tran(\mix', \bmix) \leq \tran(\tmix, \bmix)$. From triangle inequality, we have $$\tran(\tmix, \mix') \leq \tran(\tmix, \bmix) + \tran(\bmix, \mix') \leq 2\tran(\tmix, \bmix).$$ Now, we start to show $\mix' \in \Spike(\tveca, \simplex_{k-1})$, i.e., the weight of spikes in the closest distribution to $\bmix$ with support $\tveca$ should be all non-negative.
   
    Towards a contradiction, assume the optimal distribution $\mix'$ contains at least one negative spike, that is, let $(\mix')^- = (\veca, \max\{-\vecw, \mathbf{0}\})$ be the negative spikes of $\mix'$, we have $(\mix')^- \neq \mathbf{0}$. In this case, since $\bmix$ is a probability distribution, $\tran(\bmix, \mix') = \tran(\bmix + (\mix')^-, (\mix')^+)$ where $(\mix')^+ = (\veca, \max\{\vecw, \mathbf{0}\})$ is positive spikes of $\mix'$.
   
    Let $\mu'$ be the optimal matching distribution corresponding to $\tran(\bmix + (\mix')^-, (\mix')^+)$. From its definition, $\mu'$ is non-negative. Let $\mu''$ be the matching distribution in $\mu'$ that maps $\bmix$ and $(\mix')^+$. Let $\mix''$ be the marginal distribution of $\mu''$ other than $\bmix$. In this case, $\mu''$ is a valid matching distribution for $\tran(\bmix, \mix'')$. Moreover, $\mu''$ gives less transportation distance than $\mu'$ does since the cost on the eliminated terms is non-negative. Hence, $\tran(\bmix, \mix'') < \tran(\bmix, \mix')$, which is a contradiction.
    
    As a result, $\mix' \in \Spike(\tveca, \Delta_{k-1})$. Therefore, $\min_{\cmix \in \Spike(\veca, \Delta_{k-1})} \tran(\tmix, \cmix) \leq \tran(\tmix, \mix')$. By taking minimization over $\bmix$, we can conclude that
    \begin{align*}
        \min_{\cmix \in \Spike(\veca, \Delta_{k-1})} \tran(\tmix, \cmix) \leq 2 \min_{\bmix \in \Spike(\convex, \Delta_{k-1})} \tran(\tmix, \bmix).
    \end{align*}
\end{proof}

Finally, we are ready to bound the transportation distance error of the reconstructed $k$-spike distribution:

\begin{lemma}\label{lm:dim1-reconstruction} 
  Let $\cmix = (\tveca, \cvecw)$ be the final result 
  (Line~\ref{line:dim1-cmix}) in Algorithm~\ref{alg:dim1}. 
  Then, we have 
  $$\tran(\cmix, \mix) \leq O(k\noise^{\frac{1}{4k-2}}).$$
\end{lemma}

\begin{proof}
    Combining Lemma~\ref{lm:inequality1d} and Lemma~\ref{lm:dim1-step4}, we can see
    $$\tran(\tmix, \mix) \leq O(k \cdot (2^{O(k)} \cdot \noise^{\frac{1}{2}})^{\frac{1}{2k-1}}) \leq O(k \noise^{\frac{1}{4k-2}}).$$
    Moreover, from Lemma~\ref{lm:step5}, we have 
    \begin{align*}
        \tran(\tmix, \cmix) \leq 2 \min_{\bmix \in \Spike(\simplex_1, \simplex_{k-1})} \tran(\tmix, \bmix) \leq 2 \tran(\tmix, \mix) = O(k \noise^{\frac{1}{4k-2}}).
    \end{align*}
    Finally, by triangle inequality, $\tran(\cmix, \mix) \leq \tran(\cmix, \tmix) + \tran(\tmix, \mix) \leq O(k \noise^{\frac{1}{4k-2}}).$
\end{proof}

By choosing $\noise = (\eps/k)^{O(k)}$, the previous lemma directly implies Theorem~\ref{thm:1dim-kspikecoin}.

% \subsection{A Remark on Adaptive Rate Performance}

% It is reasonable to expect a tighter guarantee when the components are better separated. In this section, we will show an adaptive rate upper bound on the performance of Algorithm~\ref{alg:dim1}. We first introduce the separation assumption following from \citet{wu2020optimal}:

% \begin{definition}
%     A $k$-spike distribution $\mix$ is saying to have $k_0 (\minsep, w_{\min})$-separated clusters if there exists a partition $C_1, \cdots, C_{k_0}$ of $[k]$ that
%     \begin{itemize}
%         \item $\minsep \leq \|\veca_i - \veca_j\|_2$ for any $i \in C_p, j \in C_{p'}, p \neq p'$
%         \item $w_{\min} \leq \sum_{i \in C_p} w_i$ for each $p$
%     \end{itemize}
% \end{definition}

% The next Lemma is a counterpart of Lemma~\ref{lm:inequality1d} but with separation assumption realized.

\section{An Algorithm for the Two-Dimensional Problem}
\label{sec:2d-recover}

We can generalize the 1-dimension algorithm described in Section~\ref{sec:1d-recover} to 2-dimension. 
Let $\mix:=(\veca, \vecw) \in \Spike(\simplex_2,\simplex_{k-1})$ 
be the underlying mixture 
where $\veca = \{\veca_1, \cdots, \veca_k\}$ for $\veca_{i} = (\alpha_{i,1}, \alpha_{i,2})$ and $\vecw = [w_1, \cdots, w_k]^{\top} \in \simplex_{k-1}$.
The true moments can be computed according to $M_{i,j}(\mix) = \sum_{t=1}^k w_t \alpha_{t,1}^i \alpha_{t,2}^j$. 
The input is the noisy moments $M'_{i,j}$ in $0\leq i, j, i+j \leq 2k-1$ such that $|M'_{i,j} - M_{i,j}(\mix)| \leq \noise$. We further assume that $M'_{0,0}=M_{0,0}(\mix) = 1$ and $\noise \leq 2^{-\Omega(k)}$.

The key idea is simple: 
a distribution supported in $\R^2$ can be mapped to a distribution supported in the complex plane $\C$. In particular, we define the complex set $\simplex_{\C} = \{a+b\imag \mid (a,b) \in \simplex_2\}$.
Moreover, we denote $\vecb = [\beta_{1}, \cdots, \beta_{k}]^{\top} := [\alpha_{1,1} + \alpha_{1,2} \imag, \cdots, \alpha_{k,1} + \alpha_{k,2} \imag]^{\top} \in \simplex_{\C}^k$, and define $\pmix := (\vecb, \vecw) \in \Spike(\simplex_{\C}^k, \simplex_{k-1})$ to be the complex mixture corresponding to $\mix$. 
%We note that $|\beta_i| \leq 1$ because our restriction on $\simplex_2$. 
The corresponding moments of $\pmix$ can thus be defined as $G_{i,j}(\pmix) = \sum_{t=1}^{k} w_t (\beta_t^{\dagger})^i \beta_t^j$.

% \jiannote{need an english description of the algorithm}

For $G = [G_{i,j}]^{\top}_{0 \leq i \leq k; 0 \leq j \leq k-1}$, denote
\begin{align*}
	A_G := \begin{bmatrix}
		G_{0,0} & G_{0,1} & \cdots & G_{0,k-1} \\
		G_{1,0} & G_{1,1} & \cdots & G_{1,k-1} \\
		\vdots & \vdots & \ddots & \vdots \\
		G_{k-1,0} & G_{1,k-1} & \cdots & G_{k-1,k-1}
	\end{bmatrix},
	b_G := \begin{bmatrix}
		G_{0,k} \\
		G_{1,k} \\
		\vdots \\
		G_{k-1,k}
	\end{bmatrix},
	M_G := \begin{bmatrix}
	    G_{0,0} \\
	    G_{0,1} \\
	    \vdots \\
	    G_{0,k-1}
	\end{bmatrix}.
\end{align*}

\begin{algorithm}[!ht]
\caption{Reconstruction Algorithm in Two-Dimensional Problem}\label{alg:dim2}
\begin{algorithmic}[1]
    \REQUIRE number of spikes $k$, noisy moments $M'(\cdot)$, noise level $\noise$ \\
    \ENSURE recovered spike distribution $\cmix$
    \STATE $G' \leftarrow [G'_{i,j} := \sum_{p=0}^{i}\sum_{q=0}^{j} \binom{i}{p}\binom{j}{q} (-\imag)^{i-p} \imag^{j-q} M'_{p+q,i+j-p-q}]^{\top}_{0 \leq i \leq k; 0 \leq j \leq k-1}$
    \label{line:dim2-G'}
    \STATE $\hvecc \leftarrow \arg \min_{\vecx \in \C^k} \| A_{G'} \vecx + b_{G'} \|_2^2 + \noise^2 \|\vecx\|_2^2$ 
    \label{line:dim2-hc}
    \hfill{$\Rightarrow\hvecc = [\hc_0, \cdots, \hc_{k-1}]^{\top} \in \C^k$}
    \STATE $\hvecb \leftarrow \textrm{roots}(\sum_{i=0}^{k-1} \hc_i x^i + x^k)$ 
    \label{line:dim2-hb}
    \hfill{$\Rightarrow\hvecb = [\hb_1, \cdots, \hb_k]^{\top} \in \C^k$}
    \STATE $\bvecb \leftarrow \textrm{project}_{\Delta_{1}^\C}(\hvecb)$
    \label{line:dim2-bb}
    \hfill{$\Rightarrow\bvecb = [\bb_1, \cdots, \bb_{k}]^{\top} \in \Delta_{\C}^k$}
    \STATE $\tvecb \leftarrow \bvecb + \textrm{Noise}(\noise)$
    \label{line:dim2-tb}
    \hfill{$\Rightarrow\tvecb = [\tb_1, \cdots, \tb_{k}]^{\top} \in \Delta_{\C}^k$}
    \STATE $\hvecw \leftarrow \arg \min_{\vecx \in \C^k} \|V_{\tvecb} \vecx - M_{G'}\|_2^2$ 
    \label{line:dim2-hw}
    \hfill{$\Rightarrow\hvecw = [\hw_1, \cdots, \hw_{k}]^{\top}  \in \C^k$}
    \STATE $\tvecw \leftarrow \hvecw / (\sum_{i=1}^{k} \hw_i)$
    \label{line:dim2-tw}
    \hfill{$\Rightarrow\tvecw = [\tw_1, \cdots, \tw_{k}]^{\top}  \in \Sigma_{k-1}^{\C}$}
    \STATE $\tpmix \leftarrow (\tvecb, \tvecw)$ 
    \label{line:dim2-tpmix}
    \hfill{$\Rightarrow\tpmix \in \Spike(\Delta_{\C}, \Sigma_{k-1}^{\C})$}
    \STATE $\tmix \leftarrow ([\textrm{real}(\tvecb), \textrm{imag}(\tvecb)], \textrm{real}(\tvecw))$
    \label{line:dim2-tmix}
    \hfill{$\Rightarrow\tmix \in \Spike(\Delta_2, \Sigma_{k-1})$}
    \STATE $\cvecw \leftarrow \arg \min_{\vecx \in \Delta_{k-1}} \tran(\tmix, (\tveca, \vecx))$
    \label{line:dim2-cw}
    \STATE $\cmix \leftarrow (\tveca, \cvecw)$ 
    \label{line:dim2-cmix}
    \hfill{$\Rightarrow\cmix \in \Spike(\Delta_2, \Delta_{k-1})$}
\end{algorithmic}
\end{algorithm}

The pseudocode can be found in Algorithm~\ref{alg:dim2}. The algorithm takes the number of spikes $k$ and the error bound $\noise$ to reconstruct the original spike distribution using empirical moments $M'$.
Now, we describe the implementation details of our algorithm. 

We first calculate the empirical complete moments $G'$ of $\pmix$ in Line~\ref{line:dim2-G'}. Since the process follows relationship between $G(\pmix)$ and $M(\mix)$ (see Lemma~\ref{lm:dim2-complex-corr}), $G'$ would be an estimation of ground truth $G(\pmix)$. Since it is a two-dimensional convolution, this step can be implemented in $O(k^3)$ arithmetic operations.

Then, we perform a ridge regression to obtain $\vecc$ in Line~\ref{line:dim2-hc}. We note that $A_{G(\pmix)}\vecc + b_{G(\pmix)} = \mathbf{0}$. (see Lemma~\ref{lm:dim2-char-poly}). Hence, $\hvecc$ can be seen as an approximation of $\vecc$. The explicit solution of this ridge regression is $\hvecc = (A^{\top}_{G'} A_{G'} + \noise^2 I)^{-1} A^{\top}_{G'} b_{G'}$ which can be computed in $O(k^3)$ time.

From Line~\ref{line:dim2-hb} to Line\ref{line:dim2-tb}, we aim to estimate the positions of the spikes of complex correspondence, i.e., $\beta_i$ s. Similar to the 1-dimensional case, we solve the roots of polynomial $\sum_{i=0}^{k-1} \hc_i x^i + x^k$. Note that the roots we found may locate outside $\Delta_{\C}$, which is the support of the ground truth. Thus, we use the description of $\Delta_2$ to project the solutions back to $\Delta_{\C}$ and inject small noise to ensure that all values are distinct and $\tvecb$ are still in $\Delta_{\C}$. Any noise of size at most $\noise$ suffices here. We note that from the definition of complex correspondence, the realized spikes in the original space are $[\textrm{real}(\tvecb), \textrm{imag}(\tvecb)]$. For implementation, this step can be done using the numerical root-finding algorithm allowing $\noise$ additive noise in $O(k^{1+o(1)} \cdot \log \log (1/\noise))$ arithmetic operations.

After that, we aim to recover the weight of the spikes. Line~\ref{line:dim2-hw} is a linear regression defined by Vandemonde matrix $V_{\vecb}$. Since $\vecb$ may be complex numbers, we would calculate over complex space. Again, we would apply moment inequality to bound the recovered parameter error. Hence, we normalize $\hvecw$ and get $\tvecw$ in Line~\ref{line:dim2-tmix}. Note from our definition of transportation distance, the real components and imaginary components are considered separately. Hence, we can discard the imaginary parts of $\tvecw$ and reconstruct the $k$-spike distribution in Line~\ref{line:dim2-tw}. Using linear regression, this step can be done in $O(k^3)$ time.

The remaining thing is to deal with the negative weights of $\tvecw$. We find a close $k$-spike distribution in $\Spike(\Delta_2, \simplex_{k-1})$ in Line~\ref{line:dim2-tw}. The optimization problem is equivalent to finding a transportation from $(\tveca, \tvecw^{-})$ to $(\tveca, \tvecw^{+})$ where $\tvecw^{-}$ and $\tvecw^{+}$ are the negative components and positive components of $\tvecw$ respectively, i.e., $\tw^{-}_i = \max\{0, -\tw_i\}$ and $\tw^{+}_i = \max\{0, \tw_i\}$. This transportation can be found using the standard network flow technique, which takes $O(k^3)$ time.

Since the noise satisfies $\noise \leq 2^{-\Omega(k)}$, the whole algorithm requires $O(k^3)$ arithmetic operations.

\subsection{Error Analysis}

From now, we bound the reconstruction error of the algorithm. The following lemma presents the relationship between the moments of the original $k$-spike distribution and its complex correspondence.

\begin{lemma}\label{lm:dim2-complex-corr}
Let $\mix = (\veca, \vecw) \in \Spike(\Delta_2, \Delta_{k-1})$ and let $\pmix = (\vecb, \vecw) \in \Spike(\Delta_{\C}, \Delta_{k-1})$ be its complex correspondence. Then, the complete moment satisfies
\begin{align*}
    G_{i,j} (\pmix) &= \sum_{p=0}^{i}\sum_{q=0}^{j} \binom{i}{p}\binom{j}{q} (-\imag)^{i-p} \imag^{j-q} M_{p+q,i+j-p-q}(\mix).
\end{align*}
\end{lemma}

\begin{proof}
    According to the definition,
    \begin{align*}
        G_{i,j} (\pmix) &= \sum_{t=1}^{k} w_t (\beta_t^{\dagger})^i \beta_t^j
            = \sum_{t=1}^k w_t (\alpha_{t,1} - \alpha_{t,2} \imag)^i (\alpha_{t,1} + \alpha_{t,2} \imag)^j \\
            &= \sum_{t=1}^k w_t \sum_{p=0}^i \binom{i}{p} \alpha_{t,1}^p \alpha_{t,2}^{i-p} (-\imag)^{i-p} \sum_{q=0}^j \binom{j}{q} \alpha_{t,1}^q \alpha_{t,2}^{j-q} \imag^{j-q} \\
            &= \sum_{p=0}^{i}\sum_{q=0}^{j} \binom{i}{p}\binom{j}{q} (-\imag)^{i-p} \imag^{j-q} M_{p+q,i+j-p-q}(\mix).
    \end{align*}
\end{proof}

The following lemma is a complex extension of Lemma~\ref{lm:dim1-char-poly}.

\begin{lemma}\label{lm:dim2-char-poly}
    Let $\pmix = (\vecb, \vecw) \in \Spike(\Delta_{\C}, \simplex_{k-1})$ where $\vecb = [\beta_1, \cdots, \beta_k]^{\top}$ and $\vecw = [w_1, \cdots, w_k]^{\top}$. 
    Let $\vecc = [c_0, \cdots, c_{k-1}]^{\top} \in \C^k$ such that $\prod_{i=1}^{k} (x - \alpha_i) = \sum_{i=0}^{k-1} c_i x^i + x^k$. 
    Then,
    $$\sum_{j=0}^{k-1} G_{i,j}(\mix) c_j + G_{i,k}(\mix) = \sum_{j=0}^{k-1} G_{j,i}(\mix) c_j^{\dagger} + G_{k,i}(\mix) = 0,$$
    for all $i \geq 0$.
    In matrix form, the equation can be written as:
    $$A_{G(\mix)} \vecc + b_{G(\mix)} = \mathbf{0}.$$
\end{lemma}
\begin{proof}
    According to the definition of $G$, it holds that
    \begin{align*}
        \sum_{j=0}^{k-1} G_{i,j}(\pmix) c_j + G_{k,i}(\pmix)
	 	&= \sum_{j=0}^{k-1} \sum_{t=1}^k w_t (\beta_t^\dagger)^i\beta_t^{j} c_j + \sum_{t=1}^k w_t (\beta_t^\dagger)^i \beta_t^{k}  \\
		&= \sum_{t=1}^k w_t (\beta_t^\dagger)^i \left(\sum_{j=0}^{k-1} c_j \beta_t^j + \beta_t^k\right) \\
		&= \sum_{t=1}^k w_t (\beta_t^\dagger)^i \prod_{j=1}^{k} (\beta_t - \beta_j) = 0,
    \end{align*}
    \begin{align*}
        \sum_{j=0}^{k-1} G_{j,i}(\pmix) c_j^{\dagger} + G_{k,i}(\pmix) &= \sum_{j=0}^{k-1} \sum_{t=1}^k w_t (\beta_t^\dagger)^j \beta_t^{i} c_j^{\dagger} + \sum_{t=1}^k w_t (\beta_t^\dagger)^k \beta_t^{i}  \\
        &= \left(\sum_{j=0}^{k-1} \sum_{t=1}^k w_t (\beta_t^\dagger)^i \beta_t^{j} c_j + \sum_{t=1}^k w_t (\beta_t^\dagger)^i \beta_t^{k}\right)^{\dagger} = 0.
    \end{align*}
\end{proof}

Similar to Lemma~\ref{lm:dim1-step1}, we present some useful properties of $\hvecc$.

\begin{lemma} \label{lm:dim2-step1}
Let $\hvecc = [\hc_0, \cdots, \hc_{k-1}]^{\top} \in \R^k$ be the intermediate result (Line~\ref{line:dim2-hc}) in the algorithm. Then, $\|G' - G(\pmix)\|_{\infty} \leq 2^{2k} \cdot \noise$, $\|\vecc\|_1 \leq 2^k$, $\|\hvecc\|_1 \leq 2^{O(k)}$ and $\|A_{G(\pmix)} \hvecc + b_{G(\pmix)}\|_\infty \leq 2^{O(k)} \cdot \noise$.
\end{lemma}

\begin{proof}
    From Vieta's formulas, we have 
    $c_i = \sum_{S \in \binom{[k]}{k-i}} \prod_{j\in S} (-\alpha_j).$
    Thus,
    $$\|\vecc\|_1 = \sum_{i=0}^{k-1} |c_i| \leq \sum_{S \in 2^{[k]}} \prod_{j\in S} |\beta_j| = \prod_{i=1}^{k} (1+|\beta_i|) \leq 2^k.$$
    where the last inequality holds because $|\beta_i| = \sqrt{\alpha_{i,1}^2 + \alpha_{i,2}^2} \leq 1$ for all $i$.
    %\jiannote{should $\leq$ be $=$??}
    
    According to Lemma~\ref{lm:dim2-complex-corr},
    \begin{align*}
        |G'_{i,j} - G_{i,j}(\pmix)| &\leq \sum_{p=0}^{i}\sum_{q=0}^{j} \left|\binom{i}{p}\binom{j}{q} (-\imag)^{i-p} \imag^{j-q} \right| \|M' - M(\mix)\|_{\infty} \leq 2^{i+j} \cdot \noise.
    \end{align*}
    This shows that $\|G' - G(\pmix)\|_{\infty} \leq 2^{2k} \cdot \noise$.
    
    From Lemma~\ref{lm:dim1-char-poly}, we can see that $\|A_{M(\mix)} \vecc + b_{M(\mix)}\|_{\infty} = 0$.
    Therefore, 
    \begin{align*}
		\|A_{G'} \vecc + b_{G'}\|_\infty &\leq \|A_{G(\pmix)} \vecc + b_{G(\pmix)}\|_{\infty} + \|A_{G'} \vecc - A_{G(\pmix)} \vecc\|_{\infty} + \|b_{G'} - b_{G(\pmix)}\|_{\infty} \\
		&\leq \|A_{G(\pmix)} \vecc + b_{G(\pmix)}\|_{\infty} + \|(A_{G'} - A_{G(\pmix)}) \vecc\|_{\infty} + \|b_{G'} - b_{G(\pmix)}\|_{\infty} \\
		&\leq \|A_{G(\pmix)} \vecc + b_{G(\pmix)}\|_{\infty} + \|A_{G'} - A_{G(\pmix)}\|_{\infty} \|\vecc\|_1 + \|b_{G'} - b_{G(\pmix)}\|_{\infty} \\
		&\leq \|A_{G(\pmix)} \vecc + b_{G(\pmix)}\|_{\infty} +  \|G' - G(\pmix)\|_{\infty} (\|\vecc\|_1 + 1) \\
		&\leq 2^{O(k)} \cdot \noise.
	\end{align*}
	The fourth inequality holds since $(A_{G'} - A_{G(\pmix)})_{i,j}=(G'-G(\pmix))_{i,j}$ and $(b_{G'}-b_{G(\pmix)})_i = (G'-G(\pmix))_{i,k}$.
	
    From the definition of $\hvecc$, we can see that
	\begin{align*}
		\|A_{G'} \hvecc + b_{G'}\|_2^2 + \noise^2 \|\hvecc\|_2^2 
		&\leq \|A_{G'} \vecc + b_{G'}\|_2^2 + \noise^2 \|\vecc\|_2^2 \\
		&\leq k \|A_{G'} \vecc + b_{G'}\|_\infty^2 + \noise^2 \|\vecc\|_1^2 \\
		&\leq 2^{O(k)} \cdot \noise^2.
	\end{align*}
	The second inequality holds since $\|\vecx\|_2 \leq \|\vecx\|_1$ and $\|\vecx\|_2 \leq \sqrt{k} \|\vecx\|_{\infty}$ holds for any vector $\vecx \in \R^k$.
	
	Now, we get that
	\begin{align*}
		\|A_{G'} \hvecc + b_{G'}\|_\infty &\leq \|A_{G'} \hvecc + b_{G'}\|_2 \leq 2^{O(k)} \cdot \noise, \\
		\|\hvecc\|_1 &\leq \sqrt{k} \|\hvecc\|_2 \leq 2^{O(k)}.
	\end{align*}
    Finally, we can bound $\|A_{G(\pmix)} \hvecc + b_{G(\pmix)}\|_\infty$ according to
	\begin{align*}
	    \|A_{G(\pmix)} \hvecc + b_{G(\pmix)}\|_\infty &\leq \|A_{G'} \hvecc + b_{G'}\|_\infty + \|A_{G'} \hvecc - A_{G(\pmix)} \hvecc\|_{\infty} + \|b_{G'} - b_{G(\pmix)}\|_{\infty} \\
	    &\leq \|A_{G'} \hvecc + b_{G'}\|_\infty + \|A_{G'} - A_{G(\pmix)}\|_{\infty} \| \hvecc\|_1 + \|b_{G'} - b_{G(\pmix)}\|_{\infty} \\
        &\leq \|A_{G'} \hvecc + b_{G'}\|_\infty + \|G' - G(\pmix)\|_{\infty} (\|\hvecc\|_1 + 1) \\
        &\leq 2^{O(k)} \cdot \noise,
	\end{align*}
    which finishes the proof.
\end{proof}

We then show $\tvecw$ obtained in Line~\ref{line:dim2-tw} is close to
the ground truth $\vecb$.

\begin{lemma} \label{lm:dim2-step2-1}
    Let $\hvecb = [\hb_1, \cdots, \hb_{k}]^{\top}\in \C^k$  be the intermediate result (Line~\ref{line:dim2-hb}) in Algorithm~\ref{alg:dim2}. Then, 
    \begin{align*}
        \sum_{i=1}^{k} w_i \prod_{j=1}^{k} |\beta_i - \hb_j|^2 \leq  2^{O(k)} \cdot \noise.
    \end{align*}
\end{lemma}

\begin{proof}
    Consider
    \begin{align*}
        &\quad\ \sum_{i=1}^{k} w_i \prod_{j=1}^{k} |\beta_i - \widehat{\beta}_j|^2
        = \sum_{i=1}^{k} w_i \left|\prod_{j=1}^{k} (\beta_i - \widehat{\beta}_j)\right|^2
        = \sum_{i=1}^{k} w_i \left|\sum_{j=0}^{k-1} \hc_j \beta_i^t + \beta_i^k\right|^2 \\
        &= \sum_{i=1}^{k} w_i \left(\sum_{j=0}^{k-1} \hc_j \beta_i^t + \beta_i^k\right)^{\dagger}\left(\sum_{j=0}^{k-1} \hc_j \beta_i^t + \beta_i^k\right)  \\
        &= \sum_{i=1}^{k} w_i \left(\sum_{p=0}^{k-1}\sum_{q=0}^{k-1}\hc_p^{\dagger} \hc_q (\beta_i^{\dagger})^p \beta_i^q + \sum_{p=0}^{k-1} \hc_p^{\dagger} (\beta_i^{\dagger})^p \beta_i^{k} + \sum_{p=0}^{k-1} \hc_p \beta_i^{p} (\beta_i^{\dagger})^k + (\beta_i^{\dagger})^k \beta_i^{k} \right) \\
        &= \sum_{p=0}^{k-1}\sum_{q=0}^{k-1}\hc_p^{\dagger} \hc_q \sum_{i=1}^{k}w_i (\beta_i^{\dagger})^p \beta_i^q + \sum_{p=0}^{k-1}\hc_p^{\dagger} \sum_{i=1}^{k} w_i (\beta_i^{\dagger})^p \beta_i^{k} + \sum_{p=0}^{k-1}\hc_p \sum_{i=1}^{k} w_i \beta_i^{p} (\beta_i^{\dagger})^k + \sum_{i=1}^k w_i (\beta_i^{\dagger})^k \beta_i^{k} \\
        &= \sum_{p=0}^{k-1} \hc_p^{\dagger} \left(\sum_{q=0}^{k-1} G_{p,q}(\mix) \hc_q + G_{p,k}(\mix)\right) + \sum_{p=0}^{k-1} \hc_p G_{k,p}(\mix) + G_{k,k}(\mix) \\
        &= \sum_{p=0}^{k-1} \hc_p^{\dagger} (A_{G(\mix)} \hc + b_{G(\mix)})_p + \sum_{p=0}^{k-1} \hc_p G_{k,p}(\mix) + G_{k,k}(\mix),
    \end{align*}
    where the last equality holds from the definition of matrix $A_{G(\mix)}$ and vector $b_{G(\mix)}$.
    According to Lemma~\ref{lm:dim2-char-poly}, $G_{k,p}(\mix) = - \sum_{q=0}^{k-1} G_{q,p}(\mix) c_q^{\dagger}$, so
    \begin{align*}
        \sum_{p=0}^{k-1} \hc_p G_{k,p}(\mix) + G_{k,k}(\mix) &= \sum_{p=0}^{k-1} \hc_p \left(- \sum_{q=0}^{k-1} G_{q,p}(\mix) c_q^{\dagger}\right) + \left(- \sum_{q=0}^{k-1} G_{q,k}(\mix) c_q^{\dagger} \right) \\
        &= -\sum_{q=0}^{k-1} c_q^{\dagger} \left(\sum_{p=0}^{k-1} G_{q,p}(\mix) \hc_p + G_{k,q}(\mix)\right) \\
        &= -\sum_{q=0}^{k-1} c_q^{\dagger} (A_{G(\mix)} \hc + b_{G(\mix)})_q.
    \end{align*}
    Therefore,
    \begin{align*}
         \sum_{i=1}^{k} w_i \prod_{j=1}^{k} |\beta_i - \widehat{\beta}_j|^2 &= \sum_{i=0}^{k-1} (\hc_i^{\dagger} - c_i^{\dagger}) (A_{G(\mix)} \hc + b_{G(\mix)})_i \\
         &\leq (\|\hc\|_1 + \|c\|_1) \|A_{G(\mix)} \hc + b_{G(\mix)}\|_{\infty} \\
         &\leq (2^k + 2^{O(k)}) \cdot 2^{O(k)}\cdot \noise  & (\text{Lemma~\ref{lm:dim2-step1}})\\
         &\leq 2^{O(k)} \cdot \noise.
    \end{align*}
\end{proof}

Similar to Lemma~\ref{lm:dim1-step2}, the following lemma shows the error still can be bounded after projection and injecting noise.

\begin{lemma} \label{lm:dim2-step2}
    Let $\tvecb = [\tb_1, \cdots, \tb_{k}]^{\top}\in \C^k$ be the intermediate result (Line~\ref{line:dim2-tb}) in Algorithm~\ref{alg:dim2}. Then,
    \begin{align*}
        \sum_{i=1}^{k} w_i \prod_{j=1}^{k} |\beta_i - \tb_j|^2 \leq 2^{O(k)} \cdot \noise.
    \end{align*}
\end{lemma}

\begin{proof}
    Let $\bvecb = [\bb_1, \cdots, \bb_{k}]^{\top} \in \R^k$ be the set of projections (Line~\ref{line:dim2-bb}). Since $\Delta_2$ is convex, $\Delta_{\C}$ is also convex. From $\beta_i \in \Delta_{\C}$, $|\beta_i - \bb_j| \leq |\beta_i - \hb_j|$ holds. Thus, 
    \begin{align*}
        \prod_{j=1}^{k} |\beta_i - \bb_j|^2 \leq \prod_{j=1}^{k} |\beta_i - \hb_j|^2.
    \end{align*}
    Recall that $\tb_j$ is obtained from $\bb_j$ by adding of size noise no more than $\noise$. We have $|\beta_i - \tb_j| \leq |\beta_i - \bb_j| + \noise$.
    Apply Lemma~\ref{lm:injected-noise} by regarding $|\beta_i - \tb_j|$ as $a_j$ and $|\beta_i - \hb_j|$ as $a_j'$. From $|\beta_i - \tb_j|\leq |\beta_i| + |\tb_j| \leq 2$, we can conclude that
    \begin{align*}
        \prod_{j=1}^{k} |\beta_i - \tb_j|^2 \leq \prod_{j=1}^{k} |\beta_i - \bb_j|^2 + O(2^k \cdot k\noise).
    \end{align*}
    Combining two inequalities, 
    \begin{align*}
       \sum_{i=1}^{k} w_i \prod_{j=1}^{k} |\beta_i - \tb_j|^2 &\leq \sum_{i=1}^{k} w_i \prod_{j=1}^{k} |\beta_i - \hb_j| + 2k \cdot \noise \leq 2^{O(k)} \cdot \noise.
    \end{align*}
\end{proof}

Similar to Lemma~\ref{lm:dim1-step4-pre}, we can bound the error of approximating ground truth over recovered spikes.

\begin{lemma}\label{lm:dim2-step4-pre}
    Let $\hvecw = [\hw_1, \cdots, \hw_{k}]^{\top} \in \R^k$ be the intermediate result (Line~\ref{line:dim2-hw}) in Algorithm~\ref{alg:dim2}. Then,
    $$\|V_{\tvecb}\hvecw - M(\pmix)\|_{\infty} \leq 2^{O(k)} \cdot \noise^{\frac{1}{2}}.$$
\end{lemma}

\begin{proof}
    Firstly, 
    \begin{align*}
        \|V_{\tvecb}\hvecw - M(\pmix)\|_2 &\leq \|V_{\tvecb}\hvecw - M_{G'}\|_2 + \|M(\pmix) - M_{G'}\|_2 \\
        &\leq \min_{\vecx \in \C^k} \|V_{\tvecb} \vecx - M_{G'}\|_2 + \|M(\pmix) - M_{G'}\|_2 \\ 
        &\leq \min_{\vecx \in \C^k} \|V_{\tvecb} \vecx - M(\pmix)\|_2 + \|M(\pmix) - M_{G'}\|_2
    \end{align*}
    where the first and third inequalities hold due to triangle inequality, and the second inequality holds since $\|V_{\tvecb} \hvecw - M_{G'}\|_2^2 \leq \min_{\vecx \in \R^k} \|V_{\tvecb} \vecx - M_{G'}\|_2^2$.
    
    For the first term, we can see that
    \begin{align*}
        \min_{\vecx \in \C^k} \|V_{\tvecb} \vecx - M(\pmix)\|_2 &= \min_{\vecx_1, \cdots, \vecx_k \in \C^k} \left\|\sum_{i=1}^{k}w_i V_{\tvecb}\vecx_i - \sum_{i=1}^{k}w_i M(\beta_i)\right\|_2 \\
        &\leq \sum_{i=1}^{k} w_i \min_{\vecx \in \C^k} \|V_{\tvecb}\vecx - M(\beta_i)\|_2 \\
        &\leq 2^{O(k)} \sum_{i=1}^{k} w_i \prod_{j=1}^{k} |\beta_i - \tb_j| & (\text{Lemma~\ref{lm:step3}}) \\
        &\leq 2^{O(k)} \sqrt{\sum_{i=1}^{k} w_i \prod_{j=1}^{k} |\beta_i - \tb_j|^2} & (\text{AM-QM Inequality}, \vecw \in \Delta_{k-1}) \\
        &\leq 2^{O(k)} \cdot \sqrt{\noise}. & (\text{Lemma~\ref{lm:dim2-step2}}).
    \end{align*}
    For the second term, 
    \begin{align*}
    \|M(\pmix) - M_{G'}\|_2 &\leq \sqrt{k} \|M(\pmix) - M_{G'}\|_{\infty} \\
        &\leq \sqrt{k} \|G' - G(\pmix)\|_{\infty} \\
        &\leq 2^{O(k)} \cdot \noise. & (\text{Lemma~\ref{lm:dim2-step1}})
    \end{align*}
    As a result, $$\|V_{\tvecb}\hvecw - M(\pmix)\|_{\infty} \leq \|V_{\tvecb}\hvecw - M(\pmix)\|_2 \leq 2^{O(k)} \cdot \sqrt{\noise} + 2^{O(k)} \cdot \noise \leq 2^{O(k)} \cdot \noise^{\frac{1}{2}}.$$
\end{proof}

Similar to Lemma~\ref{lm:dim1-step4}, we can prove that $\tpmix$ is also a good estimation of $\pmix$.

\begin{lemma}\label{lm:dim2-step4} 
    Let $\tpmix = (\tvecb, \tvecw)$ be the intermediate result (Line~\ref{line:dim2-tpmix}) in Algorithm~\ref{alg:dim2}. Then,
    $$\Mdis_K(\tpmix, \pmix) \leq 2^{O(k)} \cdot \noise^{\frac{1}{2}}.$$
\end{lemma}

\begin{proof} 
    According to the definition of $V_{\tvecb}$ and $M(\pmix)$, we have $(V_{\tvecb}\hvecw - M(\pmix))_1 = \sum_{i=1}^{k} \hw_i - 1$.
    Therefore, $|\sum_{i=1}^{k} \hw_i - 1| \leq |(V_{\tvecb}\hvecw - M(\pmix))_1| \leq \|V_{\tvecb}\hvecw - M(\pmix)\|_{\infty} \leq 2^{O(k)} \cdot \noise^{\frac{1}{2}}.$ 
    
    Note that $\|\tvecw - \hvecw\|_{1} = \|\tvecw\|_1 - \|\hvecw\|_{1} = |(\sum_{i=1}^{k} \hw_i)^{-1} - 1|$. For $2^{O(k)} \cdot \noise^{\frac{1}{2}} \leq 1/2$, we can conclude that $\|\tvecw - \hvecw\|_{1} = |(\sum_{i=1}^{k} \hw_i)^{-1} - 1| \leq 2 |\sum_{i=1}^{k} \hw_i - 1| \leq  2^{O(k)} \cdot \noise^{\frac{1}{2}}.$
    
    Thus,
    \begin{align*}
        \Mdis_K(\tpmix, \pmix)&=\|M(\tpmix) - M(\pmix)\|_\infty = \|V_{\tvecb} \tvecw - M(\pmix)\|_{\infty} \\
        &\leq  \|V_{\tvecb} \hvecw - M(\pmix)\|_{\infty} + \|V_{\tvecb}\|_\infty \|\tvecw - \hvecw\|_1 \\
        &\leq 2^{O(k)} \cdot \noise^{\frac{1}{2}}.
    \end{align*}
    where the first inequality holds because of the triangle inequality.
\end{proof}

Now everything is in place to show the final bound.

\begin{lemma}\label{lm:dim2-reconstruction} 
Let $\cmix = (\tveca, \cvecw)$ be the final result 
(Line~\ref{line:dim2-cmix}) in Algorithm~\ref{alg:dim2}. 
Then, we have 
$$\tran(\cmix, \mix) \leq O(k\noise^{\frac{1}{4k-2}}).$$
\end{lemma}

\begin{proof}
    Apply Lemma~\ref{lm:dim2-step4} to the result of Lemma~\ref{thm:complex-moment-inequality}. We have
    $$\tran(\tpmix, \pmix) \leq O(k \cdot (2^{O(k)} \cdot \noise^{\frac{1}{2}})^{\frac{1}{2k-1}}) \leq O(k \noise^{\frac{1}{4k-2}}).$$
    
    Denote $\tpmix' = (\tvecb, \textrm{real}(\tvecw))$ as the result that discards all imaginary components from $\tpmix$. From the definition of transportation distance for complex weights, the imaginary components is independent from the real components. Moreover, $\pmix$ has no imaginary components. Thus, discarding all imaginary components from $\tpmix$ reduces the transportation distance to $\pmix$. That is, 
    \begin{align*}
        \tran(\tpmix', \pmix) \leq \tran(\tpmix, \pmix).
    \end{align*}
    Moreover, from any $\beta_i, \beta_j \in \C$, we have 
    \begin{align*}
        \quad \|[\textrm{real}(\beta_i), \textrm{imag}(\beta_i)] - [\textrm{real}(\beta_j), \textrm{imag}(\beta_j)]\|_1 
        &\leq 2 \|[\textrm{real}(\beta_i), \textrm{imag}(\beta_i)] - [\textrm{real}(\beta_j), \textrm{imag}(\beta_j)]\|_2 \\
        &= 2|\beta_i - \beta_j|.
    \end{align*}
    This shows the distance over $\C$ is smaller than $2$ times of the distance over $\R^2$. As a result, 
    \begin{align*}
        \tran(\tmix, \mix) \leq 2\tran(\tmix, \mix) \leq O(k \noise^{\frac{1}{4k-2}}).
    \end{align*}
    Moreover, from Lemma~\ref{lm:step5}, we have 
    \begin{align*}
        \tran(\tmix, \cmix) \leq 2 \min_{\bmix \in \Spike(\Delta_{2}, \Delta_{k-1})} \tran(\tmix, \bmix) \leq 2 \tran(\tmix, \mix) = O(k \noise^{\frac{1}{4k-2}}).
    \end{align*}
    Finally, from triangle inequality, $\tran(\cmix, \mix) \leq \tran(\cmix, \tmix) + \tran(\tmix, \mix) \leq O(k \noise^{\frac{1}{4k-2}}).$
\end{proof}

\section{An Algorithm for Higher-dimensional Problem}
\label{sec:high-dim-recover-full}

We assume the moment information is revealed by a noisy moment oracle. A moment oracle corresponding to $k$-spike distribution $\mix$ is defined to be a function $M'(\cdot)$ that takes a matrix (or a vector) $R$ with $\|R\|_{\infty}\leq 1$ and generates a noisy vector $M'(R)$ such that $\|M'(R) - M(\proj{R}(\mix))\|_{\infty} \leq \noise$ and $M'_{\mathbf{0}}(R) = M_{\mathbf{0}}(\proj{R}(\mix)) = 1$ where $\noise \leq 2^{-\Omega(k)}$ is the noise level and $M$ is the moment vector defined in \eqref{eq:highdimmomvector}.

Let $e_t\in \R^d$ be the unit vector that takes value $1$ in only the $t$-th coordinate and $0$ in the remaining coordinates. Denote $\mathbf{1_d} = (1, \cdots, 1) \in \R^d$. Let $\mathbb{S}_{d-1}$ be the unit sphere in $\mathbb{R}^d$.

\begin{algorithm}[!ht]
\caption{Reconstruction Algorithm in High Dimension} \label{alg:dim3}
\begin{algorithmic}[1]
    \STATE \textbf{input:} number of spikes $k$,  noisy moments $M'(\cdot)$, noise level $\noise$ \\
    \textbf{output:} recovered spike distribution $\cmix$
    \STATE $\vecr \leftarrow \mathbb{S}_{d-1}$
    \label{line:dim3-r}
    \STATE $\tpmix'\leftarrow \textrm{OneDimension}(k, M'(\frac{\vecr+\mathbf{1}}{4}), \noise)$ 
    \label{line:dim3-tpmix'}
    \hfill{$\Rightarrow\tpmix' = (\tvecy', \tvecw') \in \Spike(\simplex_1, \simplex_{k-1})$}
    \STATE $\tpmix \leftarrow([4\tvecy'-\mathbf{1}_d], \tvecw')$
    \label{line:dim3-tpmix}
    \hfill{$\Rightarrow\tpmix = (\tvecy, \tvecw)$}
    \FOR {$t\in[1, d]$}
    \STATE $\hpmix_t' \leftarrow \textrm{TwoDimension}(k, M'([\frac{\vecr+\mathbf{1}}{4}, \frac{e_t}{2}]), \noise)$ 
    \label{line:dim3-hpmix'}
    \hfill{$\Rightarrow\hpmix_t' = ([\hvecy_{:,t}', \hveca_{:,t}'], \hvecw_{:,t}')\in \Spike(\simplex_2, \simplex_{k-1})$}
    \STATE $\hpmix_t \leftarrow ([4\vecy_{:,t}'-\mathbf{1}_d, 2\hveca_{:,t}'], \hvecw_{:,t}')$
    \label{line:dim3-hpmix}
    \hfill{$\Rightarrow\hpmix_t = ([\hvecy_{:,t}, \hveca_{:,t}], \hvecw_{:,t})$}
    \label{line:dim3-ta-for-j}
    \FOR {$j \in [1,k]$}
    \STATE $s^*_{j,t} \leftarrow \arg \min_{s : \hw_{s,t} > \sqrt{\epsilon}} |\ty_j - \hy_{s,t}|$
    \label{line:dim3-s}
    \STATE $\ta_{j,t} \leftarrow \ha_{s_{j,t}^*,t}$
    \label{line:dim3-ta-jt}
    \ENDFOR
    \STATE $\tveca_{:,t} \leftarrow [\ta_{1,t}, \cdots, \ta_{k,t}]^{\top}$ 
    \label{line:dim3-ta-.t}
    \hfill {$\Rightarrow\tveca \in \R^{k}$}
    \ENDFOR
    \label{line:dim3-ta-endfor-j}
    \STATE $\tveca \leftarrow [\tveca_{:,1}, \cdots, \tveca_{:,d}]^{\top}$
    \label{line:dim3-ta-..}
    \hfill{$\Rightarrow\tveca = [\tveca_{1}, \cdots, \tveca_{k}]\in \mathbb{R}^{d\times k}$}
    \STATE $\tmix \leftarrow (\tveca, \tvecw)$ 
    \label{line:dim3-tmix}
    \hfill{$\Rightarrow\tmix \in \Spike(\R^d, \simplex_{k-1})$}
    \STATE $\cveca \leftarrow \textrm{project}_{\simplex_{d-1}}(\tveca)$
    \label{line:dim3-ca}
    \hfill{$\Rightarrow\cveca = [\cveca_{1}, \cdots, \cveca_{k}]\in \Delta_{d-1}^{k}$}
    \STATE $\cmix \leftarrow (\cveca, \tvecw)$
    \label{line:dim3-cmix}
    \hfill $\Rightarrow\cmix \in \Spike(\simplex_{d-1}, \simplex_{k-1})$
\end{algorithmic}
\end{algorithm}

The pseudocode can be found in Algorithm~\ref{alg:dim3}. The algorithm's input consists of the number of spikes $k$, the moment error bound $\noise$, and the noisy moment oracle $M'(\cdot)$ mentioned above. Now, we describe the implementation details of our algorithm.

We first generate a random vector $\vecr$ in Line~\ref{line:dim3-r} by sampling from unit sphere $\mathbb{S}_{d-1} = \{\vecr \in \R^d \mid \|\vecr\|_2 = 1\}$. Note that by a probabilistic argument, the distance between spikes is roughly kept after the projection along $\vecr$ (see Lemma~\ref{lm:prob-proj}). Then, we aim to use the one-dimension algorithm to recover $\proj{\vecr}(\mix) $. 
However, the support of $\proj{\vecr}(\mix)$ is contained in $[-1,1]$ but not $[0, 1]$. 
In Line~\ref{line:dim3-tpmix'}, we apply the 1-dimension algorithm (Algorithm~\ref{alg:dim1}) 
to the noisy moments of a shifted and scaled map $\proj{(\vecr+\mathbf{1}_d)/4}(\mix)$ whose support is in $[0,1]$. Then we scale and shift back to obtain the result $\tpmix$ in Line~\ref{line:dim3-tpmix}. Intuitively, $\tpmix$ is close to $\proj{\vecr}(\mix)$.

%\zynote{2d-Projection above}

Now, we try to recover the coordinates of the spikes.
%in $\Spike(\simplex_{d-1}, \simplex_{k-1})$.
In particular, we would like to find, for each spike of $\tpmix$, the coordinates of the corresponding location in $\simplex_{d-1}$. This is done by considering each dimension separately. For each coordinate $t\in [d]$, we run the two-dimension algorithm (i.e., Algorithm~\ref{alg:dim2}) to recovery the projection
$\proj{[\vecr, e_t]}(\mix)$, that is a linear map of $\mix$ onto the subspace spanned by vector $\vecr$ and axis $e_t$. 
However, $[\vecr, e_t]\veca_i$ is in $[-1,1]\times[0,1]$. Again, in Line~\ref{line:dim3-hpmix'}, we apply the algorithm to the noisy moments of a shifted and scaled projection $\proj{B_t}(\mix)$ where $B_t=[\frac{\vecr+\mathbf{1}_d}{4}, \frac{e_t}{2}]$, and scale the result back in Line~\ref{line:dim3-hpmix}. 
The result is denoted as $\hpmix_{t}$. 

Next, we assemble the 1d projection $\tpmix$ and 2-d projections $\hpmix_{t}, t\in [d]$.
Due to the noise, $\tpmix_{t}$ and $\tpmix$ may have different supports along $\vecr$. 
So we assign the coordinates of $\tpmix$ according to the closest spike of $\hpmix_{t}$ with relatively large weights in Line~\ref{line:dim3-ta-for-j} to Line~\ref{line:dim3-ta-.t}.
Due to the error, the support of the reconstructed distribution $\tmix$ may not locate in $\simplex_{d-1}$. 
Thus the final step in Line~\ref{line:dim3-ca} is to project each individual spike back to $\simplex_{d-1}$.

We note the bottleneck of this algorithm is Line~\ref{line:dim3-hpmix'}, where we have to solve the two-dimensional problem $d$ times. Hence, the total running time of the algorithm is $O(d k^3)$.

\subsection{Error Analysis}

%In our algorithm, we project the high dimension mixture to 1-dimension and then reconstruct the mixture using the projection.
First, we show the random vector $\vecr$ has a critical property that two distant  spikes are still distant after the projection with high probability.
This result is standard and similar statements are known in the literature (e.g., Lemma 37 in \citet{wu2020optimal}).

\begin{lemma}\label{lm:prob-proj}
Let $\supp=(\veca_1, \cdots, \veca_k)\subset\simplex_{d-1}$ such that $d > 3$.
For every constant $0 < \eta < 1$ and a random vector $\vecr \leftarrow \mathbb{S}_{d-1}$, with probability at least $1-\eta$, for any $\veca_i,\veca_j\in \supp$, we have that 
$$
|\vecr^{\top} (\veca_i - \veca_j)| \geq \frac{\eta}{k^2d}\|\veca_i-\veca_j\|_1.
$$
\end{lemma}

\begin{proof}
Let $0 < \eps < 1$ be some constant depending on $k$ and $d$ which we fix later. Assume $\vecr = [r_1, \cdots, r_d]^{\top}$. For a fixed pair $(\veca_i, \veca_j)$, we have $$\Pr_{\vecr \sim \mathbb{S}_{d-1}} [|\vecr^{\top}(\veca_i - \veca_j)| < \eps\|\veca_i - \veca_j\|_2] = \Pr_{\vecr \sim \mathbb{S}_{d-1}}[|r_1| < \epsilon].$$

Suppose $\vecr$ is obtained by first sampling a vector from
unit ball $\mathbb{B}_d$ and then 
normalizing it to the unit sphere $\mathbb{S}_{d-1}$. 
We can see that
\begin{align*}
\Pr_{\vecr \sim \mathbb{S}_{d-1}}[|r_1| < \epsilon] \leq \Pr_{\vecr \sim \mathbb{B}_{d}}[|r_1| < \epsilon]
\end{align*}
where the RHS is equal to the probability that $\vecr$ lies in 
the slice $[-\epsilon, \epsilon]$.

Let $V_d$ be the volume of $d$-dimensional ball.
We have 
\begin{align*} 
\Pr_{\vecr \sim \mathbb{B}_{d}}[|r_1| < \epsilon] &= \frac{\int_{-\epsilon}^{\epsilon} V_{d-1} (1-x^2)^{\frac{d-1}{2}}\d x}{V_d} 
    \leq \epsilon\sqrt{d}
\end{align*}
where the last inequality holds since $V_d / V_{d-1} \geq 2/\sqrt{d}$.

By union bound over all pairs $(\veca_i, \veca_j)$, we can see that
the failure probability can be bounded by 
$$
\Pr_{\vecr \sim \mathbb{S}_{d-1}} [\exists \veca_i, \veca_j \in \supp : |\vecr^{\top}(A_i - A_j)| < \eps\|A_i - A_j\|_2] < k^2 \eps\sqrt{d}.
$$
Notice that $\|\veca\|_1 \leq \sqrt{d}\|\veca\|_2$ for any $\veca \in \R^d$, and
%by plugging in, 
%$$\Pr_{\vecr \sim \mathbb{S}_{d-1}} [\exists \veca_i, \veca_j \in \supp : |\vecr^{\top}(A_i - A_j)| < %\frac{\eps}{\sqrt{d}}\|A_i - A_j\|_1] < k^2\eps\sqrt{d}$$
let $\eps = \frac{\eta}{k^2\sqrt{d}}$. We conclude that 
$$\Pr_{\vecr \sim \mathbb{S}_{d-1}} [\forall \veca_i, \veca_j \in \supp : |\vecr^{\top}(A_i - A_j)| \geq \frac{\eta}{k^2d}\|A_i - A_j\|_1] > 1-\eta.$$
\end{proof}

Next, we show that we can recover and reconstruct the higher dimension mixture $\mix$, using its 1d projection $\tpmix \sim \proj{\vecr}(\mix)$, by assigning each spike $(\ty_i, \tw_i)$ in the one dimension distribution a location in $\mathbb{R}^d$. 
In a high level, 
consider a clustering of the spikes in $\proj{\vecr}(\mix)$. The spikes in one cluster still form a cluster in the original $\mix$ by Lemma~\ref{lm:prob-proj}
since the projection approximately keeps the distance. 
%Since $\tpmix$ is an approximation for $\proj{\vecr}(\mix)$,
%one may expect this property holds for large clusters. 
Thus assigning the clusters a location in $\mathbb{R}^d$ can produce a good estimation for $\mix$.
%In fact, we note that this clustering implicitly creates a separation condition for small weights.
%However, creating separation here only loses an $O(1)$ factor in exponential, which is much better than an $O(k)$ factor in previous works.
The following lemma formally proves that we can reconstruct the original mixture using two-dimensional projections.

% Let $\widetilde{\mu} \in M(\tvecw, \vecw)$ be the optimal matching distribution corresponding to $\tran(\tpmix, \proj{\vecr}(\mix))$. 

% According to the definition of transportation distance, it satisfies
% \begin{align*}
%     \tran(\tpmix, \proj{\vecr}(\mix)) = \sum_{i=1}^{k}\sum_{j=1}^{k} \widetilde{\mu}_{i,j} |\widetilde{y}_i - \vecr^{\top} \veca_j|, \tw_i = \sum_{j=1}^{k} \widetilde{\mu}_{i,j}, w_j = \sum_{i=1}^{k} \widetilde{\mu}_{i,j}, \widetilde{\mu} \geq \mathbf{0}
% \end{align*}
% Since $\tpmix$ and $\tmix$ share the same weight, $\proj{\vecr}(\mix)$ and $\mix$ share the same weight distribution, $\widetilde{\mu}$ is also a valid matching corresponding to $\tran(\tmix, \mix)$. 
% Therefore, $\tran(\tmix, \mix) \leq \sum_{i, j} \widetilde{\mu}_{i,j} \|\tveca_i - \veca_j\|_1$. Our plan is to show it can be bounded by $(k\eps)^{O(1)}$.

\begin{lemma}\label{lm:dimn-step3}
Let $\tmix \in \Spike(\R^d, \simplex_{k-1})$ be the intermediate result (Line~\ref{line:dim3-tmix}) in Algorithm~\ref{alg:dim3}. 
Let $\eps$ be the smallest number such that $\tran(\tpmix, \proj{\vecr}(\mix)) \leq \eps$ and $\tran(\hpmix_t, \proj{[\vecr, e_t]}(\mix)) \leq \eps$ for all $t$.
If Lemma~\ref{lm:prob-proj} holds for $\mix$, 
we have that 
$$\tran(\tmix, \mix) \leq O(k^3d^2\sqrt{\eps}).$$
\end{lemma}

\begin{proof}
Let $C_1, \cdots, C_m$ be a partition of $[k]$ for $\mix$ such that:
\begin{itemize}
    \item $|\vecr^{\top} \veca_i - \vecr^{\top} \veca_j| \leq k\sqrt{\eps}$ if $i \in C_l, j \in C_l$ for some $l \in [m]$.
    \item $|\vecr^{\top} \veca_i - \vecr^{\top} \veca_j| > \sqrt{\eps}$ if $i \in C_l, j \in C_{l'}$ for $l \neq l'$
\end{itemize}

For every cluster $C_p$, we construct another set $\widetilde{C}_p$ for $\tmix$ containing every index $j \in [k]$ such that $\min_{i\in C_{p}} |\ty_j - \vecr^{\top}\veca_i| < \frac{\sqrt{\eps}}{2}$. From the construction of the partition, for each spike $\ty_j$, there exists at most one $C_{p'}$ such that $\min_{i\in C_{p'}} |\ty_j - \vecr^{\top}\veca_i| < \frac{\sqrt{\eps}}{2}$. Thus, we have all $\widetilde{C}_1, \cdots, \widetilde{C}_m$ disjoint.

In this case, for every pair of $i \notin C_p$ and $j \in \widetilde{C}_p$, we have $|\ty_j - \vecr^{\top}\veca_i| > \frac{\sqrt{\epsilon}}{2}$. Since the total transportation distance between $\tpmix$ and $\proj{\vecr}(\mix)$ is no more than $\epsilon$. Hence, a portion of at most $\frac{\eps}{\sqrt{\eps}/2}$ of the weight in $C_p$ can be matched out of $\widetilde{C}_p$. This gives
$$\big|\sum_{i \in C_p} w_i - \sum_{j \in \widetilde{C}_p} \tw_j\big| < 2\sqrt{\epsilon}.$$

Recall the transportation distance $\tran(\mix, \tmix)$ is the solution of the following program over coupling distribution $\mu$.
\begin{align*}
    \textrm{ min } &  \sum_{i=1}^{k}\sum_{j=1}^{k} \mu_{i,j} \|\tveca_j - \veca_i\|_1 \\
    \textrm{ s.t. } & w_i = \sum_{j=1}^{k} \mu_{i,j}, \tw_j = \sum_{i=1}^{k} \mu_{i,j}, \mu \geq \mathbf{0}.
\end{align*}
Consider a specific distribution $\mu$ generated by the following two-stage procedure: 
Initially, set $\mu_{i,j} = 0$ for every pair of $(i,j)$
In the first stage, increase $\mu_{i,j}$ for every $i \in C_p$ and $j \in \widetilde{C}_p$ whenever $w_i \geq \sum_{j=1}^{k} \mu_{i,j}, \tw_j \geq \sum_{i=1}^{k} \mu_{i,j}$.
In the second stage, increase $\mu_{i,j}$ arbitrarily for every pair of $(i, j)$ to satisfy every constraint.
Since the restriction set is convex, $C_p$ are disjoint, and $\widetilde{C}_p$ are disjoint, for some specific $p$, this two-stage process ensures 
\begin{align*}
    \sum_{i \in C_p} \sum_{j \in \widetilde{C}_p} \mu_{i,j} = \min \{ \sum_{i \in C_p} w_i, \sum_{j \in \widetilde{C}_p} \tw_j\}.
\end{align*}
As a result, the coupling distribution satisfies
\begin{align*}
    \sum_{i \in C_p} \sum_{j \notin \widetilde{C}_p} \mu_{i,j} \leq |\sum_{i \in C_p} w_i - \sum_{j \in \widetilde{C}_p} \tw_j| \leq 2\sqrt{\epsilon}.
\end{align*}

Now we start to show $\mu$ is a good coupling distribution for $\tran(\mix, \tmix)$.
Since $\{C_p\}$ is a partition of $[k]$, we can decompose the transportation distance into: 
\begin{align*}
    \tran(\mix, \tmix) &\leq \sum_{i=1}^{k}\sum_{j=1}^{k} \mu_{i,j} \|\tveca_j - \veca_i\|_1 = \sum_{p} \sum_{i \in C_p} \sum_{j=1}^{k} \mu_{i,j} \|\tveca_j - \veca_i\|_1.
\end{align*}

Next, we bound the RHS by considering each cluster $C_p$. 
If it holds $\sum_{i\in C_p} \sum_{j=1}^{k} \mu_{i,j} \leq 2k\sqrt{\epsilon}$, from $\|\tveca_j - \veca_i\|_1 \leq 1$, we can directly get $$\sum_{i \in C_p} \sum_{j=1}^{k} \mu_{i,j} \|\tveca_j - \veca_i\|_1 \leq 2k\sqrt{\epsilon}.$$ 
Otherwise, we have $\sum_{i\in C_p} \sum_{j=1}^{k} \mu_{i,j} > 2k\sqrt{\epsilon}$. We can decompose the last summation into:
\begin{align*}
    \sum_{i \in C_p} \sum_{j=1}^{k} \mu_{i,j} \|\tveca_j - \veca_i\|_1 &= \sum_{i \in C_p} \sum_{j \notin \widetilde{C}_p} \mu_{i,j} \|\tveca_j - \veca_i\|_1  + \sum_{i \in C_p} \sum_{j \in \widetilde{C}_p} \mu_{i,j} \|\tveca_j - \veca_i\|_1.
\end{align*}

For the first term, we have $\sum_{i \in C_p} \sum_{j \notin \widetilde{C}_p} \mu_{i,j} \leq 2\sqrt{\epsilon}$ as we had discussed above. Together with $\|\tveca_j - \veca_i\|_1 \leq 1$, we can bound the first term by $2\sqrt{\epsilon}$. For the second term, we have $\sum_{i \in C_p} \sum_{j \in \widetilde{C}_p} \mu_{i,j} \leq 1$ according to the property of coupling distribution and the following lemma shows $\|\tveca_j - \veca_i\|_1$ is small. We delay the proof for readability.

\begin{lemma}\label{lm:dimn-step3.5}
If $\sum_{i\in C_p} w_i > 2k\sqrt{\epsilon}$, 
then $\|\tveca_j - \veca_i\|_1 \leq O(k^3d^2\sqrt{\epsilon})$ for all $i \in C_p$ and $j \in \widetilde{C}_p$.
\end{lemma}

\noindent
As a result, we can see that
\begin{align*}
    \sum_{i \in C_p} \sum_{j=1}^{k} \mu_{i,j} \|\tveca_j - \veca_i\|_1 \leq O(k^3d^2 \sqrt{\epsilon}).
\end{align*}
Hence, we can conclude that 
$
\tran(\mix, \tmix) \leq O(k^3d^2 \sqrt{\epsilon}).
$
\end{proof}

\begin{proofof} {Lemma~\ref{lm:dimn-step3.5}}
For some fixed cluster $C_p$ with $i \in C_p$ and $j \in \widetilde{C}_p$, we will prove the statement by showing $\veca_i$ and $\tveca_j$ are close on every dimension $t$.

Firstly, we claim that for the cluster $C_p$, there exist $s \in [k]$ and $i_1 \in C_p$ in which $\hw_{s,t} \geq \sqrt{\epsilon}$ and $|\hy_{s,t} - \vecr^{\top} \veca_{i_1}| + |\ha_{s,t} - \alpha_{i_1,t}| \leq \sqrt{\epsilon}$ in which $\hpmix_{t}$ contains a spike at $(\hy_{s,t}, \ha_{s,t})$ of weight $\hw_{s,t}$ and $\alpha_{i_1,t}$ is the dimension $t$ of $\alpha_{i_1}$. Towards a contradiction, each spike in $\hpmix_{t}$ is either of weight less than $\sqrt{\epsilon}$ or is of distance at least $\sqrt{\epsilon}$ from the projected spikes in $C_p$ along $[\vecr, e_t]$. Then, each unit of the weights in $C_p$ suffers a transportation cost of at least $\sqrt{\epsilon}$ after the first $k \times \sqrt{\epsilon}$ weights since it is not sufficient to cover spikes in $C_p$ using nearby spikes. Note the total weight in $C_p$ is at least $2k\sqrt{\epsilon}$. So in this case, the transportation distance between $\hpmix$ and $\proj{\vecr}(\mix)$ is at least $\tran(\hpmix_t, \proj{[\vecr, e_t]}(\mix)) \geq (2k\sqrt{\epsilon} - \sqrt{\epsilon} \times k) \times \sqrt{\epsilon} > k\epsilon$ which contradicts to the definition of $\epsilon$.

Next, we take some $i_2 \in C_p$ for $j$ such that $|\ty_j - \vecr^{\top} \veca_{i_2}| \leq \frac{\sqrt{\epsilon}}{2}$. This always exists according to the definition of $\widetilde{C}_p$. From the clustering, it holds that $|\vecr^{\top} \veca_{i} - \vecr^{\top} \veca_{i_2}| \leq k\sqrt{\epsilon}$ and $|\vecr^{\top} \veca_{i_1} - \vecr^{\top} \veca_{i_2}| \leq k\sqrt{\epsilon}$. With the triangle inequality, we have 
$$|\ty_j - \hy_{s,t}| \leq |\hy_{s,t} - \vecr^{\top} \veca_{i_1}| + |\vecr^{\top} \veca_{i_1} - \vecr^{\top} \veca_{i_2}| + |\ty_j - \vecr^{\top} \veca_{i_2}| \leq (1.5 + k)\sqrt{\epsilon}.$$ 
Since $\hw_{s,t} \geq \sqrt{\epsilon}$ which fits the condition in Line~\ref{line:dim3-s} of Algorithm~\ref{alg:dim3}, the minimization thus ensures $|\ty_{j} - \hy_{s^*_{j,t},t}| \leq (1.5 + k)\sqrt{\epsilon}.$

In addition, since the weight of spike $(\hy_{s^*_{j,t},t}, \ha_{s^*_{j,t}, t})$ satisfies $\hw_{s^*_{j,t},t} > \sqrt{\eps}$ and the total transportation distance satisfies $\tran(\hpmix_t, \proj{[\vecr, e_t]}(\mix)) \leq \eps$. There must exist some $i_3 \in [k]$ such that $|\vecr^{\top} \veca_{i_3} - \hy_{s^*_{j,t},t}| + |\alpha_{i_3, t} - \ha_{s^*_{j,t}, t}| < \frac{\eps}{\sqrt{\eps}} = \sqrt{\eps}$, that is, $(\hy_{s^*_{j,t},t}, \ha_{s^*_{j,t}, t})$ is close to spike $(\vecr^{\top} \veca_{i_3}, \alpha_{i_3, t})$.

From the triangle inequality, we have 
\begin{align*}
    &\quad\ |\vecr^{\top} \veca_{i} - \vecr^{\top} \veca_{i_3}| + |\alpha_{i_3, t} - \ha_{s^*_{j,t}, t}| \\
    &\leq |\vecr^{\top} \veca_i - \vecr^{\top} \veca_{i_2}| + |\ty_j - \vecr^{\top} \veca_{i_2}| + |\ty_{j} - \hy_{s^*_{j,t},t}| + |\vecr^{\top} \veca_{i_3} - \hy_{s^*_{j,t},t}| + |\alpha_{i_3, t} - \ha_{s^*_{j,t}, t}| \\
    &\leq (3 + 2k)\sqrt{\epsilon}.
\end{align*}

Applying Lemma~\ref{lm:prob-proj} for $\veca_i$s, we have that $|\vecr^{\top} \veca_{i} - \vecr^{\top} \veca_{i_3}| \geq \frac{\eta}{k^2 d} \|\veca_{i} - \veca_{i_3} \|_1 \geq \frac{\eta}{k^2 d} |\alpha_{i,t} - \alpha_{i_3,t}|$ under the good event. Thus, 
\begin{align*}
    |\ta_{j,t} - \alpha_{i,t}| &= |\ha_{s^*_{j,t}, t} - \alpha_{i,t}| \\
    &\leq |\alpha_{i,t} - \alpha_{i_3, t}| + |\alpha_{i_3} - \ha_{s^*_{j,t},t}| \\
    &\leq \frac{k^2d}{\eta} (|\vecr^{\top} \veca_{i} - \vecr^{\top} \veca_{i_3}| + |\alpha_{i_3, t} - \ha_{s^*_{j,t}, t}|) \\
    &\leq \frac{k^2d}{\eta} \cdot (3 + 2k)\sqrt{\epsilon} \leq O(k^3d \sqrt{\epsilon})
\end{align*}
where the equality holds due to the assignment of $\ta_{j,t} \leftarrow \ha_{s_{j,t}^*,t}$ in Line~\ref{line:dim3-ta-jt} of Algorithm~\ref{alg:dim3} and the second inequality holds since $\eta < 1$.
Taking summation over $t$, we conclude that 
\begin{align*}
    \|\tveca_j - \veca_i\|_1 \leq \sum_{t=1}^{d} |\ta_{j,t} - \alpha_{i,t}| \leq O(k^3d^2 \sqrt{\epsilon}).
\end{align*}
\end{proofof}

Note that we can bound $\eps$ using Lemma~\ref{lm:dim1-reconstruction} and Lemma~\ref{lm:dim2-reconstruction}. 
As a result, we can provide a performance guarantee of our algorithm for 
the high-dimensional case.

\begin{theorem}\label{thm:dimn-reconstruction} 
Let $\cmix = (\cveca, \tvecw)$ be the final result 
(Line~\ref{line:dim3-cmix}) in Algorithm~\ref{alg:dim3}. 
Then with probability at least $1-\eta$ for any fixed constant $0 < \eta < 1$,
$$\tran(\cmix, \mix) \leq (kd \noise^{\frac{1}{k}})^{O(1)}.$$
\end{theorem}

\begin{proof}
    Since $\tpmix'$ is generated by one dimension algorithm, according to Lemma~\ref{lm:dim1-reconstruction}, we have $$\tran(\tpmix', \proj{\frac{\vecr+\mathbf{1}}{4}}(\mix)) \leq O(k \noise^{\frac{1}{4k-2}}).$$ 
    Note that
    $\proj{\frac{\vecr+\mathbf{1}}{4}}(\mix)$ can be transformed to $\proj{\vecr}(\mix)$ by transforming the coordinates by $\alpha \rightarrow 4\alpha - 1$. We have 
    $$
    \tran(\tmix, \proj{\vecr}(\mix)) \leq 4\tran(\tpmix', \proj{\frac{\vecr+\mathbf{1}}{4}}(\mix)) \leq O(k \noise^{\frac{1}{4k-2}}).
    $$
    Similarly, according to Lemma~\ref{lm:dim2-reconstruction}, we can see 
    $
    \tran(\hpmix_t, \proj{[\vecr, e_t]}(\mix)) \leq O(k \noise^{\frac{1}{4k-2}}).
    $
    Thus, from the definition of $\eps$, we have $\eps \leq O(k \noise^{\frac{1}{4k-2}})$. With Lemma~\ref{lm:dimn-step3}, we have $$\tran(\tmix, \mix) \leq O(k^3d^2\sqrt{\eps}) \leq (kd \noise^{\frac{1}{k}})^{O(1)}.$$
    Note that $\simplex_{d-1}$ is a convex set and $\cveca_j$ is the projection of $\tveca_j$ to the same set. Thus, $\|\veca_i - \cveca_j\|_1 \leq \|\veca_i - \tveca_j\|_1$ holds for all $A_i \in \simplex_{d-1}$ and $j$. Hence, the projection to $\simplex_{d-1}$ does not increase the transportation distance. As a result, 
    $
        \tran(\tmix, \cmix) \leq  \tran(\tmix, \mix).
    $
    Finally, by triangle inequality, we have
    $$\tran(\cmix, \mix) \leq \tran(\tmix, \mix) + \tran(\cmix, \tmix) \leq (kd \noise^{\frac{1}{k}})^{O(1)}.$$
\end{proof}

By choosing $\noise = (\eps/(dk))^{O(k)}$, the previous theorem directly implies Theorem~\ref{thm:highdim-kspike}.

\section{Applications to Topic Models}
\label{sec:topic-full}

In this section, we discuss the application of our results to topic models.
Here, the underlying vocabulary consists of $d$ words.
We are given a corpus of documents. We adopt the popular 
“bag of words” model and take each document as an unordered multiset of words. 
The assumption is that there is a small number $k$
of “pure” topics, where each topic is a distribution over $[d]$. 
A $K$-word document (i.e., a string in $[d]^K$) is generated by first selecting a
topic $i\in\simplex_{d-1}$ from the mixture $\mix$, and then sampling
$K$ words i.i.d. according to $\veca_{i}$ from this topic.
Here, $K$ is referred to as {\em the snapshot number} and we call such a sample a {\em $K$-snapshot} of $p$.
Our goal here is to recover $\mix\in \Spike(\simplex_{d-1},\simplex_{k-1})$,
which is a discrete distribution over $k$ pure topics.
Again, we measure the accuracy in terms of $L_1$-transportation distance.

\subsection{An Algorithm for $d=O(k)$}
\label{subsec:lowdimtopic}

In this section, we first directly handle the high-dimensional case using Algorithm~\ref{alg:dim3}.
We will perform dimension reduction when $d\gg k$, in Section~\ref{subsec:larged}.

\begin{lemma}
\label{lm:lowdimtopic}
There is an algorithm that can learn an arbitrary $k$-spike mixture 
of discrete distributions supported in $\simplex_{d-1}$
for any $d$ within $L_1$ transportation distance $\epsilon$ with probability at least $0.99$ using $(kd/\epsilon)^{O(k)}$ many $(2k-1)$-snapshots.
\end{lemma}

The following lemma is a high-dimensional version for empirical error (see e.g. \citet{rabani2014learning,li2015learning,gordon2020sparse}).

\begin{lemma}
\label{lm:dim-sampling}
For matrix $R = [\vecr_{1}, \cdots, \vecr_{p}]\in \mathbb{R}^{d \times p}$ with $\|R\|_{\infty}\leq 1$,
using $2^{O(kp)}\cdot\noise^{-2}\cdot\log(1/\eta)$ samples, we can obtain transformed empirical moments $M'(R)$ satisfying $$\max_{|\bt|\leq K}|M'_{\bt}(R) - M_{\bt}(\proj{R}(\mix))| \leq \noise$$
with probability at least $1 - \eta$.
\end{lemma}

\begin{proof}
Assume the words of a document are $x_1, \cdots, x_K$. 
If the document is from some topic $i$,
we have $\mathbb{E}[x_j = c] = \alpha_{p,t}$ for all $i, j, c$. 
For each word $x_j$, we generate a i.i.d. random vector $\vecb_j = [\beta_{1,j}, \cdots, \beta_{p,j}]^{\top}\in \{0,1\}^p$ such that $\mathbb{E}[\beta_{c,j}] = r_{c, x_j}$. For $\bt = (t_1, \cdots, t_p)$ with $|\bt|\leq K$, we define observable $h_{\bt} = \mathbf{1}[\forall c : \sum_{j=1}^{K} \beta_{c,j} = t_c]$.

Note that $h_{\bt}$ is the normalized histogram for projected distribution $\proj{R}(\mix)$. Hence, from the relation between normalized histogram and standard moments, we can conclude that 
\begin{align*}
    M(\proj{R}(\mix)) = \left(\bigotimes_{t=1}^{p} \textrm{Pas}\right) h
\end{align*}
where $\textrm{Pas} \in \mathbb{R}^{(K+1)\times (K+1)}$ such that $\textrm{Pas}_{i,j} = [j\geq i] \binom{j}{i} \binom{K+1}{i}$ and $\otimes$ is the tensor product.
Since $\|\textrm{Pas}\| \leq 2^{O(k)}$, we have $\left(\bigotimes_{t=1}^{p} \textrm{Pas}\right) \leq 2^{O{(kp)}}$. 

Let $\widehat{h}_{\bt}$ be the empirical average of $h_{\bt}$.
Since $h_{\bt}$ is a Bernoulli variable, from Hoeffding's inequality, we have 
$\Pr[|\widehat{h}_{\bt} - \mathbb{E}[h_{\bt}]| < \eps] \leq 2 \exp(-2\eps^2 s)$. 
By applying union bound on $\bt$, we can see that using $s>2^{O(kp)}\cdot\noise^{-2}\cdot\log(1/\eta)$ samples, we have $|\widehat{h}_{\bt} - \mathbb{E}[h_{\bt}]| < \noise \cdot 2^{-\Omega(kp)}$ with probability at least $1-\eta$.
Moreover, if we calculate $M'$ from $\widehat{h}$ using the relationship between $M$ and $h$, this leads to 
$\max_{|\bt|\leq K}|M'_{\bt}(R) - M_{\bt}(\proj{R}(\mix))| \leq \noise$.
\end{proof}

\begin{proofof}{Lemma~\ref{lm:lowdimtopic}}
Let $\noise = (\eps/kd)^{O(k)}$. From Theorem~\ref{thm:dimn-reconstruction}, the recovered
$k$-spike distribution $\cmix$ satisfies $\tran(\cmix, \mix) \leq (kd\noise^{\frac{1}{k}})^{O(1)} = O(\eps)$.
According to Lemma~\ref{lm:dim-sampling}, we can construct such a noisy moment oracle $M'(\cdot)$ using $(kd/\eps)^{O(k)}$ samples with high probability. 
\end{proofof}

\subsection{Dimension Reduction When $d\gg k^{\Omega(1)}$}
\label{subsec:larged}
%Until now, a brute force algorithm can be easily generated from our inequality.
%Which is to enumerate the mixture and test whether the moment is sufficiently closed.

In topic modeling, the number of words is typically much larger than the number of pure topics
(i.e., $d\gg k$).
In this case, we face another challenge in recovering the mixture.
First, there are $\binom{d}{1}+\binom{d}{2}+\cdots+\binom{d}{k}=O(d^k)$ many different moments.
Obtaining all empirical moments accurately enough would require a huge number of $2k-1$-snapshot samples.
So if $d\gg k$, we reduce the dimension from $d$ to $O(k)$.
Now, we prove the following theorem. It
improves on the result in \citet{rabani2014learning,li2015learning}, which uses
more than $(k/\epsilon)^{O(k^2)}$ $(2k-1)$-snapshots,
and \citet{gordon2020sparse} which use $(k/\eps)^{O(k)} \cdot (w_{\min}\minsep^k)^{-O(1)}$
$2k$-snapshots and requires the minimum separation assumption.

% \begin{theorem}\label{thm:highdimtopic}
% There is an algorithm that can learn an arbitrary $k$-spike mixture supported in $\simplex_{d-1}$
% for any $d$ within $L_1$ transportation distance $O(\epsilon)$ with probability at least $0.99$ using
% $\poly(d,k,\frac{1}{\epsilon})$, $\poly(d,k,\frac{1}{\epsilon})$, 
% many $1$-,$2$-,snapshots and $(k/\epsilon)^{O(k)}$ many $(2k-1)$-snapshots.
% \end{theorem}

Let $\mix$ be the target $k$-spike mixture in $\simplex_{d-1}$ ($d\gg k$).
Suppose there is a learning algorithm $\calA$ satisfying the assumption in the theorem.
We show how to apply $\calA$ to the projection of $\mix$ to a subspace of dimension at most $k$ in $\R^d$.
However, an arbitrary subspace of dimension at most $k$ is not enough, even the one spanning the $k$ spikes of $\mix$. %One challenge is that we work with transportation distance in the $L_1$-metric.
This is because $L_1$ distance is not rotationally invariant and not preserved under projection.
For example, the $L_1$ distance between $(1/d,\ldots,1/d)$ and $(2/d,\ldots,2/d, 0, \ldots, 0)$ is 1 in $\R^d$ but only $\sqrt{1/d}$ in the line spanned by the two points.
Hence, an accurate estimation of the projected measure $\mix_B=\proj{B}(\mix)$ may not translate to
an accurate estimation in the ambient space $\R^d$.
Here, we can use the dimension reduction method developed by \citet{li2015learning}, which shows that
if suffices to project
the mixture to a special subspace $B$, 
such that a unit $L_1$-ball in $B$ ($L_1$ measured in $\R^d$) is close to being an 
$\widetilde{O}(d^{-1/2})$ $L_2$-ball
($\widetilde{O}$ hides factors depending only on $k$ and $\epsilon$, but not $d$).

\begin{lemma}
\label{lm:dimensionreduction}
\citep{li2015learning} There is an algorithm that requires $\poly(d,k,\frac{1}{\epsilon})$
many 1-,2- snapshots to construct a subspace $B$ ($\dim(B)\leq k$)
with the following useful properties:
\begin{enumerate}
    \item Suppose
    $\{b_1,b_2,\cdots,b_m\}$ is an orthonormal basis of $B$ where $m=\dim(B)\leq k$
    (such a basis can be produced by the above algorithm as well).
    For any $i\in [m]$, $\|b_i\|_{\infty}\leq O(k^{3/2}\epsilon^{-2} d^{-1/2})$.
    \item
Suppose we can learn an approximation $\tmix_B$, supported on $\sspan(B)$, of
the projected measure $\mix_B=\Pi_{\sspan(B)}(\mix)$ such that $\tran(\mix_B, \tmix_B)\leq \epsilon_1=\poly(1/k,\epsilon)$ (here $L_1$ is measured in $\R^d$, not in the subspace)
using $N_1(d)$, $N_2(d)$ and $N_{K}(d)$ 1-, 2-, and $K$-snapshot samples.
Then there is an algorithm for learning a mixture $\tmix$ such that $\tran(\mix, \tmix)\leq \epsilon$ using
$O(N_1(d/\epsilon)+d\log d/\epsilon^3)$,
$O(N_2(d/\epsilon)+O(k^4 d^{3}\log n/\epsilon^6))$ and $O(N_{K}(d/\epsilon))$  1-, 2-, and
$K$-snapshot samples respectively.
\end{enumerate}
\end{lemma}

\begin{proofof}{Theorem~\ref{thm:highdimtopic}}
In light of Lemma~\ref{lm:dimensionreduction}, we only need to show how to learn the projection $\mix_B=\Pi_{\sspan(B)}(\mix)$
using the algorithm developed in Section~\ref{subsec:lowdimtopic}.

First, we show how to translate a $K$-snapshots from $\mix$ to a $K$-snapshot
in some new $k$-dimensional mixture $\mixnew$ (then we apply the algorithm in Section~\ref{subsec:lowdimtopic} for dimension $k$).
Let $L=\sum_{i=1}^m \|b_i\|_{\infty}$.
So $L\leq O(k^{5/2}\epsilon^{-2} d^{-1/2})$.
Let $f:[-L,L]\rightarrow [0,1]$ defined as $f(x)=\frac{x}{2mL}+\frac{1}{2m}$.
For each word in the original $K$-snapshot, say $i\in[d]$, we translate it to $j$ with probability
$q_{i,j}=f(b_{j,i})$, for each $j\in [m]$, where $b_{j,i}$ denotes the $i$th coordinate of $b_j$,
and translate it into $m+1$ with probability $q_{i,m+1}=1-\sum_{j=1}^m q_{i,j}$. Since we have $\sum_{j=1}^m f(b_{j,i})= \frac{1}{2mL}\sum_{j=1}^mb_{j,i}+\frac{1}{2}$, we know that $q_{i,1},q_{i,2},\cdots,q_{i,m+1}\in [0,1]$ and $\sum_{j=1}^{k+1} q_{i,j}=1$. So this is a well defined $k$-spike mixture $\mixnew$ and its $K$ snapshots.
Then we apply our learning algorithm $\calA$ for dimension $k$ to obtain some $\mixnew'$ such that
$\tran(\mixnew,\mixnew')\leq \eps_1=\poly(1/k,\epsilon)$
with $\poly(k,1/\eps_1)=\poly(k,1/\eps)$ many $2k-1$-snapshots.

Next we can see $\mixnew$ (in $\R^k$) and $\mix_{B}$ (in $\Span(B)$) 
are related by an affine transform.
For a spike $\alpha_i\in\supp(\mix)$, one can see that $\Pi_B(\alpha_i)=\sum_{j=1}^m a_{i,j}b_j$ ,
where $a_{i,j}=\langle\alpha_i,b_j\rangle$.
$\alpha_{i}$ produces a new spike $\beta_i$ in $\mixnew$ and the mapping is as follows:
for each $j\in [k]$,
we have
$$
\beta_{i,j}=\sum_{t=1}^d \alpha_{i,t}f(b_{j,t})=
\sum_{t=1}^d \alpha_{i,t}(\frac{b_{j,t}}{2mL}+\frac{1}{2m})=\frac{a_{i,j}}{2mL}+\frac{1}{2m}
$$
($\beta_{i,j}$ denotes the $j$th coordinate of $\beta_i$).
So $a_{i,j}=2mL(\beta_{i,j}-\frac{1}{2m})$ and $g(x)=2mLx-L$ is the affine transformation.
Now we can translate $\mixnew'$ into a mixture in $Span(B)$ by applying $g(.)$ in each coordinate of each $\beta_i\in \supp(\mixnew')$, and obtain an estimation of $\mix_B$, say $\mix_B'$. 

It remains to show $\tran(\mix_B',\mix_B)\leq\epsilon_1$.
If we let $\tran_B$ denote the transportation distance in $\Span(B)$ (where if $a=\sum_{j=1}^k a_{j}b_j, c=\sum_{j=1}^k c_{j}b_j$ then $\tran_B(a,c)=\sum_{j=1}^k |a_{j}-c_{j}|$). Then $\tran_B(\mix_B',\mix_B)\leq 2mL\tran(\mixnew',\mixnew)=\poly(\epsilon,\frac{1}{k})/\sqrt{d}$.
Hence, we have 
$$
\tran(\mix_B',\mix_B)\leq d \max_i \|b_i\|_\infty\tran_B(\mix_B',\mix_B)\leq \epsilon_1.
$$
where the first inequality holds since 
$\tran(a,c)=\|\sum_{j=1}^k (a_{j}-c_{j})b_j\|_1 \leq d\max_i \|b_i\|_\infty  \sum_{j=1}^k (a_{j}-c_{j})$ for any two point $a,c\in \Span(B)$.
\end{proofof}

\section{Applications to Gaussian Mixture Learning}
\label{sec:Gaussian-full}

In this section, we show how to leverage our results
for sparse moment problem to obtain improved algorithms for learning Gaussian mixtures. 
We consider the following setting
studied in \citet{wu2020optimal, doss2020optimal}.
A $k$-Gaussian mixture in $\mathbb{R}^d$ can be parameterised as $\mix_N = (\veca, \vecw, \Sigma)$. Here, $\veca = \{\veca_1, \veca_2, \cdots, \veca_k\}$ and $\vecw = \{w_1, w_2, \cdots, w_k\} \in \Delta_{k-1}$ where $\veca_i \in \mathbb{R}^d$ specifies the mean of $i$th component and $w_i \in [0, 1]$ is the corresponding weight. $\Sigma \in \mathbb{R}^{d\times d}$ is the common covariance matrix for all $k$ mixture components and we assume that $\Sigma$ is known in advance.
We further assume $\|\veca_i\|_2 \leq 1$ and the maximal eigenvalue $\|\Sigma\|_2$ is bounded by a constant. 
For $k$-Gaussian mixture $\mix_N = (\veca, \vecw, \Sigma)$, each observation is distributed as 
\begin{align*}
    \mix_N \sim \sum_{i=1}^{k} w_i N(\veca_i, \Sigma).
\end{align*}
We consider the parameter learning problem, that is, to learn the parameter $\veca$ and $\vecw$ given known covariance matrix $\Sigma$ and a set of i.i.d. samples from $\mix_N$. The model is also called {\em Gaussian location mixture model} \citep{wu2020optimal, doss2020optimal}.
\footnote{
\citet{wu2020optimal} also studied the problem with unknown $\Sigma$. We leave it as a future direction.
}
%Our plan is to utilize the discrete mixture $\mix = (\veca, \vecw)$ with the same parameters, and then apply the algorithm for sparse moment problems on $\mix$. 

\subsection{Efficient Algorithm for $d=1$}

In the 1-dimensional case, we denote the known variance by $\sigma$
which is upper bounded by some constant.
As shown by \citet{wu2020optimal}, the moments of Gaussian mixture have a close connection with the moments of corresponding discrete distributions $\mix = (\veca, \vecw)$. For $x \sim N(\mu, 1)$, we have $\Exp[H_t(x)] = \mu^t$ for Hermite polynomial $H_t(x)$ which is defined as
\begin{align}
\label{eq:hermite}
    H_t(x)  = \sum_{i=0}^{t} h_{t,i} x^i = t!\sum_{j=0}^{\lfloor \frac{t}{2} \rfloor} \frac{(-1/2)^j}{j!(t - 2j)!} x^{t-2j}.
\end{align}
For 1-dimensional Gaussian mixture $\mix_N = (\veca, \vecw, \sigma)$ with variance $\sigma$, the $t$th moment of the discrete distribution $\mix = (\veca, \vecw)$ satisfies 
\begin{align*}
    M_t(\mix) = \Exp_{x \sim \mix_N}\left[\sum_{i=0}^{t} h_{t,i} \sigma^{t-i} x^i\right].
\end{align*}
The following lemma guarantees the performance of estimating the moment by sampling.
\begin{lemma}{(Lemma 5 of \citet{wu2020optimal}, restated)}
\label{lm:var-Gauss}
    Let $x_1, \cdots, x_n \sim \mix_N$ be $n$ independent samples. 
    $$\tM_t(\mix) = \frac{1}{n} \sum_{j=1}^{n} \left[  \sum_{i=0}^{t} h_{t,i} \sigma^{t-i} x^i_j \right]$$
    is an unbiased estimator for $M_{t}(\mix)$, and the variance of this estimator can be bounded by 
    $$\Var[{\tM_t(\mix)}] \leq \frac{1}{n} (\sigma t)^{O(t)}.$$
\end{lemma}

Thus we can compute the moment of the original distribution and use our Algorithm~\ref{alg:dim1} for recovering the parameter of $\mix$. More concretely, we can replace the last two lines (an SDP) in Algorithm 2 of \citet{wu2020optimal} by our Algorithm~\ref{alg:dim1}, and obtain the following theorem.
The post-sampling time is improved from from $O(k^{2\omega})$ 
in \citet{wu2020optimal} to $O(k^2)$.

\begin{theorem}
\label{thm:1dim-Gaussian}
Let $\mix_N$ be an arbitrary $k$-Gaussian mixture over $\mathbb{R}$ with means $\alpha_1, \cdots, \alpha_n$ and known variance $\sigma$ bounded by some constant.
There is an algorithm that can learn the parameter $\mix=(\veca, \vecw)$ within transportation distance $O(\eps)$, with probability at least $0.99$, using $(k/\eps)^{O(k)}$ samples from the mixture $\mix_N$. Moreover, once we obtain the estimation of the moments of $\mix_N$, 
our algorithm only uses $O(k^2)$ arithmetic operations.
\end{theorem}

\begin{proof}
    Let $c$ be the constant to be determined. 
    With $n = (k/\epsilon)^{ck}$ samples, the variance of the empirical moment can be bounded by
    \begin{align*}
        \Var[M_t'(\mix)] \leq \frac{1}{(k/\epsilon)^{ck}} \cdot k^{O(k)} \leq (\epsilon/k)^{\Omega(k)}
    \end{align*}
    where the first inequality holds due to Lemma~\ref{lm:var-Gauss} and $\sigma$ is bounded by a constant, and the second inequality holds by selecting a large enough constant $c>0$.
    According to Chebyshev's inequality, for each $t$, with probability at least $1 - 0.01 k^{-2}$, 
    \begin{align*}
        |M'_{t}(\mix) - M_{t}(\mix)| \leq 0.1k \cdot (\epsilon/k)^{\Omega(k)} \leq (\epsilon/k)^{\Omega(k)}.
    \end{align*}
    By taking union bound, the probability that the above inequality holds for all $0 \leq t \leq 2k - 1$ is greater than $0.99$. We can conclude the result by applying Theorem~\ref{thm:1dim-kspikecoin} directly.
\end{proof}

\begin{remark} (Connection between heat equations and Gaussian mixtures)
\label{remark:heatequation}
Consider the heat equation
$$
\frac{\partial u(x,t)}{\partial t} =\Delta_x u(x,t),
\quad\text{where}\quad
\Delta_x u =\sum_i \frac{\partial^2 u}{\partial x_i^2}.
$$
Suppose at $t=0$, $u(x,0)$ is a $k$-spike distribution
$u(x,0)=\sum_{i=1}^k w_i \delta_{\alpha_i}(x)$.
It is well known that for $t>0$,
$u(x,t)=\sum_{i=1}^k w_i N(\alpha_i, 2t)$
(see e.g.,\citet{gorenflo2002moment,MV10}).
Hence, if we know $u(.,t)$ at time $t>0$, we can compute 
the moments and recovering the heat source 
$u(.,0)$ at $t=0$ is equivalent to recovering the means of the Gaussian mixture components (i.e., the spikes) from the moment information,
hence can be solved by the moment problem.
\end{remark}

\subsection{Efficient Algorithm for $d>1$}
\label{sec:highdim-Gaussian}

For higher dimensional Gaussian mixture, 
we can reduce the problem to learning the discrete mixture $\mix = (\veca, \vecw)$,
and leverage Algorithm~\ref{alg:dim3} to solve the problem. 
Assume the locations of all Gaussians are in the unit ball.
The error can be bounded easily according to Theorem~\ref{thm:highdim-kspike}. 

Firstly, as shown in \citet{doss2020optimal}, one can use SVD to reduce a $d$-dimensions problem to a $k$-dimension problem using $\poly(1/\eps, d, k)$ samples. Thus, we only need to consider settings with $d \leq k$.
Similar to the 1d case, we can transform the problem of learning Gaussian mixture $\mix_N = (\veca, \vecw, \Sigma)$
to the problem of learning the discrete mixture $\mix = (\veca, \vecw)$.
In particular, as we show below, we can estimate the moments of the projection of $\mix$ from the samples. After obtaining the noisy moment information, we can 
apply Algorithm~\ref{alg:dim3} to recover the parameter $(\veca, \vecw)$.
We mention that we need to modify Algorithm~\ref{alg:dim3} slightly:
we change to domain from $\Delta_{d-1}$ to the unit ball (it suffices to change the projected domain in Line~\ref{line:dim3-ca} in Algorithm~\ref{alg:dim3} from the simplex to the unit ball). We note this does not affect the proof to Theorem~\ref{thm:highdim-kspike} 
since it only requires the projected space to be a convex domain.

The remaining task is to show how to estimate the projected moments
for 1d and 2d projections.
To estimate the moments of 1-dimensional projection, we can use the estimator in the last section.
In particular, for the projected measure $\proj{\vecr}(\mix)$ where $\vecr$ is an arbitrary unit vector, the $t$th moment can be computed by
\begin{align*}
    M_{t}(\proj{\vecr}(\mix)) = \Exp_{\vecx \sim \mix_N} \left [\sum_{i=0}^{t} h_{t,i} (\vecr^{\top} \Sigma \vecr)^{t-i} (\vecr^{\top} \vecx)^i \right].
\end{align*}
We can use sample average to estimate $M_{t}(\proj{\vecr}(\mix))$,
and we denote the estimation as $\tM_{t}(\proj{\vecr}(\mix))$.

Now, we consider the problem of estimating the moments
for 2-dimensional projection along $\vecr_1$ and $\vecr_2$.
First, we compute the vector $\vecr_2' = \vecr_2 - \frac{\vecr_1^{\top} \Sigma \vecr_2}{\vecr_1^{\top} \Sigma \vecr_1} \vecr_1$.
According to Gram–Schmidt process, 
$\vecr_2'$ is $\Sigma$-orthogonal to $\vecr_1$, i.e.,
$\vecr_1^{\top} \Sigma \vecr_2' = 0$.
Moreover, $\vecr_1^{\top} \vecx$ and $\vecr_2'^{\top} \vecx$ are independent Gaussian distributions for variable $\vecx \sim N(0, \Sigma)$. 
Note that for independent Gaussian variables $x_1 \sim N(\mu_1, 1)$ and $x_2 \sim N(\mu_2, 1)$, we have $\Exp[H_{t_1}(x_1) H_{t_2}(x_2)] = \Exp[H_{t_1}(x_1)] \Exp[H_{t_2}(x_2)] = \mu_1^{t_1} \mu_2^{t_2}$,
where $H$ is the Hermite polynomial defined in \eqref{eq:hermite}.
As a result, the projected moments along $\vecr_1$ and $\vecr_2'$ can be computed by 
\begin{align*}
    M_{(t_1, t_2)}(\proj{[\vecr_1, \vecr_2']}(\mix)) = \Exp_{\vecx \sim \mix_N} \left [ \sum_{i_1=0}^{t_1}  \sum_{i_2=0}^{t_2} h_{t_1,i_1}  h_{t_2,i_2} (\vecr_1^{\top} \Sigma \vecr_1)^{t_1-i_1} (\vecr_2'^{\top} \Sigma \vecr_2')^{t_2-i_2} (\vecr_1^{\top} \vecx)^{i_1} (\vecr_2'^{\top} \vecx)^{i_2} \right].
\end{align*}
Moreover, the projected moments along $\vecr_1$ and $\vecr_2$ can be computed by
\begin{align*}
    M_{(t_1, t_2)}(\proj{[\vecr_1, \vecr_2]}(\mix)) = \sum_{i=0}^{t_2} \binom{t_2}{i} \left(\frac{\vecr_1^{\top} \Sigma \vecr_2}{\vecr_1^{\top} \Sigma \vecr_1}\right)^{i} M_{(t_1+i,t_2-i)}(\proj{[\vecr_1, \vecr_2']}(\mix)) 
\end{align*}
since $\vecr_2^{\top} \vecx = \vecr_2'^{\top} \vecx + \frac{\vecr_1^{\top} \Sigma \vecr_2}{\vecr_1^{\top} \Sigma \vecr_1} \vecr_1^{\top} \vecx$ which implies $(r_2^{\top} \vecx)^{t_2} = \sum_{i=0}^{t_2} \binom{t_2}{i} \left(\frac{\vecr_1^{\top} \Sigma \vecr_2}{\vecr_1^{\top} \Sigma \vecr_1}\right)^{i} (r_1^{\top} \vecx)^{i} (r_2'^{\top} \vecx)^{t_2-i}$.
We use $\tM_{(t_1, t_2)}(\proj{[\vecr_1, \vecr_2]}(\mix))$
to denote the sample average estimation of 
$M_{(t_1, t_2)}(\proj{[\vecr_1, \vecr_2]}(\mix))$.
%We construct the moment oracle $M_{t}(\proj{\vecr}(\mix))$ and $M_{t}(\proj{\vecr_1, %\vecr_2}(\mix))$ by estimating the expectation using the samples from $\mix_N$ and %use Algorithm~\ref{alg:dim3} with the moment oracle. 
We are ready to state the performance guarantee:

\begin{proofof}{Theorem~\ref{thm:ndim-Gaussian}}
    %Since there is a reduction statement from high-dimensional problem to dimension, 
    We first consider the problems of dimension $d \leq k$. 
    The only thing we need to show is that the above estimators 
    $\tM_{t}(\proj{\vecr}(\mix))$ and $\tM_{t}(\proj{\vecr_1, \vecr_2}(\mix))$
    have sufficient accuracy with 
    high probability using $n = (k /\eps)^{\Theta(k)}$ samples.
    Clearly, both estimators are unbiased. 
    So it is sufficient to bound the variance of the estimators.
    
    For 1-dimensional projected moments, we can use the same argument as Theorem~\ref{thm:1dim-Gaussian}. In particular, we can show that with probability at least $0.999$, for all $0 \leq t \leq 2k - 1$,
    \begin{align*}
        |\tM_{t}(\proj{\vecr}(\mix)) - M_{t}(\proj{\vecr}(\mix))| \leq (k/\eps)^{O(k)}.
    \end{align*}
    
    For 2-dimensional projected moments, we first bound the variance of the estimator along $r_1$ and $r_2'$. With the same argument as Lemma~\ref{lm:var-Gauss}, we have 
    \begin{align*}
        \Var[\tM_{(t_1, t_2)}(\proj{[\vecr_1, \vecr_2']}(\mix))] \leq (\eps / k)^{\Omega(k)}.
    \end{align*}
    for all $t_1 + t_2 \leq 2k$ with probability at lest $0.999$.
    Moreover, since we randomly chooses the same $r_1$ over a sphere of radius $\Theta(1)$ in Line 6 of Algorithm~\ref{alg:dim3}, we have $\vecr_1^{\top} \Sigma \vecr_1 \geq \Omega(\frac{\mu}{k^2} \|\vecr_1\|_2^2 \|\Sigma\|_2)$ with probability at least $1 - \mu$. In this case, we have $$\frac{\vecr_1^{\top} \Sigma \vecr_2}{\vecr_1^{\top} \Sigma \vecr_1} \leq \frac{\|\vecr_1\|_2 \|\Sigma\|_2 \|\vecr_2\|}{\Omega(\frac{\mu}{k^2} \|\vecr_1\|_2^2 \|\Sigma\|_2)} \leq O(k^2/ \mu)$$
    where the last inequality holds because $\vecr_1$ and $\vecr_2$ are of norm $\Theta(1)$. Since $\sqrt{\Var[x + y]} \leq \sqrt{\Var[x]} + \sqrt{\Var[y]}$ for any random variables $x$ and $y$,
    \begin{align*}
        \sqrt{\Var[\tM_{(t_1, t_2)}(\proj{[\vecr_1, \vecr_2]}(\mix))]} \leq \sum_{i=0}^{t_2} \binom{t_2}{i} \left(\frac{\vecr_1^{\top} \Sigma \vecr_2}{\vecr_1^{\top} \Sigma \vecr_1}\right)^{i} \sqrt{\Var[\tM_{(t_1+i, t_2-i)}(\proj{[\vecr_1, \vecr_2']}(\mix))]}.
    \end{align*}
    As a result, by choosing $\mu = 0.001$, we have 
    \begin{align*}
        \Var[\tM_{(t_1, t_2)}(\proj{[\vecr_1, \vecr_2]}(\mix))] \leq (\eps /k)^{\Omega(k)}
    \end{align*} 
    holds for all $\vecr_2$ with probability at least $0.999$. 
    Conditioning on that this event holds, according Chebyshev's inequality and a union bound, with probability at least $1 - 0.001d$, for all $0 \leq t_1, t_2 \leq 2k - 1$, 
    \begin{align*}
        |\tM_{(t_1, t_2)}(\proj{[\vecr_1, \vecr_2]}(\mix))] - M_{(t_1, t_2)}(\proj{[\vecr_1, \vecr_2]}(\mix))]| \leq (\eps / k)^{\Omega(k)}.
    \end{align*}
    Recall that the algorithm uses 1d projection for one $\vecr$, 2d projection for one $\vecr_1$ and $d$ different $\vecr_2$ and the algorithm succeeds with high probability statement for random $\vecr$. Hence, by union bound over all these events, the projected moment estimations have sufficient accuracy with a probability at least $0.995$. 
    
    If the dimension $d>k$,
    we can apply the dimension reduction in \citet{doss2020optimal},
    which requires an extra sample complexity of $\poly(1/\eps, d, k)$. The total sample complexity is $(k/\eps)^{O(k)} + \poly(1/\eps, d, k)$.
\end{proofof}

Finally, we discuss the running time during the sampling phase and post-sampling
phase.
The dimension reduction requires $\poly(1/\eps, d, k)$ samples and 
$O(d^3)$ time \citep{doss2020optimal}.
For the recovery problem in dimension $k$, each $1$-d or $2$-d moment oracle can be computed in $O(nk^3)$ time where $n=(k/\eps)^{O(k)}$ is the number of samples. Since we only require $O(k)$ moment oracles, the sampling time can be bounded by $O(n k^4)$.
This improves the $O(n^{5/4}\poly(k))$ sampling time in \citet{doss2020optimal}
(their algorithm requires 1d projections to $O(n^{1/4})$ many directions.
For the post-sampling running time, our Algorithm~\ref{alg:dim3} 
runs in time $O(k^4)$ (since $d\leq k$ by dimension reduction).
This improves the $O(n^{1/2}\poly(k))$ post-sampling time in \citet{doss2020optimal}.

% This theorem improves Theorem 9 in \citet{wu2020optimal}, which removes the separation condition for the input mixture.

\section{Other Related Work}
The problem we study can be seen as a sparse version of the classic {\em moment problem}
in which our goal is to invert the mapping that takes a measure to the sequences of moments
 \citep{schmudgen2017moment,lasserre2009moments}. When the measure is supported on a finite interval, the problem is known as the Hausdorff moment problem.

Learning statistical mixture models has been studied extensively for the past two decades. 
A central problem
in this area was the problem of learning a mixture of high-dimensional
Gaussians, even robustly \citep{Das99,DS00,AK01,VW02,KSV05,AM05,FOS06,BV08,KMV10,BS10,MV10,liu2021settling,bakshi2020robustly,wu2020optimal,doss2020optimal}. 
Many other structured mixture models have also been studied (see e.g., \citet{KMRRSS94,CryanGG02,BGK04,MR05,DHKS05,
FeldmanOS05,KSV05,CR08a,DaskalakisDS12,liu2018efficiently}).
Our problem is closely related to topic
models which have also been studied extensively recently \citep{AGM12,AFHKL12,AHK12,arora2018learning}.
Most work make
certain assumptions on the structure of the mixture, such as the pure topics being separated, Dirichlet prior, the existence of anchor words, or certain rank conditions.
Some assumptions (such as \citet{AGM12,AFHKL12}) allow one to use documents of constant length 
(independent of the number of pure topics $k$ and the desired accuracy $\eps$) for recovering all topics.

Our problem is also related to the super-resolution problem in which each measurement takes the form
$
v_\ell=\int e^{2\pi \ell t} \d \mix(t)+\noise_\ell
$
where $\noise_\ell$ is the noise of the measurement. We can observe the first few $v_\ell$ for $|\ell|\leq K$ where $K$ is called {\em cutoff frequency}, and the goal is to recover the original signal $\mix$. 
There is a long history of the estimation problem. The noiseless version
can be solved by 
Prony's method \citep{prony1795}, ESPRIT algorithm \citep{roy1989esprit} or matrix pencil method \citep{hua1990matrix}, if $K\geq k$ (i.e., we have $K=2k+1$ measurements). The noisy case has 
also been studied extensively and a central goal is to understand the relations between the cutoff frequency, the size of measure noises and minimum separation  \citep{donoho1992superresolution,candes2014towards,candes2013super,Moitra15,huang2015super,chen2016fourier}. Various properties, such as the
condition number of the Vandermonde matrix, also play essential roles in this line of study \citep{Moitra15}. The relation between  the Vandermonde matrix and Schur polynomial is also exploited in \citet{chen2016fourier}.

\bibliographystyle{plainnat}
\bibliography{main}

\end{document}